\documentclass[10pt,twocolumn,twoside]{IEEEtran}
\usepackage{mathptmx}
\usepackage{float}
\usepackage[fleqn]{amsmath}
\usepackage{cite}
\usepackage[dvips]{graphicx}
\usepackage{psfrag}
\usepackage{amsmath}
\usepackage{amsbsy}
\usepackage{amssymb}
\usepackage{latexsym}
\usepackage{color}
\usepackage{algorithm}
\usepackage{algorithmic}

\usepackage{amsthm}

\theoremstyle{definition}

\theoremstyle{definition}
\newtheorem{proposition}{Proposition}

\theoremstyle{definition}
\newtheorem{example}{Example}

\theoremstyle{definition}
\newtheorem{definition}{Definition}

\theoremstyle{definition}
\newtheorem{theorem}{Theorem}

\theoremstyle{definition}
\newtheorem{lemma}{Lemma}

\theoremstyle{definition}
\newtheorem{corollary}{Corollary}

\theoremstyle{definition}
\newtheorem{assumption}{Assumption}

\theoremstyle{definition}

\theoremstyle{remark}
\newtheorem{remark}{Remark}

\theoremstyle{definition}

\theoremstyle{definition}

\theoremstyle{definition}

\theoremstyle{definition}

\theoremstyle{definition}

\newcommand{\sgn}{{\rm sgn}}

\newcommand{\norm}[1]{\left\|#1\right\|}

\newcommand{\real}{{\mathbb R}}
\newcommand{\Real}{{\mathbb R}}

\newcommand{\integer}{{\mathbb Z}}
\newcommand{\refeq}[1]{(\ref{#1})}
\newcommand{\reftab}[1]{Table \ref{#1}}
\newcommand{\reffig}[1]{Figure \ref{#1}}

\newcommand{\prox}{{\rm prox}}

\newcommand{\argmin}{\operatornamewithlimits{argmin}}
\newcommand{\argmax}{\operatornamewithlimits{argmax}}

\newcommand{\X}{{\mathcal{X}}}

\renewcommand{\H}{{\mathcal{H}}}

\newcommand{\U}{\mathcal{U}}

\newcommand{\T}{{\sf T}}
\newcommand{\Dict}{{\mathcal{D}}}

\newcommand{\di}{{\rm d}}
\newcommand{\sigu}{\textbf{u}}
\newcommand{\sigh}{\textbf{h}}
\newcommand{\sigw}{\textbf{w}}

\newcommand{\sigk}{\textbf{k}}
\newcommand{\sigx}{\textbf{x}}
\newcommand{\sigy}{\textbf{y}}
\newcommand{\sigv}{\textbf{v}}
\newcommand{\sigz}{\textbf{z}}

\newcommand{\x}{\times}
\newcommand{\s}{\sigma}

\newcommand{\ka}[1]{\kappa\left(#1\right)}

\newcommand{\inpro}[1]{\left<#1\right>}
\newcommand{\Ex}[1]{E\left[#1\right]}

\definecolor{darkgreen}{rgb}{0,.6,0}
\definecolor{medorange}{rgb}{0.7,0.3,0}
\definecolor{cyancyan}{rgb}{0.68, 0.92, 0.92}
\def\nn{\nonumber}

\begin{document}

\numberwithin{equation}{section}
\renewcommand{\theequation}{\thesection.\arabic{equation}}
\numberwithin{theorem}{section}
\numberwithin{assumption}{section}
\numberwithin{definition}{section}
\numberwithin{lemma}{section}
\numberwithin{proposition}{section}
\numberwithin{remark}{section}
\numberwithin{corollary}{section}
\numberwithin{table}{section}
\numberwithin{figure}{section}
\numberwithin{fact}{section}
\numberwithin{example}{section}
%
% paper title
\title{Barrier-Certified Adaptive Reinforcement Learning\\ with Applications to Brushbot Navigation}

 \author{Motoya Ohnishi, Li Wang,
 	Gennaro Notomista,
and Magnus Egerstedt
\thanks{This work was sponsored in part by the U.S. National Science Foundation under Grant No. 1531195.
	The work of M. Ohnishi was supported in part by the Scandinavia-Japan Sasakawa Foundation under Grant GA17-JPN-0002 and the Travel Grant of the School of Electrical Engineering, Royal Institute of Technology.}
\thanks{M. Ohnishi is with the School of Electrical Engineering, Royal Institute of Technology, 11428 Stockholm, Sweden,
		the Georgia Robotics and Intelligent Systems Laboratory, Georgia Institute of Technology, Atlanta, GA 30332 USA,
	and also with the RIKEN Center for Advanced Intelligence Project, Tokyo 103-0027, Japan (e-mail: motoya@kth.se).}
\thanks{L. Wang and M. Egerstedt are with the School of Electrical and Computer
Engineering, Georgia Institute of Technology, Atlanta, GA 30332 USA (e-mail: liwang@gatech.edu; magnus@gatech.edu).}
\thanks{G. Notomista is with the School of Mechanical Engineering, Georgia Institute of Technology, Atlanta, GA 30313 USA (e-mail: g.notomista@gatech.edu).
}
}

\maketitle

\begin{abstract}
	This paper presents a safe learning framework that employs an adaptive model learning algorithm together with barrier certificates for systems with possibly nonstationary agent dynamics.
	To extract the dynamic structure of the model, we use a sparse optimization technique.
	We use the learned model in combination with control barrier certificates which constrain policies (feedback controllers)
	in order to maintain safety, which refers to avoiding particular undesirable regions of the state space.
	Under certain conditions, recovery of safety in the sense of Lyapunov stability after violations of safety due to the nonstationarity
	is guaranteed.
	In addition, we reformulate an action-value function approximation to make any kernel-based nonlinear function estimation method applicable
	to our adaptive learning framework.
	Lastly, solutions to the barrier-certified policy optimization are guaranteed to be globally optimal,
	ensuring the greedy policy improvement under mild conditions.
	The resulting framework is validated via simulations of a quadrotor, which has previously been used under {\em stationarity} assumptions in the safe learnings literature,
	and is then tested on a real robot, the brushbot, whose dynamics is unknown, highly complex and nonstationary.
\end{abstract}


\section{Introduction}
\label{sec:introduction}
By exploring and interacting with an environment, reinforcement learning can
determine the optimal policy with respect to the long-term rewards given to an agent \cite{reinforcement,lewis2009reinforcement}.
Whereas the idea of determining the optimal policy in terms of a cost over some time horizon is standard
in the controls literature \cite{optimalcontrol}, reinforcement learning is aimed at
learning the long-term rewards by exploring the states and actions.
As such, the agent dynamics is no longer explicitly taken into account, but rather is subsumed by the
data.
%%%%%%%

%%%%%%%
If no information about the agent dynamics is available, however, an agent might end up in certain regions of the state space that
must be avoided while exploring.
Avoiding such regions of the state space is referred to as {\it safety}.
Safety includes collision avoidance, boundary-transgression avoidance,
connectivity maintenance in teams of mobile robots, and other mandatory constraints, and
this tension between exploration and safety becomes particularly pronounced in robotics, where
safety is crucial.
%%%%%%%%%%%%

%%%%%%%%%%%%
In this paper, we address this safety issue,
by employing model learning in combination with barrier certificates.
In particular, we focus on learning for systems with discrete-time nonstationary (or time-varying) agent dynamics.
Nonstationarity comes, for example, from failures of actuators, battery degradations, or sudden environmental disturbances.
The result is a method that adapts to nonstationary agent dynamics and, under certain conditions, ensures recovery of safety in the sense of Lyapunov stability even after violations of safety due to the nonstationarity occur.
We also propose discrete-time barrier certificates that guarantee global optimality of solutions to the barrier-certified policy optimization,
and we use the learned model for barrier certificates.
%%%%%%%%%%%%%

%%%%%%%%%%%%%
Over the last decade, the safety issue has been addressed under the name of safe learning,
and plenty of solutions have been proposed \cite{berkenkamp2017safe,berkenkamp2016safe,schreiter2015safe,akametalu2014reachability,shalev2016safe,ammar2015safe,van2017online,achiam2017constrained,demonst1,wang2017safe}.
To ensure safety while exploring, an initial knowledge of the agent dynamics,
initial safe policy or a {\it teacher} advising the agent is necessary \cite{Safesurvey,berkenkamp2017safe}.
To obtain a model of the agent dynamics, human operators may maneuver the agent and record its trajectories \cite{demonstsurvey,demonst1}.
It is also possible that an agent continues exploring without entering the states with low long-term {\it risks} (e.g., \cite{geibel2006reinforcement, achiam2017constrained}).
Due to the inherent uncertainty, the worst case scenario (e.g., possible lowest rewards) is typically taken into account \cite{wang2017safe,coraluppi1999risk} and the set of safe policies can be expanded by exploring the states \cite{berkenkamp2017safe,berkenkamp2016safe}.
To address the issue of this uncertainty for nonlinear-model estimation tasks, Gaussian process regression \cite{rasmussen2006gaussian} is a strong tool, and many safe learning studies
have taken advantage of its property (e.g., \cite{wang2017safe,berkenkamp2017safe,akametalu2014reachability,schreiter2015safe,van2017online}).
%%%%%%%%%%%%%

%%%%%%%%%%%%%
Nevertheless, when the agent dynamics is nonstationary and the long-term rewards vary accordingly, the assumptions often made in
the safe learnings literature no longer hold, and violations of safety become inevitable.
In such cases, we wish to ensure that the agent is at least successfully brought back to the set of safe states
and the negative effect of an unexpected violation of safety is mitigated.
Moreover, the long-term rewards must also be learned in an adaptive manner.
These are the core motivations of this paper.
%%%%%%%%%%%%%%

%%%%%%%%%%%%%%%%%%%%%%%%%%%
To constrain the states within a desired safe region while exploring, we employ control barrier functions (cf. \cite{xu2015robustness,wieland2007constructive,glotfelter2017nonsmooth,wang2017safety,ames2017control,2017dcbf}).
When the exact model of the agent dynamics is available, control barrier certificates ensure
that an agent remains in the set of safe states for all time by constraining the instantaneous control input at
each time.  Also, an agent outside of the set of safe states is forced back to safety (Proposition \ref{prpCBF}).
A useful property of control barrier certificates is that they modify polices
only when violations of safety are truly imminent \cite{wang2017safety}.
%%%%%%%%%%%%%%

%%%%%%%%%%%%%%%%%%%%%%%%%%%
If no nominal model (or simulation) of the possibly nonstationary agent dynamics is available, on the other hand, violations of safety are inevitable.
Therefore, we wish to adaptively learn the agent dynamics, and eventually bring the agent back to safety.
To this end, we propose a learning framework for a possible nonstationary agent dynamics, which recovers safety in the sense of Lyapunov stability
under some conditions.
This learning framework ties adaptive algorithms with control barrier certificates by
focusing on set-theoretical aspects and monotonicity (or non-expansivity).
By augmenting the state with the estimate of agent dynamics, Lyapunov stability with respect to the set of augmented safe states is guaranteed (Theorem \ref{maintheo}).
Also, to efficiently enforce control barrier certificates, we employ adaptive sparse optimization techniques to extract dynamic structures (e.g., control-affine
dynamics) by identifying truly {\em active} structural components (see Section \ref{subsec:overviewdynamic} and \ref{subsec:modellearn}).
%%%%%%%%%%%%%%

%%%%%%%%%%%%%%%%%%%%%%%%%%%
In addition, the long-term rewards need to be adaptively estimated when the agent dynamics is nonstationary.
To this end, we reformulate the action-value function approximation problem so that,
even if the action-value function varies, it can be adaptively estimated in the same functional space by
employing an adaptive supervised learning algorithm in the space.
Consequently, resetting the learning whenever the agent dynamics varies becomes unnecessary.
Moreover, we present a barrier-certified policy update strategy by employing control barrier functions to effectively constrain policies.
Because the global optimality of solutions to the constrained policy optimization is necessary to ensure the greedy improvement of a policy,
we propose a discrete-time control barrier certificate that ensures
the global optimality under some mild conditions (see Section \ref{subsec:adaptiveRL} and Theorem \ref{CBFaffine} therein).
This is an improvement of the previously proposed discrete-time control barrier certificate \cite{2017dcbf}.
%%%%%%%%%%%%%%%%%%%%%%%%%%%

%%%%%%%%%%%%%%%%%%%%%%%%%%%
To validate and clarify our learning framework, we first conduct experiments of quadrotor simulations.  Then, we conduct real-robotics experiments on a {\it brushbot}, whose dynamics is unknown, highly complex and nonstationary, to test the efficacy of our framework in the real world (see Section \ref{sec:numerical}).
This is challenging due to many uncertainties and lack of simulators often used in applications of reinforcement learning
in robotics (see \cite{kober2013reinforcement} for example).
\section{Preliminaries}
\label{sec:preliminaries}
In this section, we present some of the related work and the system model considered in this paper.
Throughout, $\Real$, $\integer_{\geq0}$ and $\integer_{>0}$ are
	the sets of real numbers, nonnegative integers and positive integers,
	respectively.
	Let $\norm{\cdot}_{\H}$ be the norm induced by the inner product $\inpro{\cdot,\cdot}_{\H}$ in an inner-product space $\H$.
	In particular, define $\inpro{\sigx,\sigy}_{\Real^L}:=\sigx^{\T}\sigy$ for $L$-dimensional real vectors $\sigx,\sigy\in\real^L$,
	and $\norm{\sigx}_{\Real^L}:=\sqrt{\inpro{\sigx,\sigx}_{\Real^L}}$, where $(\cdot)^{\T}$ stands for transposition.
	We define $[\sigx;\sigy]$ as $[\sigx^{\T},\sigy^{\T}]^{\T}$, and let
$\sigx_n\in\X\subset\Real^{n_x}$ and $\sigu_n\in\U\subset\Real^{n_u}$, for $n_x,n_u\in\integer_{>0}$, denote the state and the control input at time instant $n\in\Real_{\geq0}$, respectively.
\subsection{Related Work}
\label{subsec:relatedwork}
The primary focus of this paper is the safety issue {\it while exploring}.
Typically, some initial knowledges, such as an initial safe policy and a model of the agent dynamics, are required to address the safety issue while exploring;
therefore, model learning is often employed together.
We introduce some related work on model learning and kernel-based action-value function approximation.
\subsubsection{Model Learning for Safe Maneuver}
The recent work in \cite{wang2017safe}, \cite{akametalu2014reachability}, and \cite{berkenkamp2017safe}
assumes an initial conservative set of safe policies, which is gradually expanded as more data become available.
These approaches are designed for stationary agent dynamics, and Gaussian processes (GPs) are employed to obtain the confidence interval of the model.
To ensure safety, control barrier functions and control Lyapunov functions are employed in \cite{wang2017safe} and \cite{berkenkamp2017safe},
respectively.
On the other hand, the work in \cite{van2017online} uses a trajectory optimization based on the receding
horizon control and model learning by GPs, which is computationally expensive when the model is highly nonlinear.
%%%%%%%

%%%%%%
In this paper, we aim at tying adaptive model learning algorithms and control barrier certificates by focusing on set-theoretical aspects and monotonicity (or non-expansivity).
Hence, we employ an adaptive filter with monotone approximation property, which shares similar ideas
with stable online learning for adaptive control based on Lyapunov stability (c.f. \cite{janakiraman2013lyapunov,french2000non,polycarpou1996stable,aastrom2013adaptive}, for example).  
\subsubsection{Learning Dynamic Structures in Reproducing Kernel Hilbert Spaces}
An approach that learns dynamics in reproducing kernel Hilbert spaces (RKHSs) so that the resulting model satisfies the Euler-Lagrange equation
	 was proposed in \cite{cheng2016learn}, while our paper proposes a learning framework that adaptively captures control-affine structure in RKHSs to efficiently enforce control barrier certificates.
\subsubsection{Reinforcement Learning in Reproducing Kernel Hilbert Spaces}
We introduce, briefly, ideas of existing action-value function approximation techniques.
Given a policy $\phi:\X\rightarrow\U$, the action-value function $Q^{\phi}$ associated with the policy $\phi$ is defined as
\begin{align}
Q^{\phi}(\sigx,\phi(\sigx))=V^{\phi}(\sigx):=\sum_{n=0}^{\infty}\gamma^n R(\sigx_n,\phi(\sigx_n)),  \label{Qdef}
\end{align}	
where $\gamma\in(0,1)$ is the discount factor, $(\sigx_n)_{n\in\integer_{\geq0}}$ is a trajectory of the agent starting from $\sigx_0=\sigx$, and
$R(\sigx,\sigu)\in\Real$ is the immediate reward.
It is known that the action-value function follows the Bellman equation (c.f. {\cite[Equation~(66)]{lewis2009reinforcement}}):
	\begin{align}
		&{Q^{\phi}}(\sigx_n,\sigu_n)=\gamma{Q^{\phi}}(\sigx_{n+1},\phi(\sigx_{n+1}))+R(\sigx_n,\sigu_n). \label{bellman2}
	\end{align}
	For robotics applications, where the states and controls are continuous,
	some form of function approximators is required to approximate the action-value function (and/or policies).
	Nonparametric learning such as a kernel method is often desirable when {\it a priori} knowledge about a suitable set of basis functions for learning is unavailable.
	Kernel-based reinforcement learning has been studied in the literature, e.g., \cite{ormoneit2002kernel,xu2007kernel,taylor2009kernelized,sun2016online,nishiyama2012hilbert,grunewalder2012modelling,engel2005reinforcement,barreto2016practical,barreto2011reinforcement,kveton2012kernel,bae2011reinforcement,ormoneit2002kernel,reisinger2008online,cui2017kernel,van2015learning}.
	Due to the property of reproducing kernels, the framework of linear learning algorithms is
	directly applied to nonlinear function estimation tasks in a possibly infinite-dimensional functional space, namely a reproducing kernel Hilbert space.
	\begin{definition}[{\cite[page~343]{aronszajn50}}]
		Given a nonempty set $\mathcal{Z}$ and $\H$ which is a Hilbert space defined in $\mathcal{Z}$, the function $\ka{\sigz,\sigw}$ of $\sigz$ is called a reproducing kernel of $\H$ if
		\begin{enumerate}
			\item for every $\sigw\in\mathcal{Z}$, $\ka{\sigz,\sigw}$ as a function of $\sigz\in\mathcal{Z}$ belongs to $\H$, and
			\item it has the reproducing property, i.e., the following holds for every $\sigw\in\mathcal{Z}$ and every $\varphi\in\H$:
			\begin{equation}
			\varphi(\sigw)=\inpro{\varphi,\ka{\cdot,\sigw}}_{\H}. \nn
			\end{equation}
		\end{enumerate}
		If $\H$ has a reproducing kernel, $\H$ is called a Reproducing Kernel Hilbert Space (RKHS).
	\end{definition}
	One of the examples of kernels is the Gaussian kernel
	$\kappa(\sigx,\sigy):=
	\dfrac{1}{(2\pi\sigma^2)^{L/2}}
	\exp\left(-\dfrac{\norm{\sigx - \sigy}_{\real^{L}}^2}
	{2\sigma^2}\right)$, $\sigx,\sigy\in\Real^L$, $\sigma>0$.
	It is well-known that the Gaussian reproducing kernel Hilbert space has universality \cite{steinwart01}, i.e., any continuous function on every compact subset of $\Real^L$ can be
	approximated with an arbitrary accuracy.
	Another widely used kernel is the polynomial kernel
	$\kappa(\sigx,\sigy):=(\sigx^{\T}\sigy+c)^{d},\;c\geq0,d\in\integer_{>0}$.
%%%%%%

%%%%%%
In contrast to these existing approaches, we explicitly define a so-called reproducing kernel Hilbert space (RKHS) so that adaptive supervised learning
of action-value functions can be conducted in the same space without having to reset the learning.
	Consequently, we can also conduct an action-value function approximation in the same RKHS even after the agent dynamics changes or
 policies are updated (See the remark below Theorem \ref{theoRKHS} and Section \ref{experisimQfunc}).
	The GP SARSA can also be reproduced by employing a GP in the explicitly defined RKHS as is discussed in Appendix \ref{appencomp}.
Specifically, in this paper, a possibly nonstationary agent dynamics is considered as detailed below.
%%%%
\subsection{System Model}
\label{subsec:problemsettings}
%%%
In this paper, we consider the following discrete-time deterministic nonlinear model of the nonstationary agent dynamics,
\begin{align}
	\sigx_{n+1}-\sigx_n=p(\sigx_n,\sigu_n)+f(\sigx_n)+g(\sigx_n)\sigu_n, \label{genemodel}
\end{align}	
where $p:\X\x\U\rightarrow\Real^{n_x}$, $f:\X\rightarrow\Real^{n_x}$, $g:\X\rightarrow\Real^{n_x\x n_u}$ are continuous.
Hereafter, we regard $\X\x\U$ as the same as $\mathcal{Z}\subset\Real^{n_x+n_u}$ under the one-to-one correspondence between $\sigz:=[\sigx;\sigu]\in\mathcal{Z}$ and $(\sigx,\sigu)\in\X\x\U$ if there is no confusion.
%%%%%%%%%%

%%%%%%%%%%
We consider an agent with dynamics given in \refeq{genemodel}, and the goal is
to find an optimal policy which drives the agent to a desirable state {\em while remaining in the set of safe states (or the safe set)} 
$\mathcal{C}\subset\X$ defined as
\begin{align}
	\mathcal{C}:=\{\sigx\in\X|B(\sigx)\geq 0\}, \label{safeset}
\end{align}
where $B:\X\rightarrow\Real$.
An optimal policy is a policy $\phi$ that attains an optimal value $Q^{\phi}(\sigx,\phi(\sigx))$
for every state $\sigx\in\X$.
Note that the value associated with a policy varies when the dynamics is nonstationary, and that
a quadruple $(\sigx_n,\sigu_n,\sigx_{n+1},R(\sigx_n,\sigu_n))$ is available at each time instant $n$.
%%%%%%%%

%%%%%%%%
With these preliminaries in place, we can present our safe learning framework.
\section{Safe Learning Framework}
\label{sec:overview}
Under possibly nonstationary dynamics, our safe learning framework
adaptively estimates the long-term rewards to update policies with safety constraints.
Also, recovery of safety in the sense of Lyapunov stability during exploration is guaranteed
under certain conditions.
Define $\psi:\mathcal{Z}\rightarrow\Real$ as $\psi(\sigx,\sigu):=p(\sigx,\sigu)+f(\sigx)+g(\sigx)\sigu$,
and suppose that the estimator of $\psi$ at time instant $n$, denoted by $\hat{\psi}_n$, is approximated by the model parameter $\sigh_n\in\Real^r,\;r\in\integer_{>0}$ in the linear form as
\begin{equation}
\hat{\psi}_n(\sigz_n):=\sigh_n^{\T}\sigk(\sigz_n). \nn
\end{equation}
Here, $\sigk(\sigz_n)\in\Real^{r}$ is the output of basis functions at $\sigz_n$.
If the model parameter is accurately estimated (or the exact agent dynamics is available), the safe set $\mathcal{C}$ becomes
forward invariant and asymptotically stable by enforcing control barrier certificates at each time instant $n$.
\subsection{Discrete-time Control Barrier Functions}
The idea of control barrier functions is similar to Lyapunov functions;
they require no explicit computations of the forward reachable set while
ensuring certain properties by constraining the instantaneous control input.
Particularly, control barrier functions guarantee that an agent starting from the safe set
remains safe (i.e., forward invariance), and that an agent outside of the safe set is forced back to safety (i.e., Lyapunov stability with respect to the safe set).
To make barrier certificates compatible with model learning and reinforcement learning,
we employ the discrete-time control barrier certificates.
\begin{definition}[{\cite[Definition~4]{2017dcbf}}]
	A map $B:\X\rightarrow\Real$ is a discrete-time exponential control barrier function if 
	there exists a control input $\sigu_n\in\U$ such that 
	\begin{align}
		B(\sigx_{n+1})-B(\sigx_n)\geq -\eta B(\sigx_n), ~\forall n\in\integer_{\geq 0}, ~0< \eta \leq1. \label{DCBFdef2}
	\end{align}
\end{definition}
Note that we intentionally removed the condition $B(\sigx_0)\geq0$ originally presented in {\cite[Definition~4]{2017dcbf}}.
Then, the forward invariance and asymptotic stability with respect to the safe set are ensured by the following proposition.
\begin{proposition}
	\label{prpCBF}
	The set $\mathcal{C}$ defined in \refeq{safeset} for a valid discrete-time exponential control barrier function $B:\X\rightarrow\Real$ is forward invariant
	when $B(\sigx_0)\geq0$, and is asymptotically stable when $B(\sigx_0)<0$.
\end{proposition}	
\begin{proof}
	See Appendix \ref{appprp3-1}.
\end{proof}	
Proposition \ref{prpCBF} implies that an agent remains in the safe set defined in \refeq{safeset} for all time if $B(\sigx_0)\geq0$ and the inequality \refeq{DCBFdef2} are satisfied, and the agent outside of the safe set is brought back to safety.
%%%%%%

%%%%%%
The main motivations of using control barrier functions are given below:
\begin{enumerate}
	\renewcommand{\labelenumi}{\alph{enumi}).}
	\item Little modifications of policies: control barrier functions modify polices
	only when violations of safety are imminent.  Consequently, an inaccurate or rough estimation of the model causes less negative effect on (model-free) reinforcement learning.  
	\item Asymptotic stability of the safe set: the agent outside
	of the safe set is brought back to the safe set.
	In addition to Proposition \ref{prpCBF}, this robustness property is analyzed in \cite{xu2015robustness}.
	This property together with the adaptive model learning algorithm presented in the next subsection is particularly important when the safety is violated due to the nonstationarity of the agent dynamics.
\end{enumerate}
Under a possibly nonstationary agent dynamics, we can no longer guarantee that the current estimate of the model parameter
is sufficiently accurate to enforce the inequality \refeq{DCBFdef2} or forward invariance of $\mathcal{C}$.
Nevertheless, we are still able to show that safety is recovered in the sense of Lyapunov stability under certain conditions by {\em adaptively} learning the model.

\subsection{Adaptive Model Learning Algorithms with Monotone Approximation Property}
\label{subsec:overviewadaptive}
At each time instant, an input-output pair $(\sigz_n,\delta_n)$, where $\sigz_n:=[\sigx_n;\sigu_n]$ and $\delta_n:=\sigx_{n+1}-\sigx_n$ for model learning
	is available.  
	Under possibly nonstationary agent dynamics, it is vital for the model parameter estimation to be stable even
	after the agent dynamics changes.
	In this paper, we employ an adaptive algorithm with monotone approximation property.
	Note this approach shares a similar idea with stable online learning based on Lyapunov-like conditions.
%%%%

%%%%
	Suppose that the estimate of model parameter at time instant $n$ is given by $\sigh_n\in\Real^{r},\;r\in\integer_{>0}$.
	Given a cost function $\Theta_n(\sigh)$ at time instant $n$, we update the parameter $\sigh_n$ so as to satisfy
	the strictly monotone approximation property $\norm{\sigh_{n+1}-\sigh_n^{*}}_{\Real^{r}}<\norm{\sigh_{n}-\sigh_n^{*}}_{\Real^{r}},\;\forall \sigh_n^{*}\in\Omega_n:=\argmin_{\sigh\in\Real^{r}}\Theta_n(\sigh)$
	if $\sigh_n\notin\Omega_n\neq\emptyset$, where $\emptyset$ is the empty set.
Then, if
$\Omega:=\bigcap_{n\in\integer_{\geq0}}\Omega_n$ is nonempty and if $\sigh_n\notin\Omega_n$,
it follows that $\norm{\sigh_{n+1}-\sigh^{*}}_{\Real^{r}}<\norm{\sigh_{n}-\sigh^{*}}_{\Real^{r}},\;\forall \sigh^{*}\in\Omega,\;n\in\integer_{\geq0}$.
This is illustrated in \reffig{fig:monotone}.
\begin{figure}[t]
	\begin{center}
		\includegraphics[clip,width=0.35\textwidth]{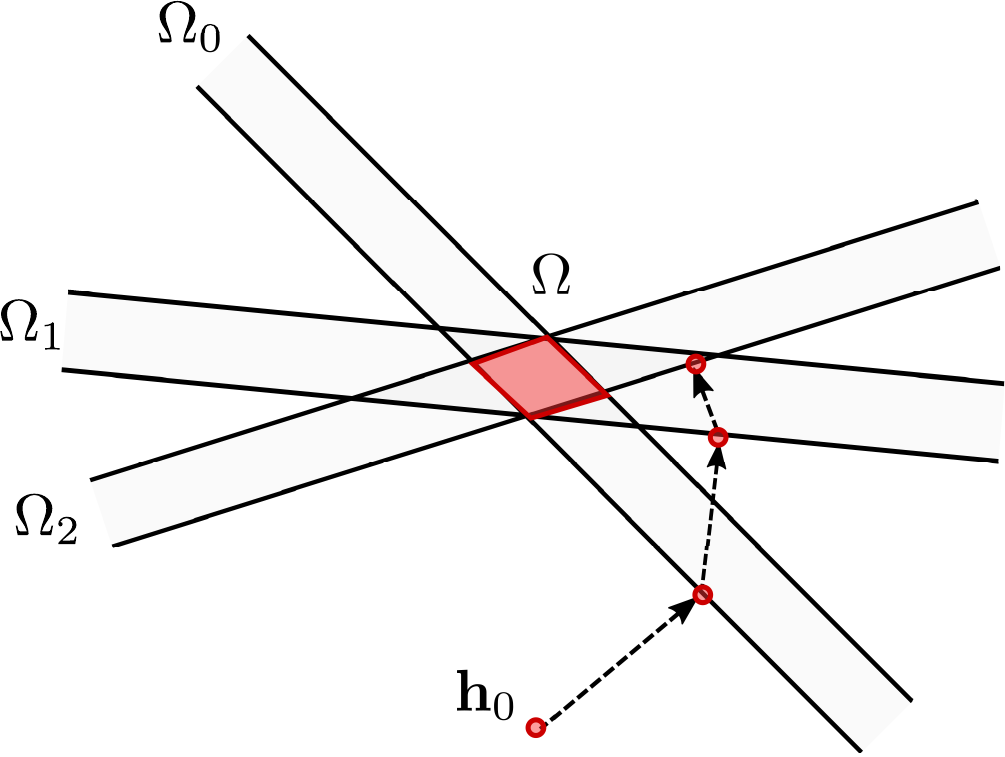}
		\caption{An illustration of the monotone approximation property.  The estimate $\sigh_n$ monotonically approaches to the set $\Omega$ of optimal vectors $\sigh^{*}$
			by sequentially minimizing the distance between $\sigh_n$ and $\Omega_n$.  Here, $\Omega_n:=\argmin_{\sigh\in\Real^{r}}\Theta_n(\sigh)$, where $\Theta_n(\sigh)$
			is the cost function at time instant $n$.}
		\label{fig:monotone}
	\end{center}
\end{figure}
Under mild conditions, we can also design algorithms (e.g., the adaptive projected subgradient method \cite{APSM1}) that satisfy $\norm{\sigh_{n}-\sigh_n^{*}}^2_{\Real^{r}}-\norm{\sigh_{n+1}-\sigh_n^{*}}^2_{\Real^{r}}\geq\varrho^2_3 dist^2(\sigh_n,\Omega_n)$, for all $\sigh_n^{*}\in\Omega_n$, and for some $\varrho_3>0$, where $dist(\sigh_n,\Omega_n):=\inf\{\norm{\sigh_n-\sigh^{*}_n}_{\Real^{r}}|\sigh^{*}_n\in\Omega_n\}$. (See \cite{APSM1} for more detailed arguments for example.)
%%%%%%

%%%%%%
At each time instant, we use the current estimate of the model to constrain control inputs so that they satisfy
\begin{align}
&B(\hat{\sigx}_{n+1})-B(\sigx_n)\geq -\eta B(\sigx_n)+\varrho_1,\;\forall n\in\integer_{\geq0},~0< \eta \leq1, \nn
\end{align}
for some margin $\rho_1>0$,
where $\hat{\sigx}_{n+1}$ is the predicted output of the current estimate $\sigh_n$ at $\sigx_n$ and $\sigu_n$.
Then, under certain conditions, we can guarantee Lyapunov stability of the system for the augmented state $[\sigx_n;\sigh_n]\in\Real^{n_x+r}$ with respect to the forward invariant set $\mathcal{C}\x\Omega\subset\Real^{n_x+r}$ as illustrated in \reffig{fig:augment}.
\begin{figure}[t]
	\begin{center}
		\includegraphics[clip,width=0.35\textwidth]{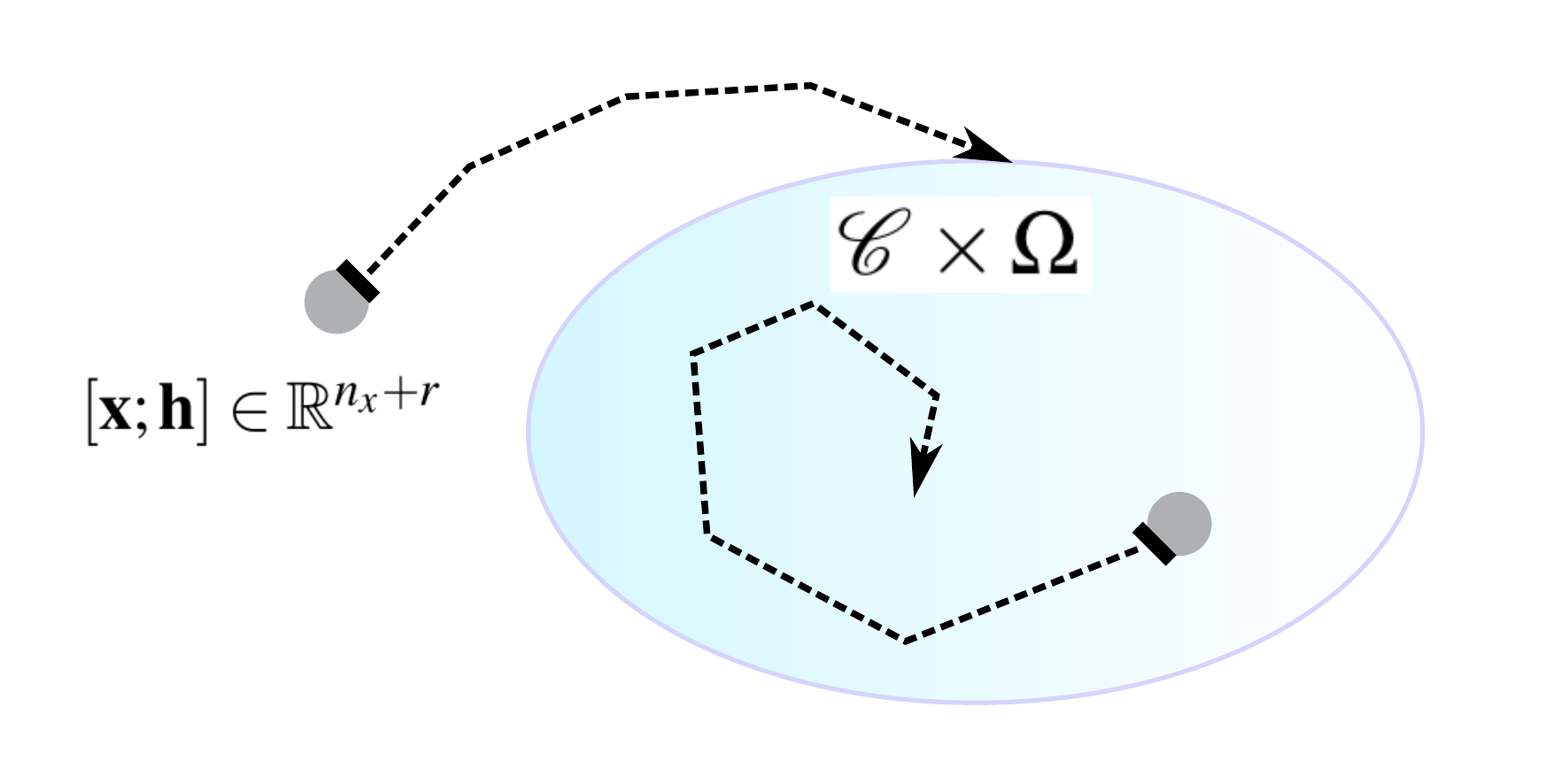}
		\caption{An illustration of Lyapunov stability of the system for the augmented state $[\sigx;\sigh]\in\Real^{n_x+r}$ with respect to the forward invariant set $\mathcal{C}\x\Omega\subset\Real^{n_x+r}$.}
		\label{fig:augment}
	\end{center}
\end{figure}
In Sections \ref{subsec:safeman} and \ref{sec:numerical}, we will theoretically and experimentally show that the system for the augmented state is stable on the set of augmented safe states.
%%%%%%

%%%%%%
To efficiently constrain policies by using control barrier functions, the learned model
is preferred to be affine in control.  (see Section \ref{subsec:adaptiveRL} and Theorem \ref{CBFaffine} therein.)
As such, outputs of the learned model should have preferred dynamic structures while capturing
the true agent dynamics.
%%%%%%%%%%
\subsection{Leaning Dynamic Structure via Sparse Optimizations}
\label{subsec:overviewdynamic}
Control-affine dynamics is given by \refeq{genemodel} with $p=0$, where $0$ denotes the null function.
Therefore, the simplest way is to learn the agent dynamics with the constraint $p=0$.
In practice, however, it is unrealistic to assume that $p=0$ due to
the effects of frictions and other disturbances.
Instead, as long as the term $p$ is negligibly small, we can consider $p$ to be a {\em system noise} added to a control-affine dynamics.
To encourage the term $p$ to be as small as possible while capturing the true input-output relations of the agent dynamics,
we use adaptive sparse optimization techniques.
In particular, motivated by the monotone approximation property due to convexity of the formulations, we use (sparse) kernel adaptive filters for the systems with nonlinear dynamics.
Specifically, we take the following steps to extract the control-affine structure:
\begin{enumerate}
	\item Assume for simplicity that $n_x=1$.  We suppose that $p\in\H_p$, $f\in\H_f$, and $g^{(1)},g^{(2)},...,g^{(n_u)}\in\H_g$, where $\H_p$, $\H_f$ and $\H_g$ are RKHSs, and $g(\sigx)=[g^{(1)}(\sigx),g^{(2)}(\sigx),\cdots,g^{(n_u)}(\sigx)]$.
	\item Let $\H_u$ be the RKHS associated with the reproducing kernel $\ka{\sigu,\sigv}:=\sigu^{\T}\sigv,\;\sigu,\sigv\in\U$,% i.e., a polynomial kernel with $c=0$ and $d=1$,
	and $\H_c$ the set of constant functions on $\U$.
	Estimate the function $\psi$ in the RKHS $\H_{\psi}:=\H_p+\H_f\otimes\H_c+\H_g\otimes\H_u$ (see Section \ref{subsec:modellearn} and Theorem \ref{directtheorem} therein).
	\item Define the cost $\Theta_n$ so as to promote sparsity of the model parameter.  If the underlying true dynamics is affine in control,
	a control-affine model (i.e., the estimate of $p$ denoted by $\hat{p}$ becomes null) is expected to be extracted.
\end{enumerate}	
The resulting control-affine part of the estimated dynamics is used in combination with control barrier certificates 
in order to efficiently constrain policies while and after learning an optimal policy. (see Theorem \ref{maintheo} and Theorem \ref{CBFaffine} for more details.) 
%%%%

%%%%
\subsection{Barrier-certified Policy Update}
Lastly, we present the barrier-certified policy update strategy.
To update policies, we use the long-term rewards that needs to be adaptively estimated for systems with possibly nonstationary agent dynamics.
\subsubsection{Adaptive Action-value Function Approximation in RKHSs}
\label{subsec:overviewbellman}
Again, motivated by the monotone approximation property (see Corollary \ref{theoNonex}) and the flexibility of nonparametric learning that requires no fixed set of basis functions, we employ kernel-based adaptive algorithms to estimate the action-value function. 
One of the issues arising when applying a kernel-based method to an action-value function approximation
is that the output of the action-value function $Q^{\phi}(\sigx_n,\sigu_n)\in\H_Q$ associated with a policy $\phi$, where $\H_{Q}$ is assumed to be an RKHS, is unobservable.
Nevertheless, we know that the action-value function follows the Bellman equation \refeq{bellman2}.
Hence, by defining a function $\psi^Q:\mathcal{Z}^2\rightarrow\Real$,
where $\Real^{2(n_x+n_u)}\supset\mathcal{Z}^2=\mathcal{Z}\x\mathcal{Z}$, as
\begin{align}
&\psi^Q([\sigz;\sigw]):={Q^{\phi}}(\sigx,\sigu)-\gamma {Q^{\phi}}(\sigy,\sigv),\label{psidef}\\
&\sigx,\sigy\in\X,\;\sigu,\sigv\in\U,\;\sigz=[\sigx;\sigu],\;\sigw=[\sigy;\sigv], \nn
\end{align}	
the Bellman equation in \refeq{bellman2} is solved via iterative nonlinear function estimation with the input-output pairs $\{([\sigx_n;\sigu_n;\sigx_{n+1};\phi(\sigx_{n+1})],R(\sigx_n,\sigu_n))\}_{n\in\integer_{\geq0}}$.
In fact, the function $\psi^Q$ is an element of a properly constructed RKHS $\H_{\psi^Q}$ (see Section \ref{subsec:adaptiveRL} and Theorem \ref{theoRKHS} therein).
Because the domain of $\H_{\psi^Q}$ is defined as $\mathcal{Z}\x\mathcal{Z}$ instead of $\mathcal{Z}$, the RKHS $\H_{\psi^Q}$ does not depend on the agent dynamics.
Therefore, we do not have to reset learning even after the dynamics changes or the policy is updated, and we can analyze
convergence and/or monotone approximation property of an action-value function approximation in the same RKHS (see Section \ref{experisimQfunc}, for example).

\subsubsection{Policy Update}
For a current policy $\phi:\X\rightarrow\U$, assume that the action-value function $Q^{\phi}$
with respect to $\phi$ at time instant $n$ is available.
Given a discrete-time exponential control barrier function $B$ and $0< \eta\leq1$, the barrier certified safe control space is define as
\begin{align}
\hspace{-0.5em}\mathcal{S}(\sigx_n):=\{\sigu_n\in\U|B(\sigx_{n+1})-B(\sigx_n)\geq -\eta B(\sigx_n)\}. \nn
\end{align}
From Proposition \ref{prpCBF}, the set $\mathcal{C}$ defined in \refeq{safeset} is forward invariant and asymptotically stable if
$\sigu_n\in \mathcal{S}(\sigx_n)$ for all $n\in\integer_{\geq0}$.
Then, the updated policy $\phi^+$ given by
\begin{align}
\phi^+(\sigx):=\argmax_{\sigu\in{\mathcal{S}}(\sigx)}\left[Q^{\phi}(\sigx,\sigu)\right], \label{improvepolicy}
\end{align}
is well-known (e.g., \cite{puterman1979convergence,bertsekas2005dynamic}) to satisfy that $Q^{\phi}(\sigx,\phi(\sigx))\leq Q^{\phi^+}(\sigx,\phi^+(\sigx))$,
where $Q^{\phi^+}$ is the action-value function with respect to $\phi^+$.
In practice, we use the estimate of $Q^{\phi}$ because the exact function $Q^{\phi}$ is unavailable.
For example, the action-value function is estimated over $N_f\in\integer_{>0}$ iterations, and the policy is updated every $N_f$ iterations.
\section{Analysis of Barrier-certified Adaptive Reinforcement Learning}
\label{sec:proposed}
In the previous section, we presented our barrier-certified adaptive reinforcement learning framework.
In this section, we present theoretical analysis of our framework to further strengthen the arguments.
\subsection{Safety Recovery: Adaptive Model Learning and Control Barrier Certificates}
\label{subsec:safeman}
The monotone approximation property of model parameters is closely related to Lyapunov stability.
In fact, by augmenting the state vector with the model parameter, we can construct
a Lyapunov function which guarantees stability with respect to the safe set under certain conditions.
%%%%%%

%%%%%%
We first make following assumptions.
\begin{assumption}
	\label{assumstab}
	\begin{enumerate}
	\item Finite-dimensional model parameter: the dimension of model parameter $\sigh$ remains finite, and is $r\in\integer_{>0}$.
	%\item The input space $\X$ is invariant.
	\item Boundedness of the basis functions: all of the basis functions (or kernel functions) are bounded over $\X$.
	\item Lipschitz continuity of the control barrier function: the control barrier function $B$ is Lipschitz continuous over $\X$ with Lipschitz constant $\nu_B$.
	\item Validity of barrier certificates: there exists a control input $\sigu_n\in\U$ satisfying for a sufficiently small $\varrho_1>0$ that
	\begin{align}
	&B(\hat{\sigx}_{n+1})-B(\sigx_n)\geq -\eta B(\sigx_n)+\varrho_1,\nn\\
	&~~~~~~~~~~~~~~~~ \forall n\in\integer_{\geq0},~0< \eta \leq1, \label{DCBFdef3}
	\end{align}
	where $\hat{\sigx}_{n+1}$ is the predicted output of the current estimate $\sigh_n$ at $\sigx_n$ and $\sigu_n$.
	\item Appropriate cost functions: if $\sigh_n\in\Omega_n:=\argmin_{\sigh\in\Real^{r}}\Theta_n(\sigh)$, where $\Theta_n(\sigh)$
	is the continuous cost function at time instant $n$, then $\norm{\sigx_{n+1}-\hat{\sigx}_{n+1}}_{\Real^{n_x}}\leq \frac{\varrho_1}{\nu_B}$.
	\item Model learning with monotone approximation property: model parameter $\sigh_n$ is updated as $\sigh_{n+1}=T_n(\sigh_n)$, where $T_n:\Real^{r}\rightarrow\Real^{r}$
	is continuous and has monotone approximation property: if $\sigh_n\notin\Omega_n$, then $dist^2(\sigh_n,\Omega_n)\geq\varrho^2_2$ and $\norm{\sigh_{n}-\sigh_n^{*}}^2_{\Real^{r}}-\norm{\sigh_{n+1}-\sigh_n^{*}}^2_{\Real^{r}}\geq\varrho^2_3 dist^2(\sigh_n,\Omega_n)$, for all $\sigh_n^{*}\in\Omega_n$,
	and for some $\varrho_2,\varrho_3>0$.  If $\sigh_n\in\Omega_n$, then $\sigh_{n+1}=\sigh_n$.
	\item Data consistency: The set $\Omega:=\bigcap_{n\in\integer_{\geq 0}}\Omega_n$ is nonempty.
	\end{enumerate}
\end{assumption}	
\begin{remark}[On Assumption \ref{assumstab}.1]	
	Assumption \ref{assumstab}.1 is made so that Lyapunov stability can be analyzed in an Euclidean space and is reasonable if polynomial kernels are employed for learning or if the input space $\mathcal{Z}:=\X\x\U$ is compact.
\end{remark}	
\begin{remark}[On Assumptions \ref{assumstab}.2 and \ref{assumstab}.3]	
	Assumptions \ref{assumstab}.2 and \ref{assumstab}.3 ensure that the predicted value of the barrier function
		is close to its true value if the current estimate of model parameter is close to the true parameter.
\end{remark}	
\begin{remark}[On Assumption \ref{assumstab}.4]	
	Assumption \ref{assumstab}.4 implies that we can enforce barrier certificates for the current estimate of the dynamics with a
	sufficiently small margin $\varrho_1$.
	This assumption is necessary to implicitly bound the growth of $B(\sigx_{n+1})$ and to robustly enforce barrier certificates whenever $\sigh_n\in\Omega_n$.
	Although this assumption is somewhat restrictive, it is still reasonable if the initial estimate does not largely deviate from the true dynamics.
\end{remark}	
\begin{remark}[On Assumption \ref{assumstab}.5]	
	Assumption \ref{assumstab}.5 implies that the set $\Omega_n$ or equivalently the cost $\Theta_n$
	is designed so that the predicted output $\hat{\sigx}_{n+1}$ for $\sigh_n\in\Omega_n$ is sufficiently close to the true output $\sigx_{n+1}$.  Such a cost can be easily designed.
	This assumption is necessary to render the set $\mathcal{C}\x\Omega$ forward invariant.
\end{remark}	
\begin{remark}[On Assumptions \ref{assumstab}.6 and \ref{assumstab}.7]	
	To apply theories of Lyapunov stability, Assumption \ref{assumstab}.6 is needed to make sure that
	the dynamical system for the augmented state is continuous.  Moreover, the cost (or the set $\Omega_n$) is designed so that
	$\sigh_n\in\Omega_n$ or $dist^2(\sigh_n,\Omega_n)\geq\varrho^2_2$.  See the work in \cite{APSM1} for a class of algorithms that satisfy this property, for example.
	Unless there exist some adversarial data (or inappropriate costs) that do not reflect the true agent dynamics,
	Assumption \ref{assumstab}.7 is valid and ensures that the set of augmented safe states is nonempty.  
\end{remark}	
Let the augmented state be $[\sigx;\sigh]\in\Real^{n_x+r}$.
Then, the following theorem states that the system for the augmented state is (asymptotically) stable with respect to the set of augmented safe states even after a violation of safety due to the abrupt and unexpected change of the agent dynamics occurs.
\begin{theorem}
	\label{maintheo}
	Suppose that a triple $(\sigx_n,\sigu_n,\sigx_{n+1})$ is available at time instant $n+1$.
Suppose also that a control input $\sigu_n$ satisfying \refeq{DCBFdef3} is employed for all $n\in\integer_{\geq0}$.
Then, under Assumption \ref{assumstab}, the system for the augmented state is stable with respect to the set of augmented safe states $\mathcal{C}\x\Omega\subset\Real^{n_x+r}$.
If, in addition, $\sigh_n\notin\Omega_n$ for all $n\in\integer_{\geq0}$ such that $[\sigx_n;\sigh_n]\notin\mathcal{C}\x\Omega$, then the system is uniformly globally asymptotically stable
with respect to $\mathcal{C}\x\Omega\subset\Real^{n_x+r}$.
\end{theorem}	
\begin{proof}
	See Appendix \ref{applyapu}.
\end{proof}	
\begin{remark}[On Theorem \ref{maintheo}]
	Theorem \ref{maintheo} implies that how much the current estimate gets closer to the true dynamics depends on how much the next state of the agent is deviated from the predicted next state.
	Therefore, both barrier certificates and model learning work together to guarantee stability.
	If the model learning algorithm satisfies Assumption \ref{assumstab}, then Theorem \ref{maintheo} claims
	that safety is recovered successfully.
	When GPs or kernel ridge regressions are employed for model learning, for example, introducing forgetting factors or
	letting the sample size grow as time advances will make the algorithms adaptive to time-varying systems;
	in such cases, we need to make sure that the algorithms satisfy Assumption \ref{assumstab} to guarantee safety recovery.
	Numerical simulations about safety recovery is given in Section \ref{subsec:simulate}.
\end{remark}
If the agent dynamics keeps changing or if we know that there are multiple modes for dynamics, then we may have separate model learning processes as proposed in \cite{mckinnon2018experience},
and the augmented state can be regarded as following a hybrid system.  Hence, stability should be analyzed under additional assumptions in this case.
We leave such an analysis as a future work.
%%%%%
\subsection{Structured Model Learning}
\label{subsec:modellearn}
We have seen that, by employing a model learning with monotone approximation property under Assumption \ref{assumstab}, the agent is stabilized on the set of augmented safe states even after an abrupt and unexpected change of the agent dynamics.
Here, we show that a control-affine dynamics can be learned via sparse optimizations satisfying monotone approximation property in a properly defined RKHS.
We assume that $n_x=1$ for simplicity (we can employ $n_x$ approximators if $n_x>1$).
%%%%%%

%%%%%%
First, we show that the space $\H_c$ (see Section \ref{subsec:overviewdynamic}) is an RKHS.
\begin{lemma}
	\label{prp1}
	The space $\H_c$
	is an RKHS associated with the reproducing kernel $\kappa(\sigu,\sigv)=\textbf{1}(\sigu):=1,\forall \sigu,\sigv\in\U$,
	with the inner product defined as $\inpro{\alpha\textbf{1},\beta\textbf{1}}_{\H_c}:=\alpha\beta$, $\alpha,\beta\in\Real$.
\end{lemma}
\begin{proof}
	See Appendix \ref{appprp1}.
\end{proof}	
Then, the following lemma implies that $\psi$ can be approximated in the sum space of RKHSs denoted by $\H_{\psi}$.
\begin{lemma}[{\cite[Theorem~13]{moor1}}]
	\label{prp2}
	Let $\H_1$ and $\H_2$ be two RKHSs associated with the reproducing 
	kernels $\kappa_1$ and $\kappa_2$.  Then the completion of the tensor product of $\H_1$ and $\H_2$, denoted by $\H_1\otimes\H_2$, is an RKHS associated with the reproducing kernel $\kappa_1\otimes\kappa_2$.
\end{lemma}
From Lemmas \ref{prp1} and \ref{prp2}, we can now assume that $\hat{f}\in\H_f\otimes\H_c$ and $\hat{\tilde{g}}\in\H_g\otimes\H_u$,
where $\hat{\tilde{g}}$ is an estimate of $\tilde{g}(\sigx,\sigu):=g(\sigx)\sigu$.
As such, $\psi$ can be approximated in the RKHS $\H_{\psi}:=\H_p+\H_f\otimes\H_c+\H_g\otimes\H_u$.
Therefore, we can employ a kernel adaptive filter working in the sum space $\H_{\psi}$.
%%%%%%

%%%%%%
Second, the following theorem ensures that $\psi$ can be uniquely decomposed into $p$, $f$, and $\tilde{g}$ in the RKHS $\H_{\psi}$.
\begin{theorem}
	\label{directtheorem}
	Assume that $\X$ and $\U$ have nonempty interiors.
	Assume also that $\H_p$ is a Gaussian RKHS.
	Then, $\H_{\psi}$ is the direct sum of $\H_p$, $\H_f\otimes\H_c$, and $\H_g\otimes\H_u$, i.e.,
	the intersection of any two of the RKHSs $\H_p$, $\H_f\otimes\H_c$, and $\H_g\otimes\H_u$ is $\{0\}$.
\end{theorem}	
\begin{proof}
	See Appendix \ref{apptheo1}.
\end{proof}	
\begin{remark}[On Theorem \ref{directtheorem}]
	Because only the control-affine part of the learned model is used in combination with barrier certificates (see Assumption \ref{assump1}
	and Theorem \ref{CBFaffine}) and the term $p$ is assumed to be a system noise added to the control-affine dynamics,
	the unique decomposition is crucial; if the unique decomposition does not hold, the term $p$ may be able to
	estimate the overall dynamics, including the control-affine terms.
\end{remark}
By using a sparse optimization for the coefficient vector $\sigh_n\in\Real^{r}$,  
we wish to extract a structure of the model;
	from Theorem \ref{directtheorem}, the term $\hat{p}_n$ is expected to drop off when the true agent dynamics is affine in control.

%%%%
In order to use the learned model in combination with control barrier functions, each entry of the vector $\hat{g}_n(\sigx_n)$ is required.
Assume, without loss of generality, that $\{\textbf{e}_i\}_{i\in\{1,2,...,n_u\}}\subset\U$ (this is always possible for $\U\neq\emptyset$ by transforming
coordinates of the control inputs and reducing the dimension $n_u$ if necessary).
Then, the $i$th entry of the vector $\hat{g}_n(\sigx_n)$ is given by
$\hat{g}_n(\sigx_n)\textbf{e}_i=\hat{\tilde{g}}_n(\sigx_n,\textbf{e}_i)$.
As such, we can use the learned model to constrain control inputs efficiently by using control barrier functions for explorations as well as policy updates.
We analyze an adaptive action-value function approximation with barrier-certified policy updates in the next subsection.
%%%%%%

%%%%%%
%%%%%%%%
%%%%%%
%%%%%
\subsection{Adaptive Action-value Function Approximation with Barrier-certified Policy Updates}
\label{subsec:adaptiveRL}
In this subsection, we analyze the proposed adaptive action-value function approximation with barrier-certified policy updates.
%%%%

%%%%
We showed in Section \ref{subsec:overviewbellman} that the Bellman equation in \refeq{bellman2} is solved via iterative nonlinear function estimation with the input-output pairs $\{([\sigx_n;\sigu_n;\sigx_{n+1};\phi(\sigx_{n+1})],R(\sigx_n,\sigu_n))\}_{n\in\integer_{\geq0}}$.
The following theorem states that the function $\psi^Q$ defined in \refeq{psidef} can be estimated in a properly constructed RKHS.
\begin{theorem}
	\label{theoRKHS}
	Suppose that $\H_Q$ is an RKHS associated with the reproducing kernel $\kappa^Q(\cdot,\cdot):\mathcal{Z}\x\mathcal{Z}\rightarrow\Real$.
	Define, for $\gamma\in(0,1)$,
	\begin{align}
	&\H_{\psi^Q}:=\{\varphi|\varphi([\sigz;\sigw])=\varphi^Q(\sigz)-\gamma \varphi^Q(\sigw),\nn\\
	&\;\exists \varphi^Q\in\H_{Q},\;\forall \sigz,\sigw\in\mathcal{Z}\}.\nn
	\end{align}
	Then, the operator $U:\H_{Q}\rightarrow\H_{\psi^Q}$ defined by $U(\varphi^Q)([\sigz;\sigw]):=\varphi^Q(\sigz)-\gamma \varphi^Q(\sigw),\;\forall \varphi^Q\in\H_Q$,
	is bijective.
	Moreover, $\H_{\psi^Q}$ is an RKHS with the inner product defined by
	\begin{align}
	&\hspace{-1em}\inpro{\varphi_1,\varphi_2}_{\H_{\psi^Q}}:=\inpro{\varphi_1^Q,\varphi_2^Q}_{\H_{Q}},\label{theoinpro}\\
	&\hspace{-1em}\varphi_i([\sigz;\sigw]):=\varphi_i^Q(\sigz)-\gamma\varphi_i^Q(\sigw), \;\forall \sigz,\sigw\in\mathcal{Z},\;i\in\{1,2\}. \nn
	\end{align}
	The reproducing kernel of the RKHS $\H_{\psi^Q}$ is given by
	\begin{align}
	\kappa&([\sigz;\sigw],[\tilde{\sigz};\tilde{\sigw}]):=\left(\kappa^Q(\sigz,\tilde{\sigz})-\gamma\kappa^Q(\sigz,\tilde{\sigw})\right)\nn\\
	&-\gamma\left(\kappa^Q(\sigw,\tilde{\sigz})-\gamma\kappa^Q(\sigw,\tilde{\sigw})\right),\;\sigz,\sigw,\tilde{\sigz},\tilde{\sigw}\in\mathcal{Z}. \label{kapa}
	\end{align}
\end{theorem}	
\begin{proof}
	See Appendix \ref{apptheo2}.
\end{proof}	
From Theorem \ref{theoRKHS}, we can use any kernel-based method by assuming that the action-value function is in $\H_{Q}$.
The estimate of $Q^{\phi}$ denoted by $\hat{Q}^{\phi}$ is obtained by $U^{-1}(\hat{\psi}^Q)$, where $\hat{\psi}^Q$ is the estimate of ${\psi}^Q\in\H_{\psi^Q}$.
For instance, suppose that the estimate of ${\psi}^Q(\sigz,\sigw)$ for an input $[\sigz;\sigw]$ at time instant $n$ is given by
\begin{align}
\hat{\psi}_n^Q([\sigz;\sigw]):={\sigh_n^Q}^{\T}\sigk([\sigz;\sigw]),\nn
\end{align}
where $\sigh_n^Q\in\Real^{r}$ is the model parameter, and $\sigk([\sigz;\sigw]):=\left[\kappa\left([\sigz;\sigw],[\tilde{\sigz}_{1};\tilde{\sigw}_{1}]\right);\kappa\left([\sigz;\sigw],[\tilde{\sigz}_{2};\tilde{\sigw}_{2}]\right);\cdots;\kappa\left([\sigz;\sigw],[\tilde{\sigz}_{r};\tilde{\sigw}_{r}]\right)\right]\in\Real^{r}$
for $\{\tilde{\sigz}_j\}_{j\in\{1,2,...,r\}},\{\tilde{\sigw}_j\}_{j\in\{1,2,...,r\}}\subset\mathcal{Z}$ and for $\kappa(\cdot,\cdot)$ defined by \refeq{kapa}.
Then, the estimate of $Q^{\phi}(\sigz)$ for an input $\sigz$ at time instant $n$ is given by
\begin{align}
\hat{Q}_n^{\phi}(\sigz):={\sigh_n^Q}^{\T}\sigk^Q(\sigz), \label{Qest}
\end{align}
where $\sigk^Q(\sigz):=$\\
$\left[U^{-1}\left(\kappa\left(\cdot,[\tilde{\sigz}_{1};\tilde{\sigw}_{1}]\right)\right)(\sigz);\cdots;U^{-1}\left(\kappa\left(\cdot,[\tilde{\sigz}_{r};\tilde{\sigw}_{r}]\right)\right)(\sigz)\right]\in\Real^{r}$.
\begin{remark}[On Theorem \ref{theoRKHS}]
	As discussed in Appendix \ref{appencomp}, the GP SARSA is reproduced by applying a GP in the space $\H_{\psi^Q}$, although the GP SARSA or other kernel-based action-value function approximation is ad-hoc and designed for estimating the action-value function associated with a fixed policy under a stationary agent dynamics.
\end{remark}
	When the parameter $\sigh_n^Q$ for the estimator $\hat{\psi}_n^Q$ is monotonically approaching to an optimal point ${\sigh^Q}^{*}$ in
	the Euclidean norm sense,
	so is the model parameter for the action-value function because the same parameter is used to estimate $\psi^Q$ and $Q^{\phi}$.
	Suppose we employ a method which monotonically brings $\hat{\psi}_n^Q$ closer to an optimal function ${\psi^Q}^{*}$ in the {\em Hilbertian} norm sense.
	Then, the following corollary implies that an estimator of the action-value function also satisfies the monotonicity.
	\begin{corollary}
		\label{theoNonex}
		Let $\H_{\psi^Q}\ni\hat{\psi}_n^Q([\sigz;\sigw]):=\hat{Q}_n^{\phi}(\sigz)-\gamma\hat{Q}_n^{\phi}(\sigw)$ and
		$\H_{\psi^Q}\ni{\psi^Q}^{*}([\sigz;\sigw]):={Q^{\phi}}^{*}(\sigz)-\gamma{Q^{\phi}}^{*}(\sigw),\;\sigz,\sigw\in\mathcal{Z}$, where
		$\hat{Q}_n^{\phi},{Q^{\phi}}^{*}\in\H_Q$.
		Then, if $\hat{\psi}_n^Q$ is approaching to ${\psi^Q}^{*}$, i.e.,%$\hat{\psi}_n^Q$ is approaching to ${\psi^Q}^{*}$ in the Hilbertian norm sense, i.e.,
		$\norm{\hat{\psi}_{n+1}^Q-{\psi^Q}^{*}}_{\H_{\psi^Q}}\leq\norm{\hat{\psi}_n^Q-{\psi^Q}^{*}}_{\H_{\psi^Q}}$,
		it follows that
		$\norm{\hat{Q}^{\phi}_{n+1}-{Q^{\phi}}^{*}}_{\H_{Q}}\leq\norm{\hat{Q}_n^{\phi}-{Q^{\phi}}^{*}}_{\H_{Q}}.$
	\end{corollary}	
	\begin{proof}
		See Appendix \ref{apptheo3}.
	\end{proof}	
Note that the use of action-value functions enables us to use random control inputs instead of the target policy $\phi$ for exploration,
and we require no models of the agent dynamics for policy updates as discussed below.
%%%

%%%%

To obtain analytical solutions for \refeq{improvepolicy},
we follow the arguments in \cite{engel2005reinforcement}.
Suppose that $\hat{Q}_n^{\phi}$ is given by \refeq{Qest}.
We define the reproducing kernel $\kappa^Q$ of $\H_Q$ as the tensor kernel given by
\begin{align}
\kappa^Q([\sigx;\sigu],[\sigy;\sigv]):=\kappa^x(\sigx,\sigy)\kappa^u(\sigu,\sigv), \label{kernel1des}
\end{align}	
where $\kappa^u(\sigu,\sigv)$ is, for example, defined by
\begin{align}
\kappa^u(\sigu,\sigv):=1+\frac{1}{4}(\sigu^{\T}\sigv). \nn
\end{align}
Then, \refeq{improvepolicy} becomes
\begin{align}
\phi^+(\sigx):=\argmax_{\sigu\in{\mathcal{S}}(\sigx)}\left[{\sigh_n^Q}^{\T}\sigk^Q([\sigx;\sigu])\right], \label{maxprob}
\end{align}
where the target value being maximized is linear to $\sigu$ at $\sigx$.
Therefore, if the set ${\mathcal{S}}(\sigx)\subset\U$ is convex,
an optimal solution to \refeq{maxprob} is guaranteed to be globally optimal, ensuring the greedy improvement of the policy.
%%%

As pointed out in \cite{2017dcbf}, $\mathcal{S}(\sigx)\subset\U$ is not a convex set in general.
Instead, we consider a convex subset of $\mathcal{S}(\sigx)$ under the following moderate assumptions:
\begin{assumption}
	\label{assump1}
	\begin{enumerate}
		\item The set $\U$ is convex.
		\item Existence of Lipschitz continuous gradient of the barrier function:
	Given
	\begin{align}
	&\hspace{-3.5em}\mathcal{R}:=\{(1-t)\sigx_n+t(\hat{f}_n(\sigx_n)+\hat{g}_n(\sigx_n)\sigu)|t\in[0,1],\sigu\in\U\},  \nn
	\end{align}
	there exists a constant $\nu\geq0$ such that the gradient of the discrete-time exponential control barrier function $B$, denoted by $\frac{\partial B(\sigx)}{\partial \sigx}$,
	satisfies 
	\begin{align}
\hspace{-1.5em}\norm{\frac{\partial B(\textbf{a})}{\partial \sigx}-\frac{\partial B(\textbf{b})}{\partial \sigx}}_{\Real^{n_x}}\leq \nu\norm{\textbf{a}-\textbf{b}}_{\Real^{n_x}},\;\forall \textbf{a},\textbf{b}\in \mathcal{R}.\nn
	\end{align}
	\end{enumerate}
\end{assumption}
Then, the following theorem holds.
\begin{theorem}
	\label{CBFaffine}
	 Under Assumptions \ref{assumstab}.3 and \ref{assump1}, assume also that $\norm{\sigx_{n+1}-(\hat{f}_n(\sigx_n)+\hat{g}_n(\sigx_n)\sigu_n+\sigx_n)}_{\Real^{n_x}}\leq\frac{\varrho_1}{\nu_B}$.
	 Then, inequality \refeq{DCBFdef2} is satisfied at time instant $n\in\integer_{\geq0}$ if $\sigu_n$ satisfies the following:
	\begin{align}
	&\frac{\partial B(\sigx_n)}{\partial \sigx}(\hat{f}_n(\sigx_n)+\hat{g}_n(\sigx_n)\sigu_n)\nonumber\\
	&\geq -\eta B(\sigx_n)+\frac{\nu}{2}\norm{\hat{f}_n(\sigx_n)+\hat{g}_n(\sigx_n)\sigu_n}_{\Real^{n_x}}^2+\varrho_1. \label{DCBF}
	\end{align}	
	Moreover, \refeq{DCBF} defines a convex constraint for $\sigu_n$.
\end{theorem}
\begin{proof}
	See Appendix \ref{apptheo5}.
\end{proof}	
\begin{remark}
	When $\frac{\partial B(\sigx_n)}{\partial \sigx}\hat{g}_n(\sigx_n)\neq0$ and $\U$ admits sufficiently large value of each entry of $\sigu_n$,
	there always exists a $\sigu_n$ that satisfies \refeq{DCBF}.
\end{remark}
Theorem \ref{CBFaffine} essentially implies that, even when the gradient of $B$ along the shift of $\sigx_n$ decreases steeply, inequality \refeq{DCBFdef2} holds if \refeq{DCBF} is satisfied. 
From Theorem \ref{CBFaffine}, the set $\hat{\mathcal{S}}_n(\sigx_n)$, defined as
\begin{align}
	&\hat{\mathcal{S}}_n(\sigx_n):=\{\sigu_n\in\U|\frac{\partial B(\sigx_n)}{\partial \sigx}(\hat{f}_n(\sigx_n)+\hat{g}_n(\sigx_n)\sigu_n)\nn\\
	&\geq -\eta B(\sigx_n)+\frac{\nu}{2}\norm{\hat{f}_n(\sigx_n)+\hat{g}_n(\sigx_n)\sigu_n}_{\Real^{n_x}}^2+\varrho_1\}\subset\mathcal{S}(\sigx_n), \label{convexset}
\end{align}
is convex under Assumption \ref{assump1}.
%%%

%%%
As witnessed in the literatures (e.g., \cite{wang2017safety}), an agent might encounter deadlock situations, where the constrained control keeps the agent remain in the same state, when control barrier certificates are employed.
It is even possible that there is no safe control driving the agent from those states.
However, an elaborative design of control barrier functions remedies this issue, as shown in the following example.
\begin{example}
	\label{nonholoexample}
If the agent is nonholonomic, turning inward safe regions when approaching their boundary might be
infeasible.
To reduce the risk of such deadlock situations, control barrier functions may be designed as
\begin{align}
&\hspace{-1em}B(\sigx)=\tilde{B}(\sigx)-\upsilon\Gamma\left(\left|\theta-{\rm atan2}\left\{\frac{\partial\tilde{B}(\sigx)}{\partial{\rm y}},\frac{\partial\tilde{B}(\sigx)}{\partial{\rm x}}\right\}\right|\right),\nn\\
&\hspace{20em}\upsilon>0, \nn%\label{elaboCBF}
\end{align}
where the state $\sigx=[{\rm x};{\rm y};\theta]$ consists of the X position ${\rm x}$, the Y position ${\rm y}$, and
the orientation $\theta$ of an agent from the world frame,
$\{\sigx\in\X|\tilde{B}(\sigx)\geq0\}$ is the original safe region, and $\Gamma$ is a strictly increasing function.
If this control barrier function exists, then the agent is forced to turn inward the original safe region before reaching its boundaries
because the control barrier function also depends on $\theta$ and takes larger value when the agent is facing inward the safe region.
An illustration of this example is given in \reffig{fig:illustnonholo}.
\end{example}
\begin{figure}[t]
	\begin{center}
		\includegraphics[clip,width=0.43\textwidth]{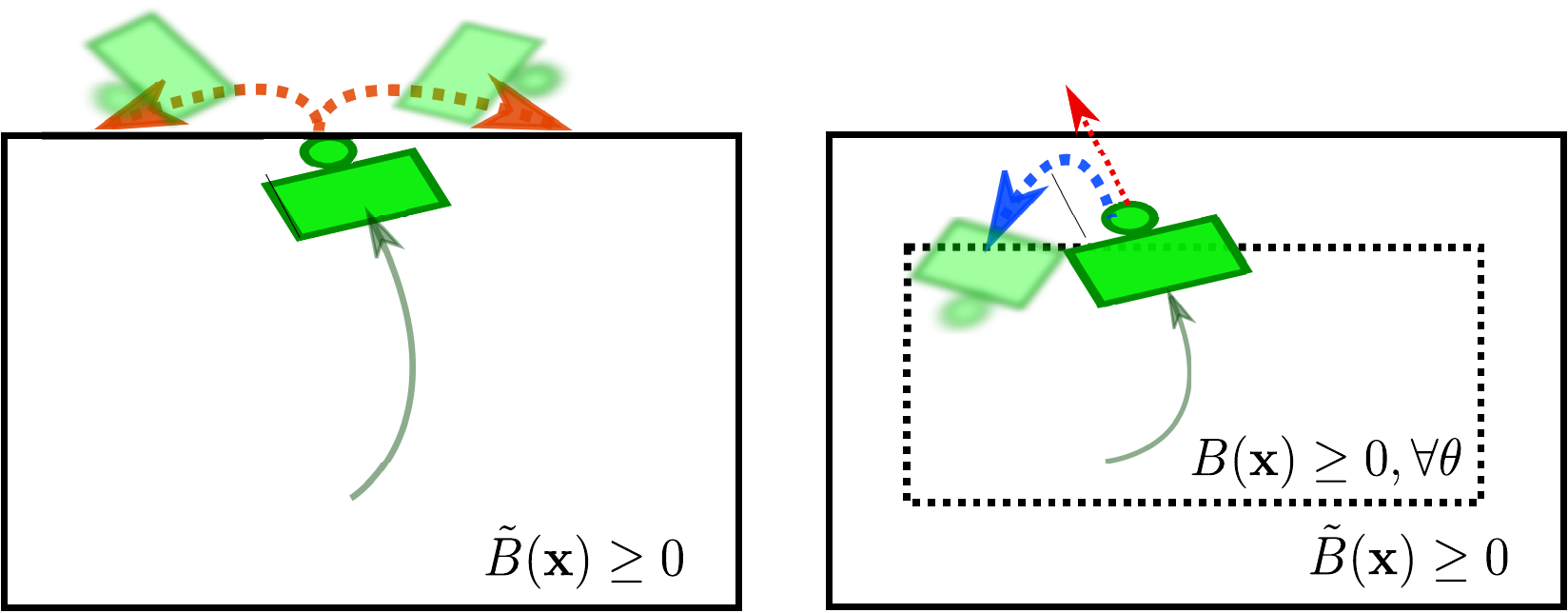}
		\caption{An illustration of how a nonholonomic agent avoids deadlocks.  When the orientation of the agent is not considered (i.e., $\tilde{B}(\sigx)$ is the barrier function), there might be
		no safe control driving the agent from those states as the left figure shows.  By taking into account the orientation (i.e., $B(\sigx)$ is the barrier function), the agent
		turns inward the safe region before reaching its boundaries as the right figure shows.}
		\label{fig:illustnonholo}
	\end{center}
\end{figure}
%%%%%
%%%%

%%%
Resulting barrier-certified adaptive reinforcement learning framework is summarized in Algorithm \ref{Qalgo}.
%%%%%%%%
\begin{algorithm}[tb]
	\caption{Barrier-certified adaptive reinforcement learning}
	\label{Qalgo}
	\begin{algorithmic}
		\STATE {\bfseries Requirement:} Assumptions \ref{assumstab} and \ref{assump1};
		$\kappa^Q$ defined as \refeq{kernel1des};
		$\sigx_0\in\X$ and $\sigu_0\in\U$; $\lambda\in(0,2)$, $\mu\geq0$ and $s\in\integer_{>0}$
		\STATE {\bfseries Output:}
		$\hat{Q}_n^{\phi}(\sigz_n)$ \hfill $\triangleright$ \refeq{Qest}
		\FOR{$n\in\integer_{\geq0}$}
		\STATE - Sample $\sigx_{n},\sigx_{n+1}\in\X$, $\sigu_n\in\mathcal{S}$, and $R(\sigx_{n},\sigu_{n})\in\Real$
		\STATE - Obtain $\phi(\sigx_{n+1})\in\hat{\mathcal{S}}_{n+1}(\sigx_{n+1})$ \hfill $\triangleright$ \refeq{convexset}
		\IF{Random Exploration}
		\STATE Select a (uniformly) random control input: \\\hfill$\sigu_{n+1}\in\hat{\mathcal{S}}_{n+1}(\sigx_{n+1})$ $\;\;\;\;\;\;$ $\triangleright$ \refeq{convexset}
		\ELSE
		\STATE Use the current policy: $\sigu_{n+1}=\phi(\sigx_{n+1})$ 
		\ENDIF
		\STATE - Model update: $\sigh_{n+1}=T_n(\sigh_n)$
		 \hfill $\triangleright$  e.g., \refeq{hnupdate}
		\STATE - Update $\hat{Q}_n^{\phi}$ by updating $\hat{\psi}^Q_n$ in $\H_{\psi^Q}$: \\$\;\;\;$(e.g., kernel adaptive filter)\\ $\;\;\;$ $\sigh^Q_{n+1}=\prox_{\lambda\mu}\left[(1-\lambda)I+\lambda\sum_{\iota=n-s+1}^{n}\frac{1}{s}P_{C_{\iota}}\right](\sigh^Q_n)$
		\\\hfill $\triangleright$  Theorem \ref{theoRKHS} and \refeq{hnupdate}
		\IF{$n\;{\rm mod}\;N_f\;=\;0$} 
		\STATE  $\phi^{+}(\sigx)=\argmax_{\sigu\in{\hat{\mathcal{S}}_n}(\sigx)}\left[\hat{Q}_n^{\phi}(\sigx,\sigu)\right]$ \hfill $\triangleright$ \refeq{convexset} and \refeq{improvepolicy}
		\STATE Let $\phi\leftarrow\phi^+$
		\ENDIF
		\ENDFOR
	\end{algorithmic}
\end{algorithm}
%%%%%%

%%%%%
\section{Experimental Results}
\label{sec:numerical}
For the sake of reproducibility and for clarifying each contribution, we first validate the proposed learning framework on simulations of
vertical movements of a quadrotor, which has been used in the safe learnings literature under stationarity assumption (e.g., \cite{akametalu2014reachability}).
Then, we test the proposed learning framework on a real robot called {\em brushbot}, whose dynamics is unknown, highly complex and nonstationary\footnote{The dynamics of the brushbot depends on the body structure, conditions of the brushes, floors and many other factors.  Thus, simulators of the brushbot are unavailable.}.
The experiments on the {\em brushbot} was conducted at the Robotarium, a remotely accessible robot testbed at Georgia institute of technology \cite{robotarium}.

\subsection{Validations of the Safe Learning Framework via Simulations of a Quadrotor}
\label{subsec:simulate}
In this experiment, we empirically validate Theorem \ref{maintheo} (i.e., Lyapunov stability of the set of augmented safe states after an unexpected and abrupt change of the agent dynamics)
and the motivations of using an online kernel method working in the RKHS $\H_{\psi^Q}$ (see Section \ref{subsec:adaptiveRL}) for action-value function approximation.
We also test
the proposed framework for simulated vertical movements of a quadrotor.
We use parametric model for the agent dynamics and nonparametric model for the action-value function
in this experiment.
The discrete-time dynamics of the vertical movement of a quadrotor is given by
\begin{align}
	&\sigx_{n+1}={\Xi}(\sigz_n)\sigh^{*}:=h^{*}_1\xi_1(\sigz_n)+h^{*}_2\xi_2(\sigz_n)+h^{*}_3\xi_3(\sigz_n)\nn\\
	&:=h_1\left[
	\begin{array}{cc}
		1 & \Delta t\\
		0 & 1
	\end{array}
	\right]\sigx_n+h_2
	\left[
	\begin{array}{c}
	-\frac{\Delta t^2}{2} \\
	-{\Delta t}
	\end{array}
	\right]+h_3
	\left[
	\begin{array}{c}
	-\frac{\Delta t^2}{2} \\
	-{\Delta t}
	\end{array}
	\right]\sigu_n,\nn\\
	&\sigh^{*}:=[h^{*}_1;h^{*}_2;h^{*}_3]\in\Real^{3},\;\xi_i:\mathcal{Z}\rightarrow\Real,\;i\in\{1,2,3\},\nn\\
	&\sigz_n:=[\sigx_n;\sigu_n]\in\mathcal{Z},\;\sigx_n:=[{\rm x}_n;\dot{\rm {x}}_n], \nn
\end{align}
where $\Delta t\in(0,\infty)$ denotes the time interval, ${\rm x}_n$ and $\dot{\rm x}_n$ are the vertical position
and the vertical velocity of the quadrotor at time instant $n$, respectively.
When the weight of the quadrotor is $0.027{\rm kg}$, the nominal model is given by $h_1=1,\;h_2=9.81$, and $h_3=1/0.027$.
Let the time interval $\Delta t$ be $0.02$ seconds for the simulations, and the maximum input $2\x0.027\x9.81$.
%%%%

%%%%
Control barrier certificates are used to limit the region of exploration
to the area: ${\rm x}\in[-3,3]$, and we employ the following two barrier functions:
\begin{align}
&B_t(\sigx)=3-{\rm x}, \nn\\
&B_b(\sigx)={\rm x}+3, \nn
\end{align}
%%%%%%
and we use the barrier-certificate parameter $\eta=0.01$ (see \refeq{DCBFdef2}) in this experiment.
Note that the safe set is equivalently expressed by
\begin{align}
	\mathcal{C}=[-3,3]=\{\sigx\in\X|B_t(\sigx)\geq 0\land B_b(\sigx)\geq 0\}, \nn
\end{align}
and the barrier functions satisfy Assumption \ref{assump1}.2 with the Lipschitz constant $\nu=0$.
%%%%

%%%%
The immediate reward is given by
\begin{align}
R(\sigx,\sigu)=-2{\rm x}^2-\frac{1}{2}{\rm \dot{x}}^2+12, \;\forall n\in\integer_{\geq0},\nn
\end{align}
where the constant is added to prevent the resulting value of explored states from
becoming negative, i.e., lower than the value outside of the safe set.
%%%%
%%%%%%
\subsubsection{Stability of the Safe Set}
In terms of safety recovery, we compare a GP-based approach, which tends to be less adaptive to time-varying systems, and a set-theoretical adaptive model learning algorithm with monotone approximation property.
Random explorations by uniformly random control inputs are conducted for the first $20$ seconds corresponding to $1000$ iterations
under the dynamics $\sigh^{*}=[1;9.81;1/0.027]$.
Then, we change the simulated dynamics and observe if the quadrotor is stabilized on the set of augmented safe states.
To clearly visualize the difference between the GP-based approach and the adaptive model learning algorithm,
we let the new agent dynamics be $\sigh^{*}=[1;9.81;5/0.027]$, which is an extreme situation
where the maximum input generates very large acceleration.
%%%%

%%%%
We define the update rule of model learning as
\begin{align}
	\sigh_{n+1}=\sigh_n-\lambda\Xi^{\T}(\sigz_n)(\Xi(\sigz_n)\Xi^{\T}(\sigz_n))^{-1}(\Xi(\sigz_n)\sigh_n-\sigx_{n+1}),\nn
\end{align}
which satisfies the monotone approximation property\footnote{This update is viewed as the projection of the current parameter onto the affine set
in which any element $\sigh_n^{*}$ satisfies $\Xi(\sigz_n)\sigh_n^{*}-\sigx_{n+1}=0$, and hence it follows that $\norm{\sigh_{n}-\sigh_n^{*}}^2_{\Real^{r}}-\norm{\sigh_{n+1}-\sigh_n^{*}}^2_{\Real^{r}}\geq\varrho^2_3 dist^2(\sigh_n,\Omega_n),\;\forall \sigh_n^{*}\in\Omega_n:=\argmin_{\sigh\in\Real^{r}}\left[\Xi(\sigz_n)\sigh_n^{*}-\sigx_{n+1}\right]$.  (See \cite{APSM1} for more detailed arguments.)}, where $\lambda\in(0,2)$ is the step size.
In this experiment, we used $\lambda=0.6$.
For the GP-based learning, on the other hand, we let the noise variance of the output be $0.01$, and let
the prior covariance of the parameter vector $\sigh$ be $25I$.
%%%%

%%%%
The trajectories of the vector $[{\rm x};h_2;h_3]$ of the GP-based learning and the adaptive model learning algorithm
from $n=1000$ to $n=10000$
are plotted in \reffig{fig:lyapunov}.
We can observe that the trajectory of the adaptive model learning algorithm converges to the forward invariant set $\mathcal{C}\x\Omega$.
GP-based learning seems to be slowly approaching to the safe set while safety recovery is not theoretically supported in the current settings.
\begin{figure}[t]
	\begin{center}
		\includegraphics[clip,width=0.55\textwidth]{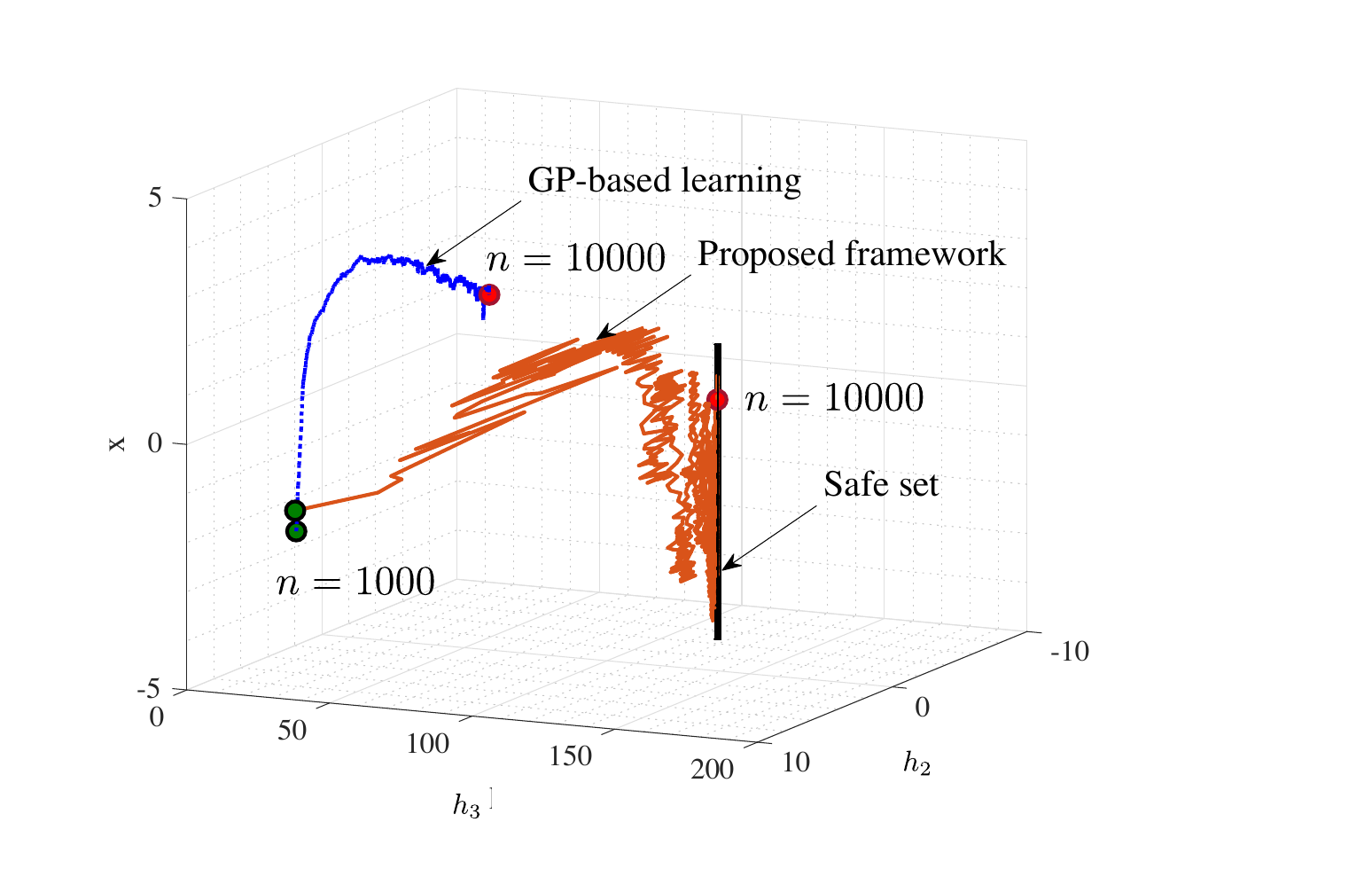}
		\caption{Trajectories of the vector $[{\rm x};h_2;h_3]$ of the GP-based learning and the adaptive model learning algorithm with barrier certificates
			from $n=1000$ to $n=10000$.  The trajectory of the adaptive model learning algorithm converges to the forward invariant set $\mathcal{C}\x\Omega$.
			GP-based learning seems slowly approaching to the safe set while safety recovery is not theoretically supported in the current settings.}
		\label{fig:lyapunov}
	\end{center}
\end{figure}
\subsubsection{Adaptive Action-value Function Approximation}
\label{experisimQfunc}
We also validate our action-value function approximation framework by employing a GP (i.e., the GP SARSA) and a kernel adaptive filter in the same RKHS.
The parameter settings for the kernel adaptive filter are summarized in \reftab{sumparamsim}.
Please refer to Appendix \ref{appemulti} for the notations that are not in the main text.
Six Gaussian kernels with different scale parameters $\s$ are employed for the kernel adaptive filter (i.e., $M=6$.  See also Appendix \ref{appemulti} for more detail about multikernel adaptive filter).
For the GP SARSA, we employ a Gaussian kernel with scale parameter $3$, which achieved sufficiently good performance, and let the noise variance of the output be $10^{-6}$ (i.e., $\Sigma=10^{-6}I$.  See Appendix \ref{appencomp}.).  Other parameters are the same as those of the kernel adaptive filter.
In addition, we also test the GP SARSA in another settings, where the kernel function is added in the first $600$ iterations (i.e., dimension of the parameter becomes $r=600$)
and is not newly added after $600$ iterations.  We call this as the GP SARSA 2 for convenience in this section.
%%%%%%
\begin{table}[t]
	\begin{center}
		\caption{Summary of the Parameter Settings of the Simulated Vertical Movements of a Quadrotor (kernel adaptive filter)}
		\label{sumparamsim}
		\begin{tabular}{ccc}\hline
			%\multicolumn{3}{c}{\textbf{General Setting}}\\\hline\hline
			Parameter & Description & Values \\\hline
			$\lambda$ & step size & $0.1$ \\%\hline
			$s$ & data size & $5$ \\%\hline
			$\mu$ & regularization parameter & $0.01$ \\%\hline
			$\epsilon_1$ & precision parameter & $0.2$ \\%\hline
			$\epsilon_2$ & large-normalized-error  & $0.1$ \\%\hline
			$r_{\rm max}$ & maximum-dictionary-size & $600$\\%\hline
			$\s$ & scale parameters & $\{50,30,10,5,2,1\}$ \\
			$\gamma$ & discount factor & $0.9$\\\hline
		\end{tabular}
	\end{center}
\end{table}
We employ an adaptive model learning algorithm for all of the three reinforcement learning approaches, and
update policies every $1000$ iterations.
For the comparison purpose, we do not reset learning even when the policy is updated\footnote{Note the GP SARSA was originally designed for stationary agent dynamics.  In this experiment, we call GP SARSA as a GP working in the RKHS $\H_{\psi^Q}$.}.
Random explorations by uniformly random control inputs are conducted for the first $200$ seconds corresponding to $10000$ iterations
under the dynamics $\sigh^{*}=[1;9.81;1/0.027]$, and the dynamics changes to
$\sigh^{*}=[1;11.81;0.9/0.027]$ (i.e., additional downward accelerations and degradations of batteries, for example) at
time instant $n=2500$.
We evaluate the policy obtained at time instant $n=10000$ for five times with different initial states,
and we also conduct $15$ runs for learning.  For each policy evaluation, the initial position ${\rm x}$ follows the uniform distribution while the velocity ${\rm \dot{x}}=0$.
%%%%

%%%%
The learning curves of the normalized mean squared errors (NMSEs) of action-value function approximation, which are
averaged over $15$ runs and smoothed, are plotted in \reffig{fig:NMSE} for the GP SARSA, the kernel adaptive filter and the GP SARSA 2.
From \reffig{fig:NMSE}, we can observe that both the GP SARSA and the kernel adaptive filter show no large degradations of the NMSE
even after the dynamics changes or the policy is updated, while the GP SARSA 2 stops improving the NMSE after the policy is updated (and the dynamics is changed).
Because no kernel function is newly added after the first $600$ iterations, the GP SARSA 2 could not adapt
to the new policy or new dynamics.
\begin{figure}[t]
	\begin{center}
		\includegraphics[clip,width=0.5\textwidth]{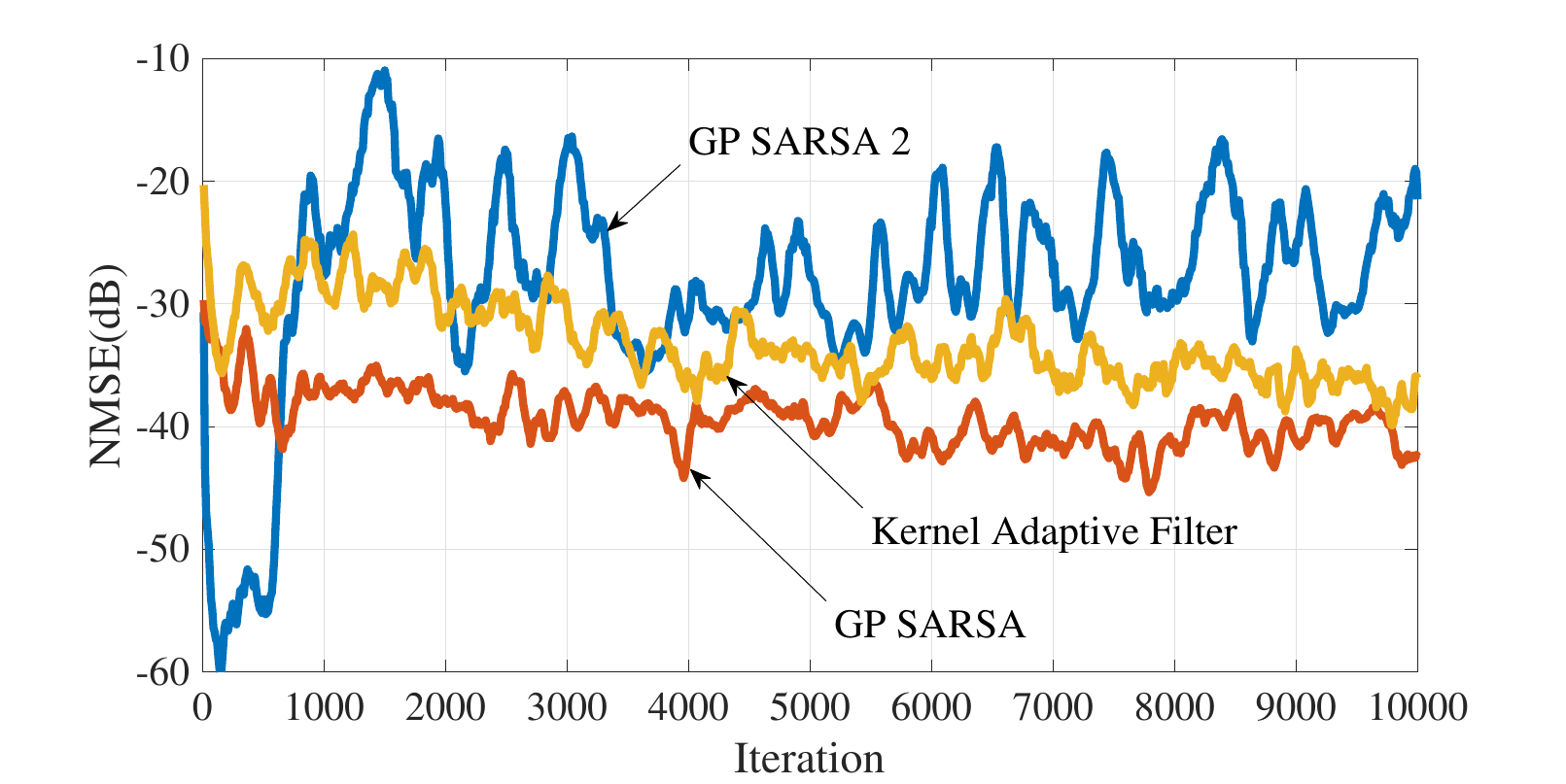}
		\caption{The learning curves of the normalized mean squared errors (NMSEs) of action-value function approximation
			for the GP SARSA, kernel adaptive filter, and the GP SARSA 2.}
		\label{fig:NMSE}
	\end{center}
\end{figure}
%%%%

%%%%
The expected values $\Ex{V^{\phi}(\sigx)}$ for the GP SARSA, the kernel adaptive filter and the GP SARSA 2
associated with the policies obtained at time instant $n=10000$ are
shown in \reftab{expectedQ}.  (expectation is taken over the $15\x5$ runs, i.e., $15$ runs for learning, each of which includes five policy evaluations).  Recall that
$V^{\phi}$ is defined in \refeq{Qdef}.
%%%%%%
\begin{table}[t]
	\begin{center}
		\caption{The expected values of the GP SARSA and the kernel adaptive filter}
		\label{expectedQ}
		\begin{tabular}{ccc}\hline
			%\multicolumn{3}{c}{\textbf{General Setting}}\\\hline\hline
			GP SARSA& kernel adaptive filter & GP SARSA 2 \\\hline
			$65.77\pm41.42 $ & $64.68\pm40.39$ & $63.01\pm42.50$\\\hline
		\end{tabular}	
	\end{center}
\end{table}
%%%%

%%%%
Among the $15$ runs for the kernel adaptive filter, we extracted the seventh run, which was successful.
The left figure of \reffig{fig:simuresult} illustrates the action-value function at time instant $n=10000$ of the seventh run for the kernel adaptive filter, and the right figure of \reffig{fig:simuresult} plots the trajectory of the optimal policy obtained at time instant $n=10000$ for the seventh run.  The simulated quadrotor was relocated at time instant $n=11000,12000,13000$, and $n=14000$,
and both the position and the velocity of the simulated quadrotor went to zeros successfully.
\begin{figure*}[t]
	\begin{minipage}{0.4\hsize}
		\begin{center}
			\includegraphics[clip,width=0.9\textwidth]{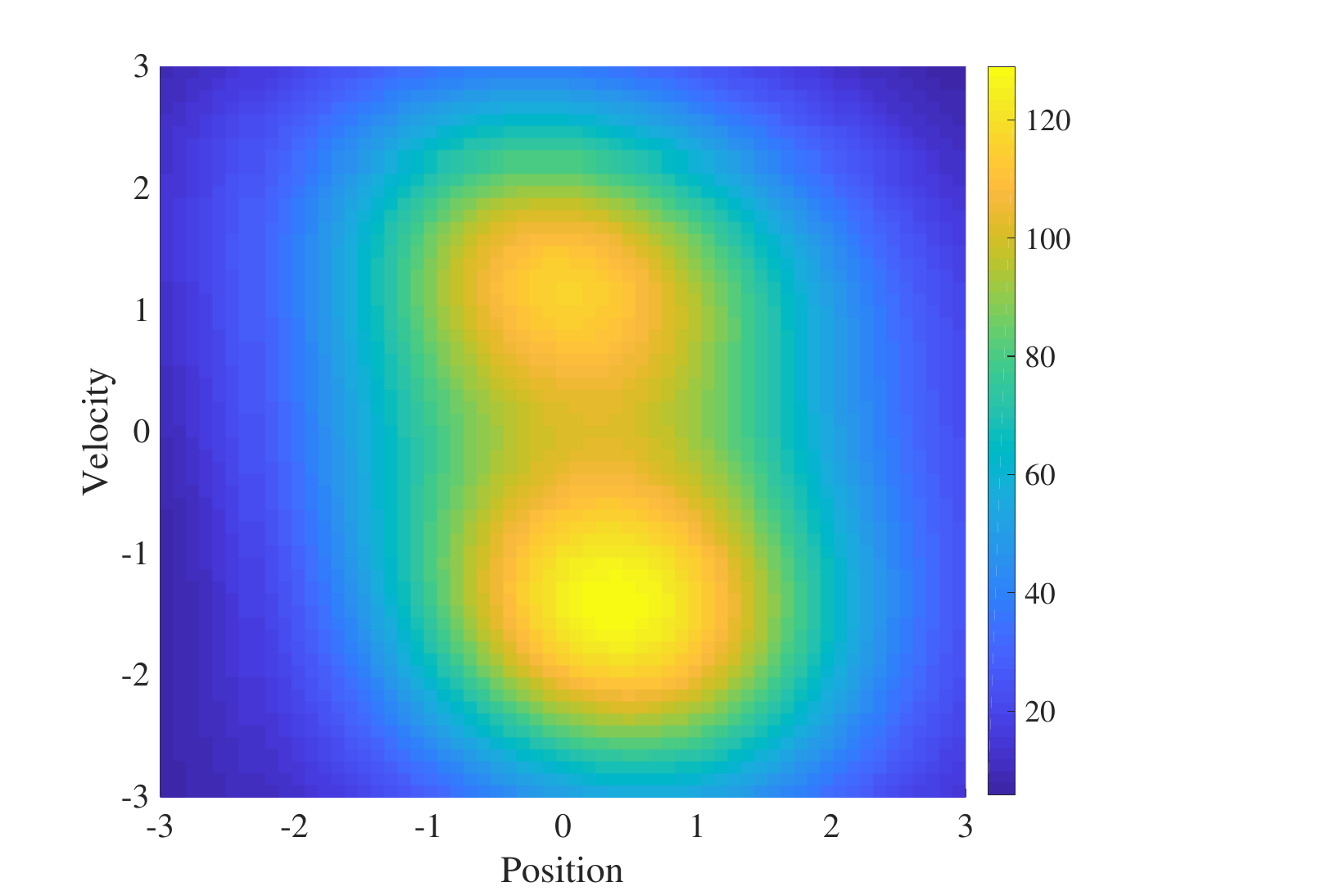}
		\end{center}	
	\end{minipage}
	\begin{minipage}{0.58\hsize}
		\begin{center}
			\includegraphics[clip,width=0.95\textwidth]{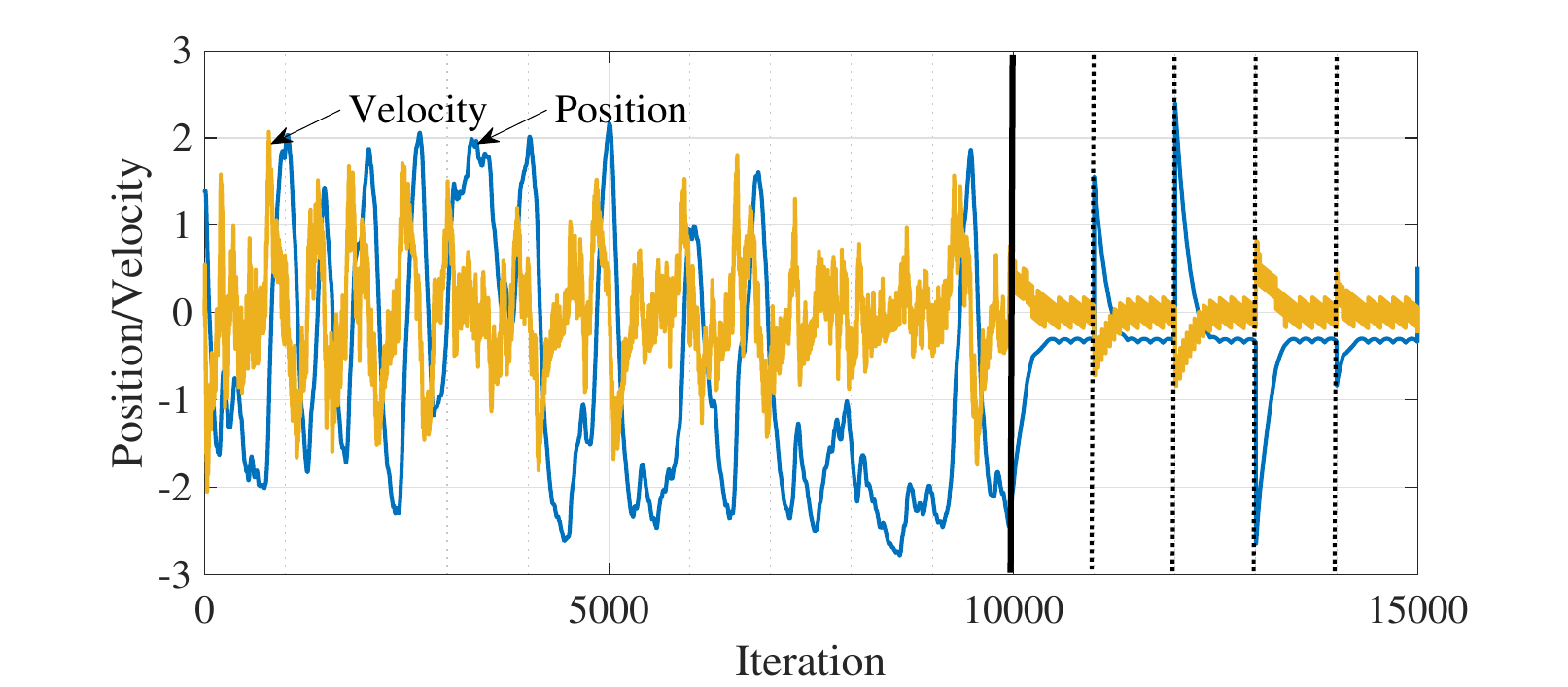}
		\end{center}
	\end{minipage}
	\caption{The left figure illustrates the action-value function over the position ${\rm x}$ and the velocity ${\rm \dot{x}}$ at $n=10000$ and at the control input $-0.027\x11.81/0.9$, which cancels out the acceleration added to the quadrotor.  Positive velocities for negative positions and negative velocities for positive positions have higher values.  The right figure shows the trajectory of the optimal policy obtained at time instant $n=10000$ for the seventh run, which was a successful run among the $15$ runs.  Dashed lines indicate the time when the quadrotor was relocated.  Both the position and the velocity of the simulated quadrotor went to zeros successfully.}
	\label{fig:simuresult}
\end{figure*}
\subsubsection{Discussion}
The control barrier certificates with an adaptive model learning algorithm recovered safety even for an extreme situation
where the control inputs start generating very large acceleration.  As long as model learning algorithm satisfies
Assumption \ref{assumstab}, safety recovery is guaranteed.
%%%%%

%%%%%
Reinforcement learning with the GP SARSA and kernel adaptive filter in the RKHS $\H_{\psi^Q}$ worked sufficiently well.
If no kernel functions are newly added, GP-based learnings cannot adapt to the new policies or agent dynamics.
Therefore, we need to sequentially add new kernel functions or use a sparse adaptive filter to prune redundant kernel functions (see also Appendix \ref{appemulti} for
a sparse adaptive filter).
We mention that identifying the RKHS $\H_{\psi^Q}$ enabled us to employ GPs for nonstationary agent dynamics without having to reset learnings.
Consequently, we can effectively reuse the previous estimation of the target function if the new target function is close to the previous one.
%%%%%

%%%%%
Our safe learning framework validated by these simulations is now ready to be applied to a real robot called brushbot as presented below.
%%%%
%%%%%
\subsection{Real-Robotics Experiments on the Brushbot}
Next, we apply our safe learning framework, which was validated by simulations, to the brushbot, which has highly nonlinear, nonholonomic and nonstationary dynamics (see \reffig{fig:brushbot}).
\begin{figure}[t]
	\begin{center}
		\includegraphics[clip,width=0.45\textwidth]{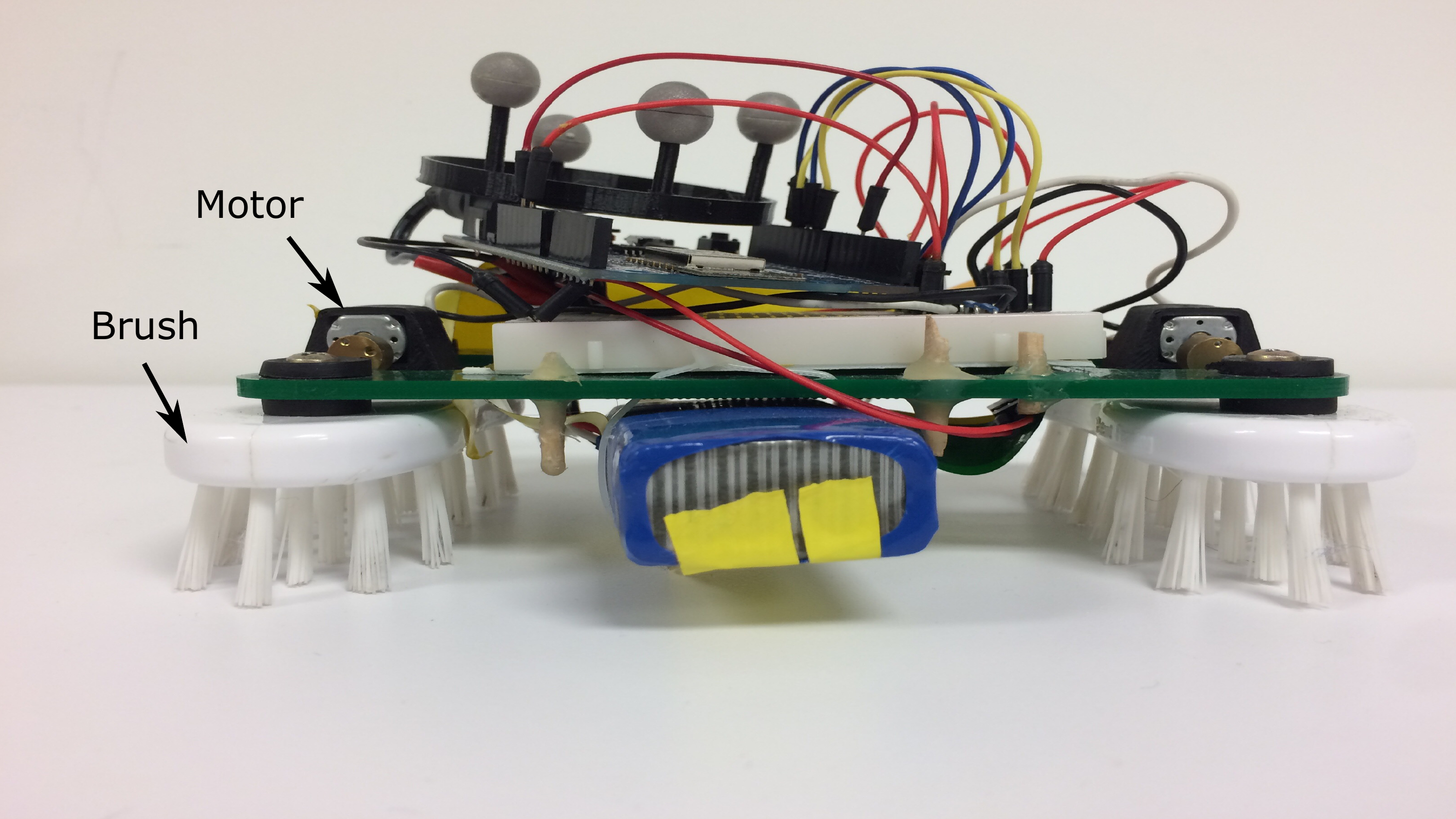}
		\includegraphics[clip,width=0.45\textwidth]{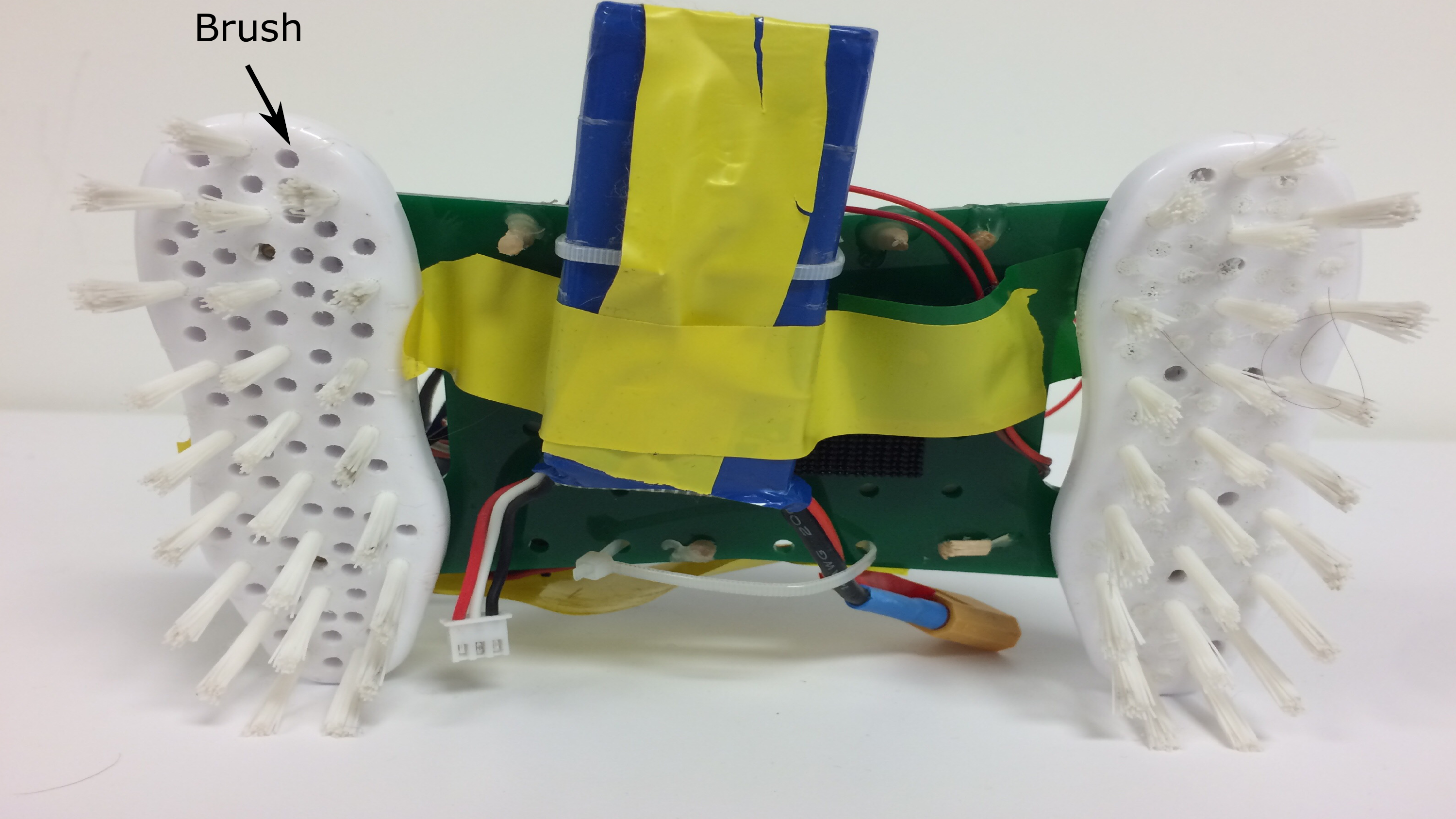}
		\caption{A picture of the brushbot used in the experiment.  Vibrations of the two motors propagate to the two brushes,
		driving the brushbot.  Control inputs are of two dimensions each of which corresponds to the rotational speed of a motor.}
		\label{fig:brushbot}
	\end{center}
\end{figure}
The objective of this experiment is to find a policy driving the brushbot to the origin, while restricting the region of exploration.
The experiment is conducted at the Robotarium, a remotely accessible robot testbed at Georgia institute of technology \cite{robotarium}.
\subsubsection{Experimental Condition}
The experimental conditions for model learning, reinforcement learning, control barrier functions and their parameter settings are presented below.
\paragraph{Model learning}
The state $\sigx=[{\rm x};{\rm y};\theta]$ consists of the X position ${\rm x}$, Y position ${\rm y}$ and the orientation $\theta\in[-\pi,\pi]$ of the brushbot from the world frame.
The exact positions and the orientation are recorded by motion capture systems every $0.3$ seconds.
A control input $\sigu$ is of two dimensions each of which corresponds to the rotational speed of a motor.
To improve the learning efficiency and reduce the total learning time required, we identify the most significant dimension and reduce the dimensions to learn.
The sole input variable of $p,f$ and $g$ for the shifts of ${\rm x}$ and ${\rm y}$, is assumed to be $\theta$.
The shift of $\theta$ is assumed to be constant over the state, and hence depends on nothing but control inputs (see Section \ref{paramset}).
The brushbot used in the present study is nonholonomic, i.e.,
it can only go forward, and positive control inputs basically drive the brushbot in the same way as negative control inputs.
As such, we use the rotational speeds of the motors as the control inputs.
Moreover, to eliminate the effect of static frictions on the model, we assume that
the zero control input given to the algorithm actually generates some minimum control inputs $u_{\delta}$ to the motors, i.e., the actual maximum control inputs to the motors are given by $u_{\rm max}+u_{\delta}$, where $u_{\rm max}$ is the
maximum control input fed to the algorithm.
\paragraph{Reinforcement learning}
The state for action-value function approximation consists of the distance $\norm{[{\rm x};{\rm y}]}_{\Real^2}$ from the origin and the orientation $\theta-{\rm atan2}\left({\rm y},{\rm x}\right)$ which is wrapped to the interval $[-\pi,\pi]$.
The immediate reward is given by
\begin{align}
	R(\sigx,\sigu)=-\norm{[{\rm x};{\rm y}]}^2_{\Real^2}+2, \;\forall n\in\integer_{\geq0},\nn
\end{align}
where the constant is added to prevent the resulting value of explored states from
becoming negative, namely, lower than the value outside of the region of exploration.
\paragraph{Discrete-time control barrier certificates}
Control barrier certificates are used to limit the region of exploration
to the rectangular area: ${\rm x}\in[-{\rm x}_{\rm max},{\rm x}_{\rm max}],\;{\rm y}\in[-{\rm y}_{\rm max},{\rm y}_{\rm max}]$, where
${\rm x}_{\rm max}>0$ and ${\rm y}_{\rm max}>0$.
Because the brushbot can only go forward, we employ the following four barrier functions:
\begin{align}
&B_1(\sigx)={\rm x}_{\rm max}-{\rm x}-\upsilon\left|\theta+\pi\right|,\nn\\
&B_2(\sigx)={\rm x}+{\rm x}_{\rm max}-\upsilon\left|\theta\right|,\nn\\
&B_3(\sigx)={\rm y}_{\rm max}-{\rm y}-\upsilon\left|\theta+\frac{\pi}{2}\right|,\nn\\
&B_4(\sigx)={\rm y}+{\rm y}_{\rm max}-\upsilon\left|\theta-\frac{\pi}{2}\right|,\nn
\end{align}
(see Example \ref{nonholoexample} for the motivations of using the above control barrier functions).
Note that those functions satisfy Assumption \ref{assump1}.2 and the Lipschitz constant $\nu$ is zero except at around $\theta=-\frac{\pi}{2},0,\frac{\pi}{2},\pi$.
(Although we can employ globally Lipschitz functions for more rigorous treatment, we use the above functions for simplicity.)
%%%%%%%%%%%%
\paragraph{Parameter settings}
\label{paramset}
The parameter settings are summarized in \reftab{sumparam}.
Please refer to Appendix \ref{appemulti} for the notations that are not in the main text.
Five Gaussian kernels with different scale parameters $\s$ are employed in action-value function approximation (i.e., $M=5$.  See also Appendix \ref{appemulti} for more detail about multikernel adaptive filter), and
six Gaussian kernels are employed in model learning for ${\rm x}$ and ${\rm y}$ (i.e., $M=6$).
In model learning for $\theta$, we define $\H_p$, $\H_f$ and $\H_g$ as sets of constant functions.
%%%%%

%%%%%
The kernels of $\H_p$ and $\H_f$ are weighed by $\tau=0.1$ in model learning (see Lemma \ref{propweight} in Appendix \ref{appemulti}).
%%%%%%
\begin{table*}[t]
	\begin{center}
		\caption{Summary of the Parameter Settings}
		\label{sumparam}
		\begin{tabular}{ccccc}\hline
			%\multicolumn{3}{c}{\textbf{General Setting}}\\\hline\hline
			Parameter & Description & \multicolumn{3}{c}{General settings} \\\hline
			${\rm x}_{\rm max}$ & maximum X position & \multicolumn{3}{c}{$1.2$} \\%\hline
			${\rm y}_{\rm max}$ & maximum Y position & \multicolumn{3}{c}{$1.2$} \\%\hline
			$\eta $ & barrier-function parameter& \multicolumn{3}{c}{$0.1$}\\%\hline
			$\upsilon$ & coefficient in barrier functions & \multicolumn{3}{c}{$0.1$} \\%\hline
			$u_{\delta}$ & actual minimum control & \multicolumn{3}{c}{$0.4$} \\%\hline
			$u_{\rm max}$ & maximum control input & \multicolumn{3}{c}{$0.623$} \\\hline
			Parameter & Description & Model learning (${\rm x}$ and ${\rm y}$) &Model learning ($\theta$) & Action-value function approximation\\\hline
			$\lambda$ & step size & $0.3$ & $0.03$ & $0.3$\\%\hline
			$s$ & data size & $5$ & $10$ & $10$\\%\hline
			$\mu$ & regularization parameter & $0.0001$ & $0$ &$0.0001$\\%\hline
			$\epsilon_1$ & precision parameter & $0.001$ & $0.01$ & $0.05$ \\%\hline
			$\epsilon_2$ & large-normalized-error  & $0.1$ & $0.1$& $0.1$\\%\hline
			$r_{\rm max}$ & maximum-dictionary-size & $500$ & $3$ & $2000$\\%\hline
			$\s$ & scale parameters & $\{10,5,2,1,0.5,0.2\}$ & -- & $\{10,5,2,1,0.5\}$\\
			$\gamma$ & discount factor & -- & --& $0.95$\\\hline
		\end{tabular}
	\end{center}
\end{table*}
%%%%%%%%
\paragraph{Procedure}
The time interval (duration of one iteration) for learning is $0.3$ seconds,
and random explorations are conducted for the first $300$ seconds corresponding to $1000$ iterations.
While exploring, the model learning algorithm adaptively learns a model whose control-affine terms, i.e., $\hat{f}_n(\sigx)+\hat{g}_n(\sigx)\sigu$, is used in combination with barrier certificates.
Although barrier functions employed in the experiment reduce deadlock situations, the brushbot is forced to turn inward the region of exploration
when a deadlock is detected.
Note that the barrier certificates are intentionally violated in such a case.
The policy is updated every $50$ seconds.
After $300$ seconds, we stop learning a model and the action-value function, and the policy replaces random explorations.
The brushbot is forced to stop when it enters into the circle of radius $0.2$ centered at the origin.
When the brushbot is driven close to the origin and enters this circle, it is pushed away from the origin to see
if it returns to the origin again (see \reffig{frameimage}).
%%%%%%%%
\subsubsection{Results}
\reffig{fig:learnedmodel} plots $\hat{p}_{n}([\sigx;0;0]),\; \hat{f}_{n}(\sigx)$, $\hat{g}_{n}^{(1)}(\sigx)$ and $\hat{g}_{n}^{(2)}(\sigx)$ for
${\rm x}$ and ${\rm y}$ at $n=1000$.
Here $\hat{g}_{n}^{(i)}$ is the estimate of ${g}^{(i)}$ at time instant $n$.
Recall that these functions only depend on $\theta$ in this experiment to improve the learning efficiency.
For the shift of $\theta$, the estimators are constant over the state, and the result is $\hat{g}_{n}^{(1)}(\sigx)=1.38$, $\hat{g}_{n}^{(2)}(\sigx)=-0.77$ and $\hat{p}_{n}([\sigx;0;0])=\hat{f}_{n}(\sigx)=0$ at $n=1000$.
As can be seen in \reffig{fig:learnedmodel}, $\hat{p}_{n}([\sigx;0;0])$ is almost zero and so is $\hat{f}_{n}(\sigx)$,
implying that the proposed algorithm successfully dropped off irrelevant structural components of a model.
\begin{figure}[t]
	\begin{center}
		\includegraphics[clip,width=0.51\textwidth]{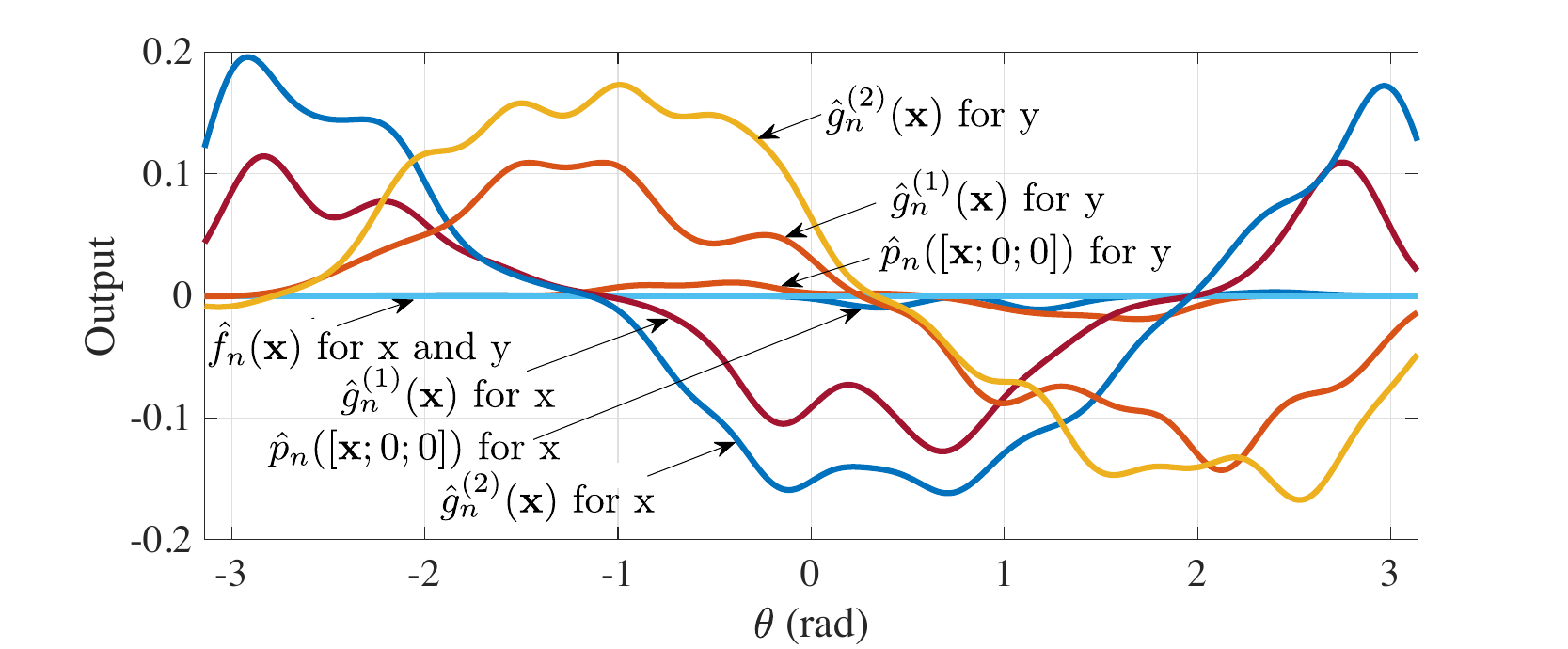}
	\end{center}
	\caption{Estimated output of the model estimator at $\sigu=[0;0]$ and $n=1000$ over the orientation $\theta$.  
		Irrelevant structures such as $\hat{p}_{n}$ and $\hat{f}_{n}$ dropped off successfully.} 
	\label{fig:learnedmodel}
\end{figure}
%%%%%%%

%%%%%%%
\reffig{fig:result1} plots the trajectory of the brushbot while exploring (i.e., X,Y positions from $n=0$ to $n=1000$).
It is observed that the brushbot remained in the region of exploration (${\rm x}\in[-1.2,1.2]$ and ${\rm y}\in[-1.2,1.2]$) most of the time.
Moreover, the values of barrier functions $B_i,\;i\in\{1,2,3,4\}$, for
the whole trajectory are plotted in \reffig{CBFcurve}.
Even though some violations of safety are seen in the figure, the brushbot returned to the safe region before large violations occurred.
Despite unknown, highly complex and nonstationary system, the proposed safe learning framework was shown to work efficiently.
\begin{figure*}[t]
	\begin{minipage}{0.38\hsize}
		\begin{center}
			\hspace{1em}
			\includegraphics[clip,width=1.1\textwidth]{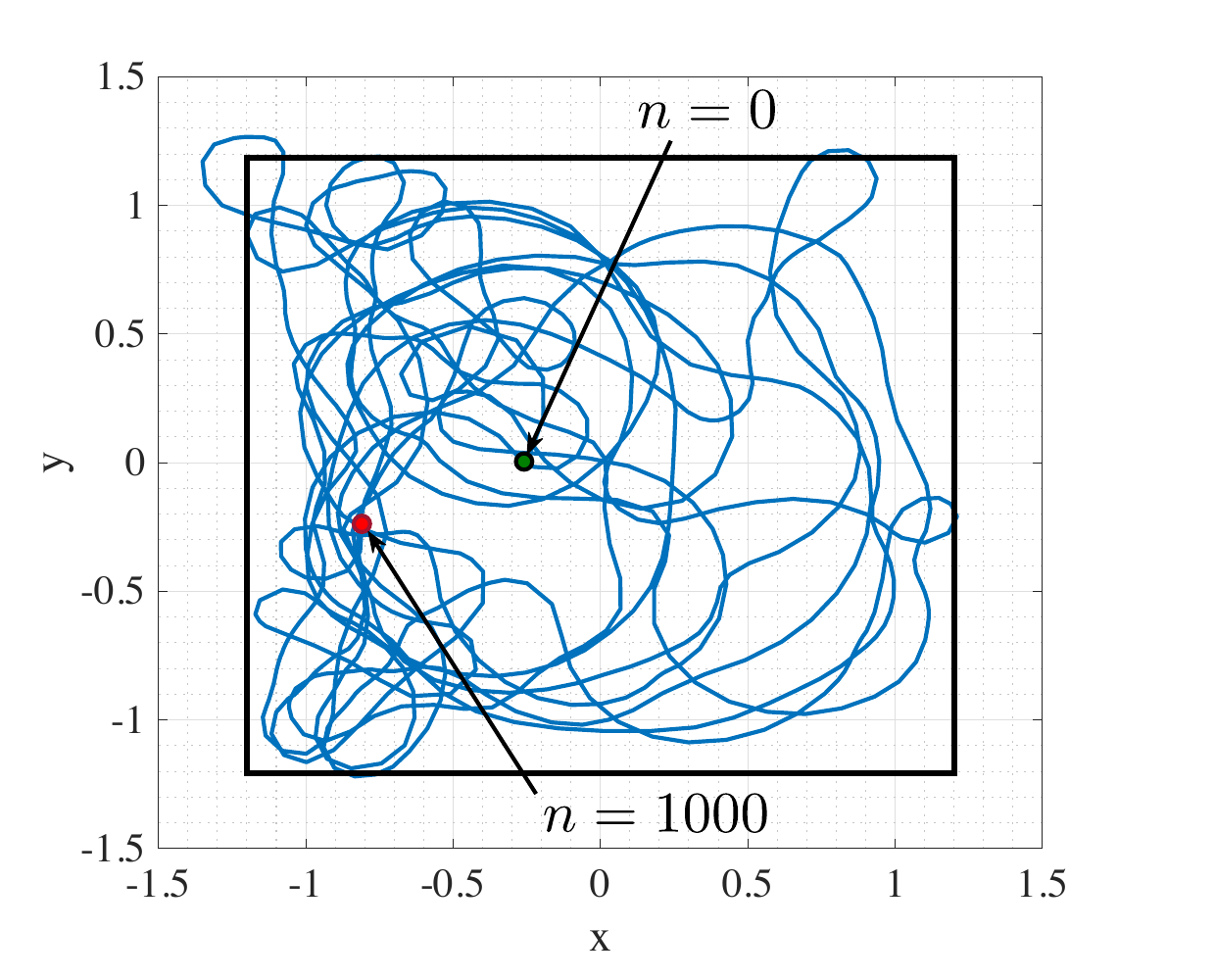}
		\end{center}	
	\end{minipage}
	\begin{minipage}{0.60\hsize}
		\begin{center}
			\hspace{1em}
			\includegraphics[clip,width=1.0\textwidth]{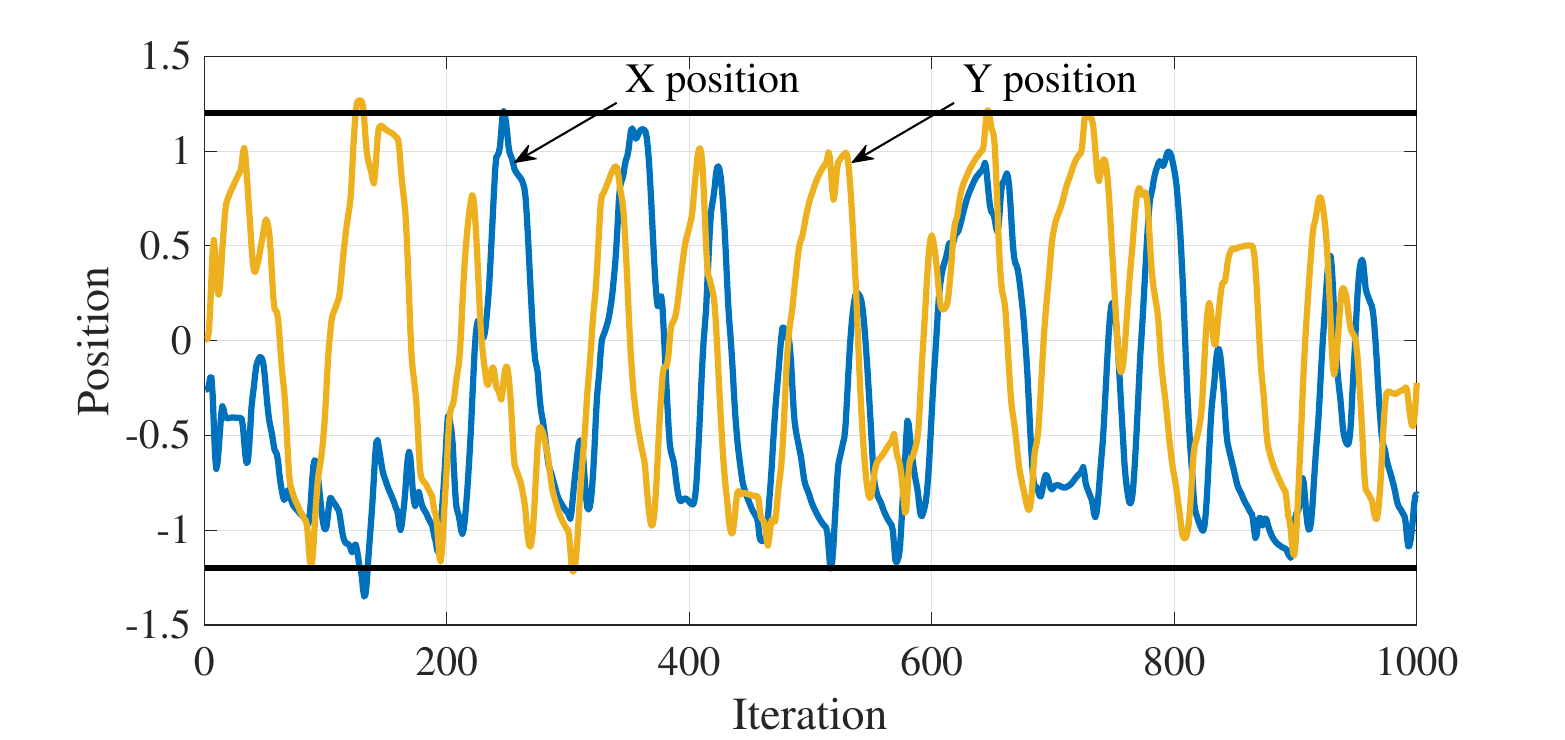}
		\end{center}
	\end{minipage}
	\caption{The left figure shows the trajectory of the brushbot while exploring, and the right figure shows X,Y positions over iterations.  The region of exploration is limited to ${\rm x}\in[-1.2,1.2]$ and ${\rm y}\in[-1.2,1.2]$.
		The brushbot remained in the region most of the time. }
	\label{fig:result1}
\end{figure*}
\begin{figure*}[t]
	\begin{center}
		\hspace{-3em}
		\includegraphics[clip,width=0.95\textwidth]{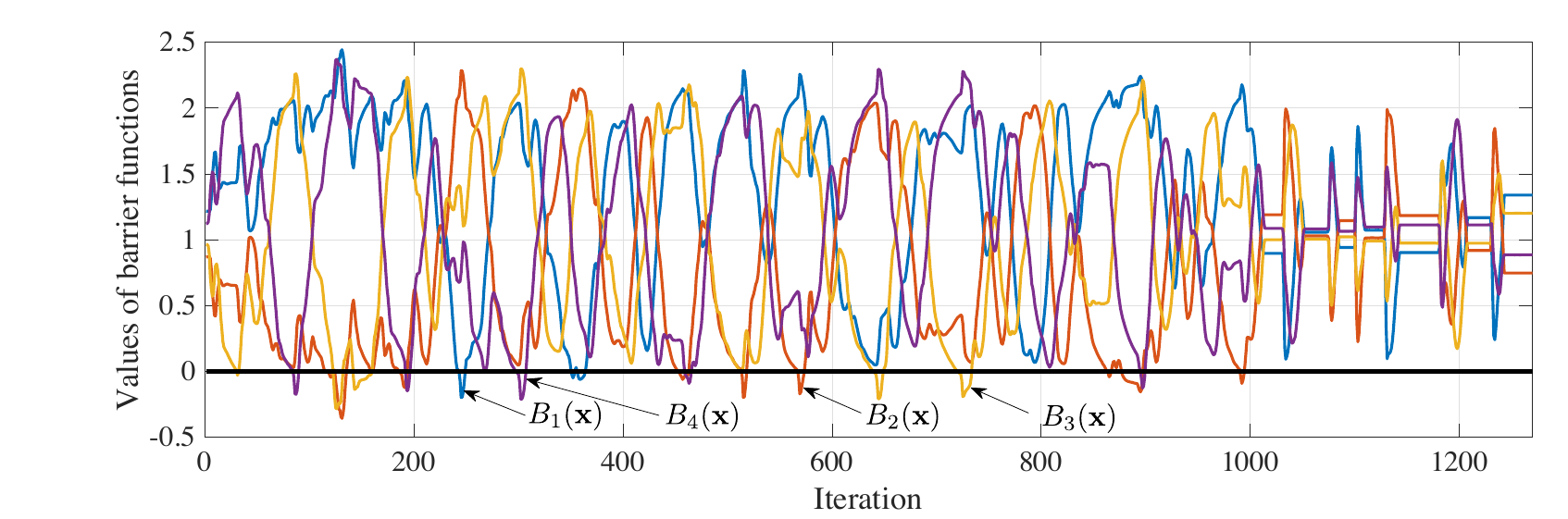}
	\end{center}
	\caption{The values of four control barrier functions employed in the experiment for the whole trajectory.  Even though some violations of safety were seen, the brushbot returned to the safe region before large violations occurred.  The nonholonomic brushbot adaptively learned a model to turn inward the region of exploration before reaching the boundaries of the region of exploration.}
	\label{CBFcurve}
\end{figure*}
%%%%%%

%%%%%%
\reffig{fig:result2} plots the trajectories of the optimal policy learned by the brushbot.
Once the optimal policy replaced random explorations, the brushbot returned to the origin until $n=1016$ as the first figure shows.
The brushbot was pushed by a sweeper at time instant $n=1031,1075,1101,1128,1181$ and $n=1230$, and
the trajectories of the brushbot after being pushed at $n=1031,1075,1101$ are also shown in \reffig{fig:result2}.
Dashed lines in the last figure indicate the time when the brushbot is pushed away.
Given relatively short learning time and the fact that no simulator was used, the brushbot learned the desirable behavior sufficiently well.
\begin{figure*}[t]
	\begin{minipage}{0.32\hsize}
		\begin{center}
			\includegraphics[clip,width=1.0\textwidth]{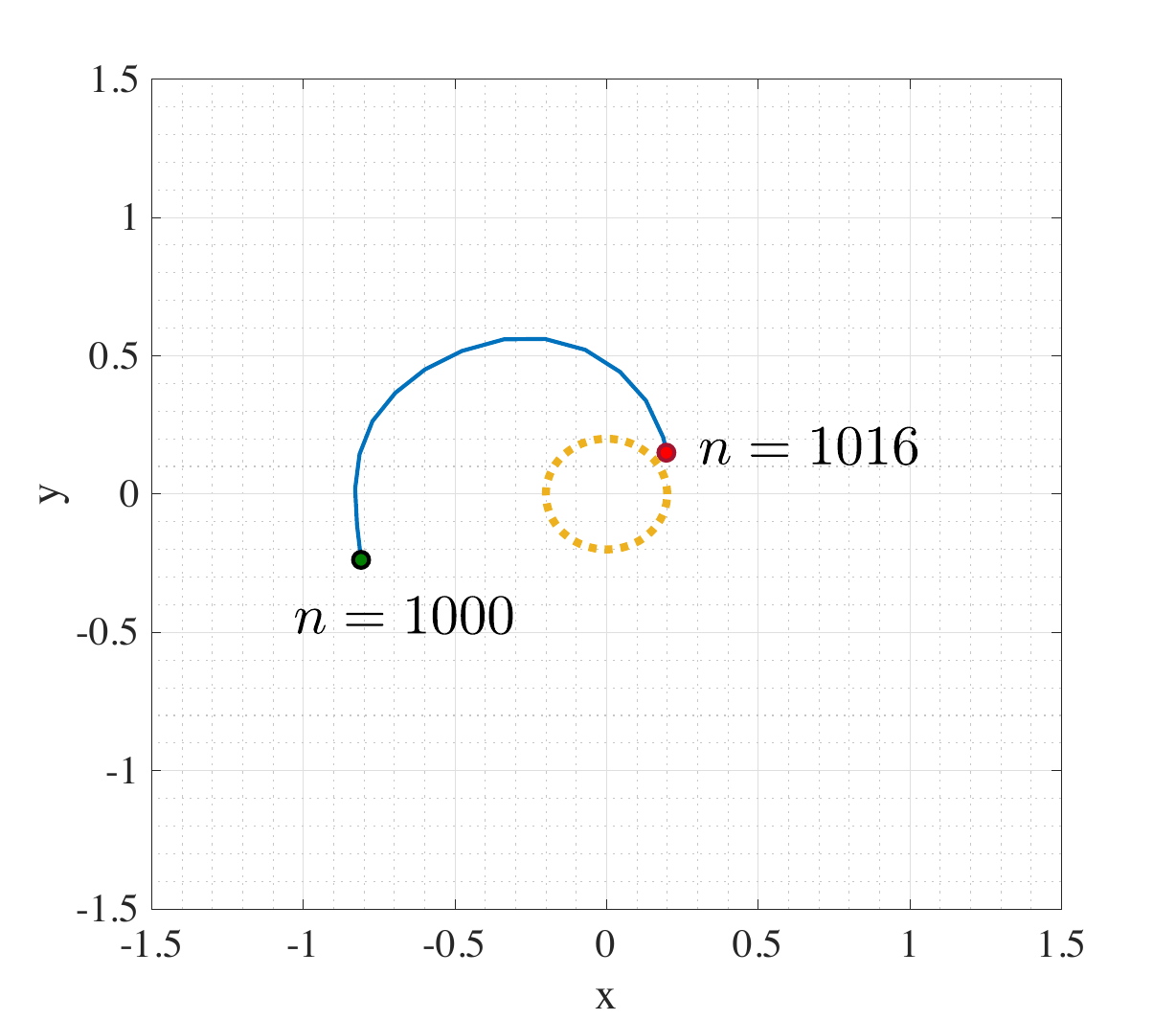}
		\end{center}	
	\end{minipage}
	\begin{minipage}{0.32\hsize}
		\begin{center}
			\includegraphics[clip,width=1.0\textwidth]{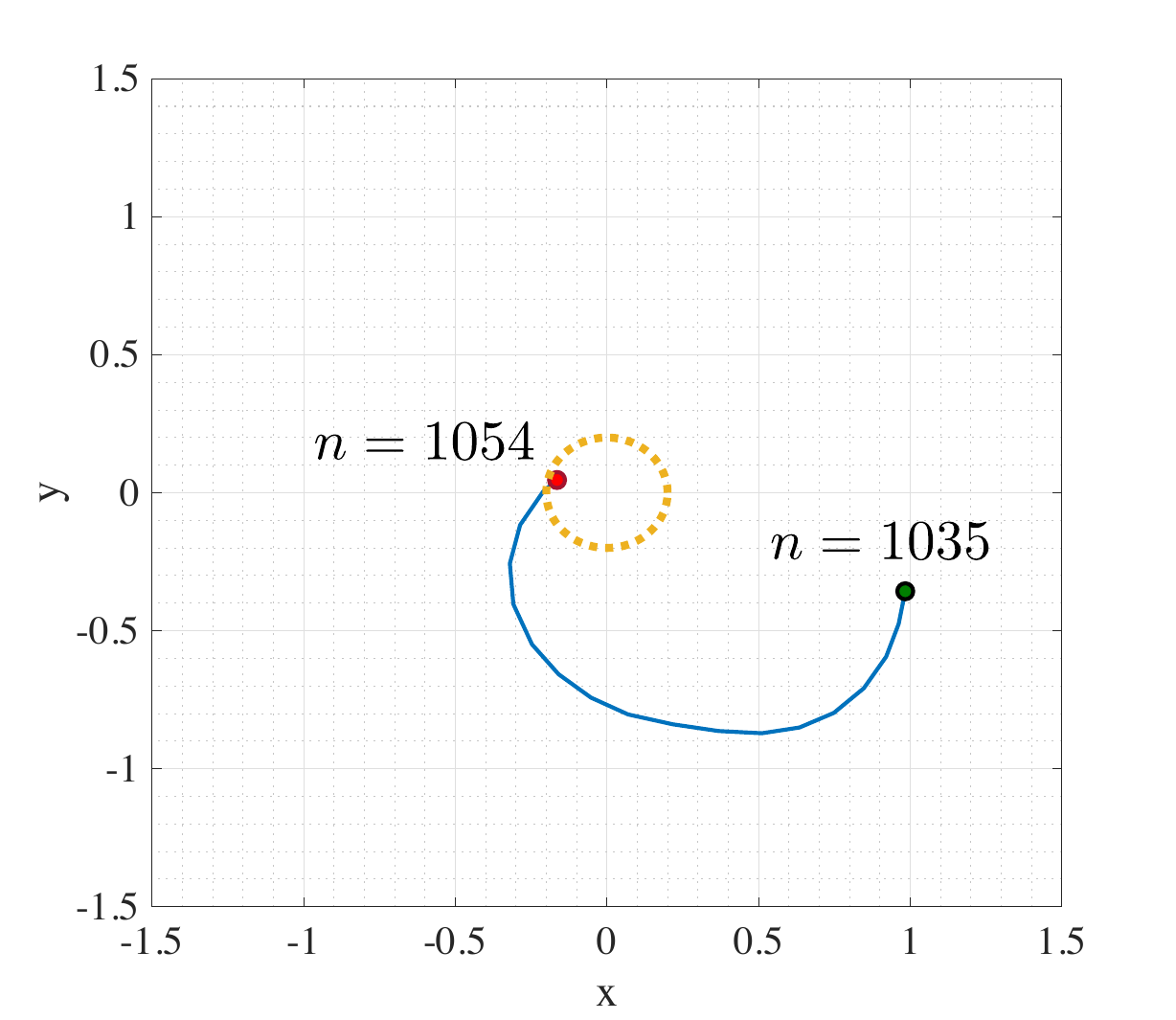}
		\end{center}	
	\end{minipage}
	\begin{minipage}{0.32\hsize}
		\begin{center}
			\includegraphics[clip,width=1.0\textwidth]{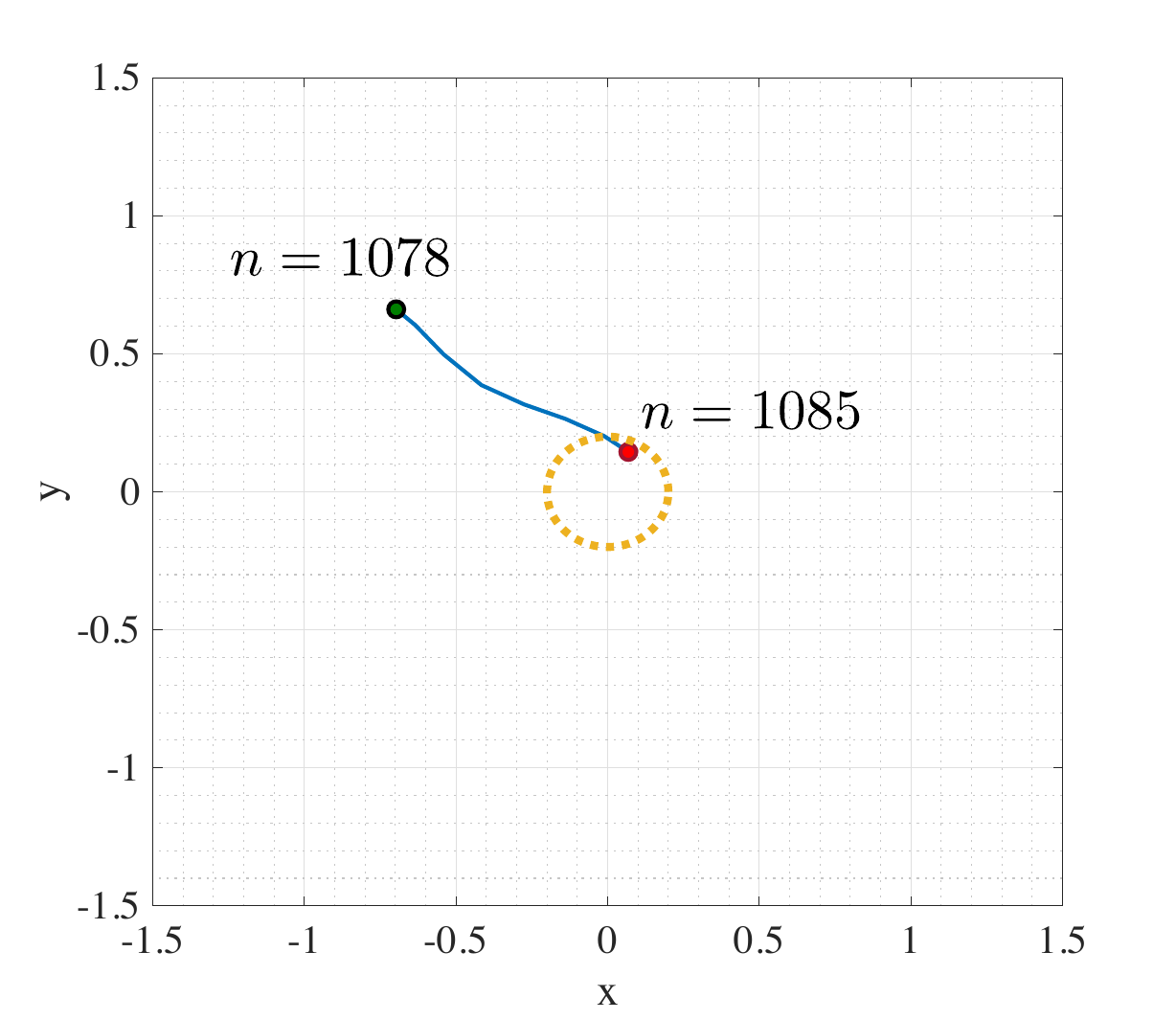}
		\end{center}	
	\end{minipage}
	\begin{minipage}{0.32\hsize}
		\begin{center}
			\includegraphics[clip,width=1.0\textwidth]{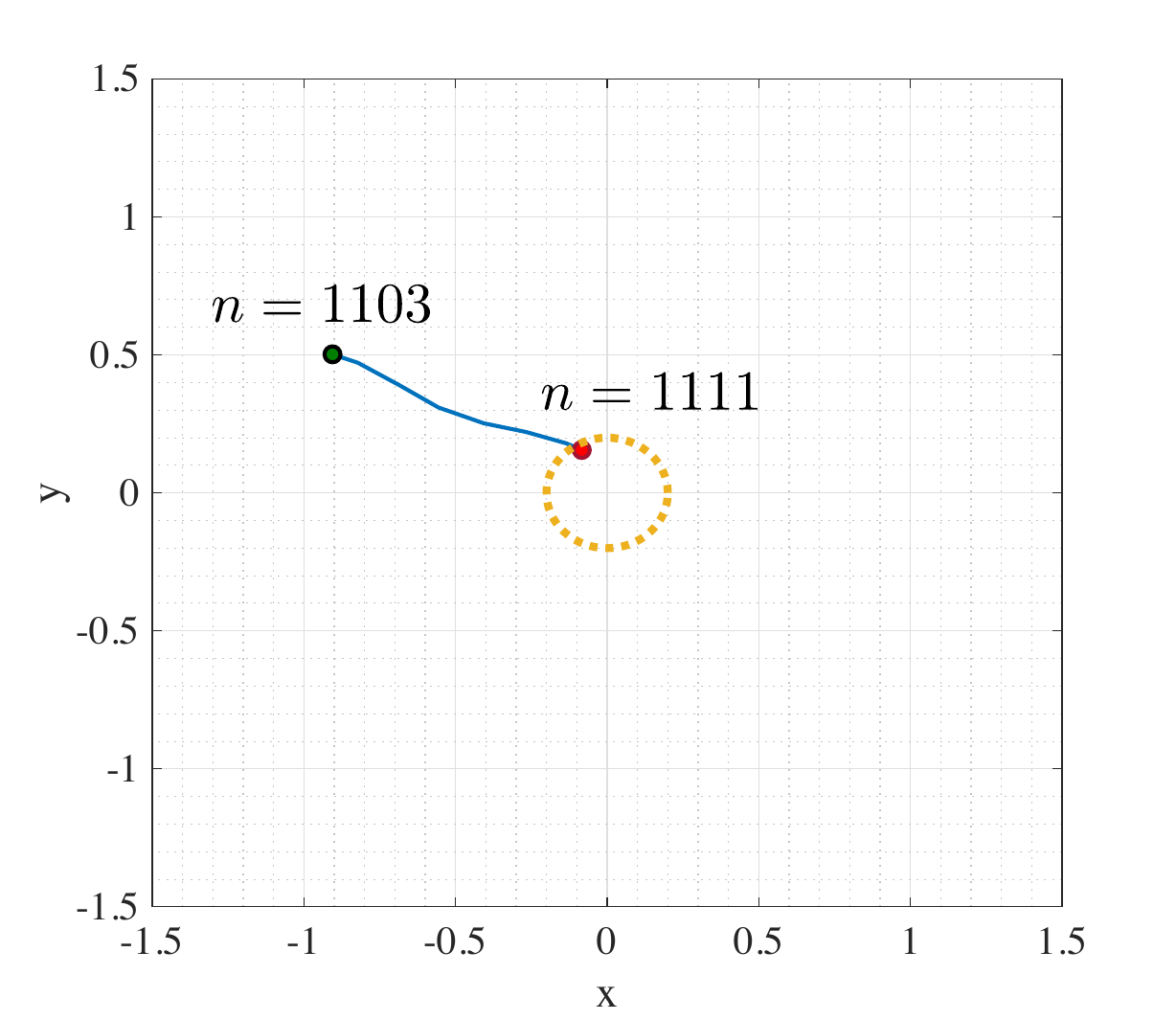}
		\end{center}	
	\end{minipage}
	\begin{minipage}{0.66\hsize}
		\begin{center}
			\includegraphics[clip,width=1.0\textwidth]{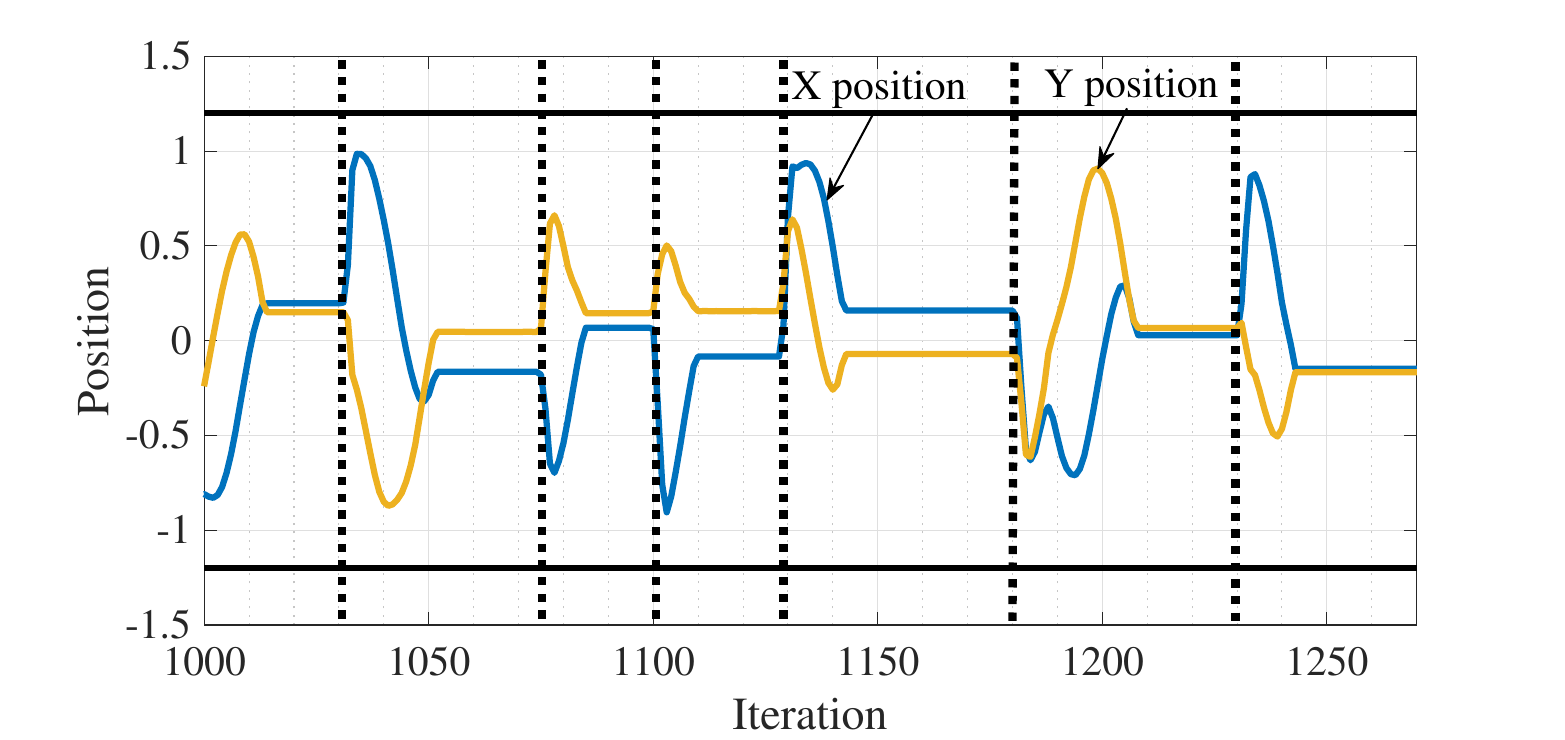}
		\end{center}	
	\end{minipage}
	\caption{Trajectories of the optimal policy learned by the brushbot.  The optimal policy replaced random explorations at $n=1000$, and the brushbot returned to the origin until $n=1016$ (first figure).
		The brushbot was pushed by a sweeper at time instant $n=1031,1075,1101,1128,1181$, and $n=1230$.
		Dashed lines in the last figure indicate the time when the brushbot was pushed away.
		The brushbot learned the desirable behavior sufficiently well. }
	\label{fig:result2}
\end{figure*}
%%%%%%

%%%%%%
\reffig{fig:result3} plots the shape of $\hat{Q}_n^{\phi}\left(\left[\norm{[{\rm x};{\rm y}]}_{\Real^2};0\right],[0;0]\right)$ over X,Y positions at $n=1000$.
It is observed that when the control input is zero (i.e., when the brushbot basically does not move), the vicinity of the origin has the highest value, which is reasonable.
\begin{figure}[t]
	\begin{center}
		\includegraphics[clip,width=0.50\textwidth]{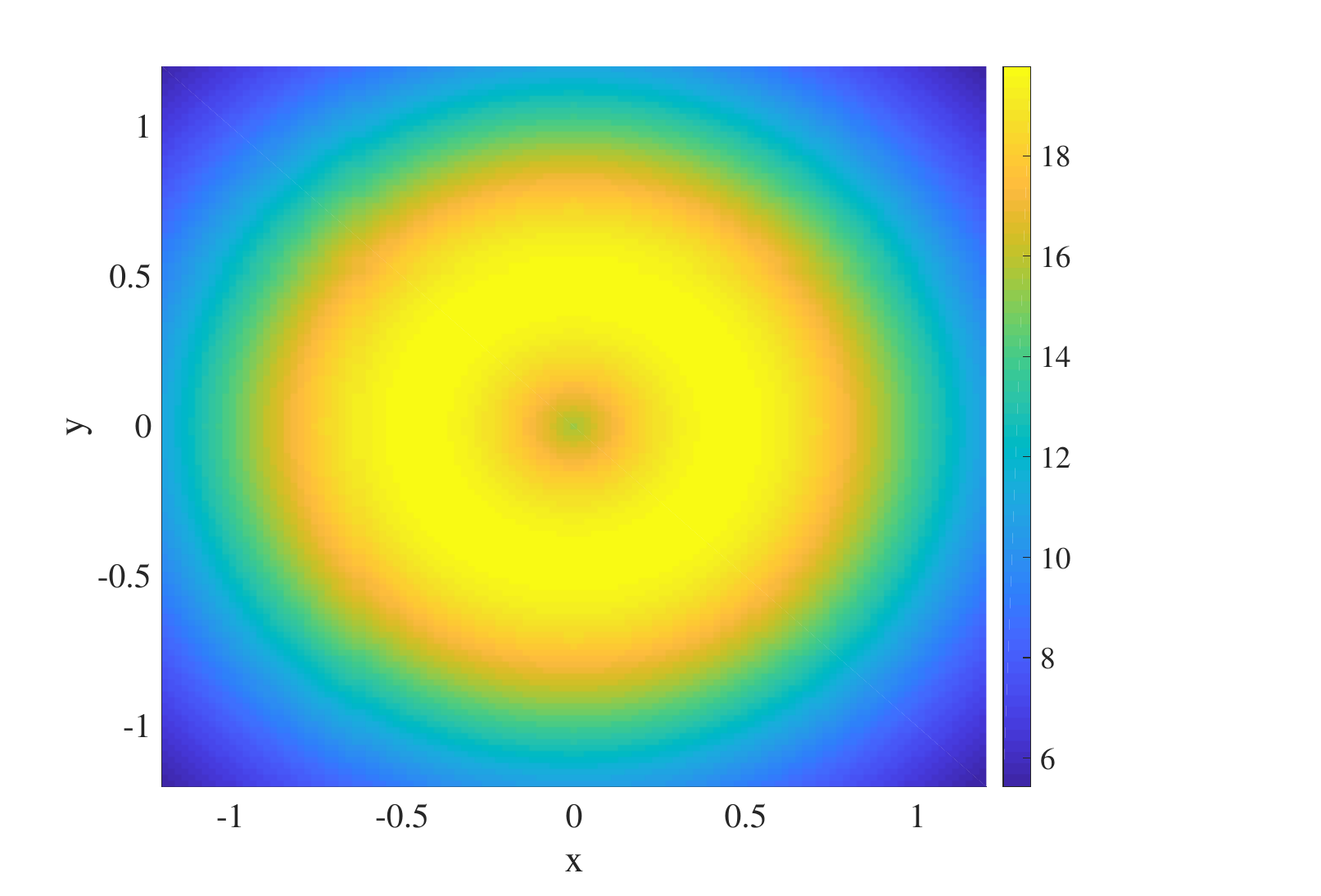}
	\end{center}
	\caption{The shape of the action-value function over X,Y positions at the control input $\sigu=[0;0]$ and $n=1000$.  The vicinity of the origin has the highest value when
		the control input is zero.}
	\label{fig:result3}
\end{figure}
%%%%%%

%%%%%%
Finally, \reffig{frameimage} shows two trajectories of the brushbot returning to the origin by using the action-value function saved at $n=1000$.
After being pushed away from the origin, the brushbot successfully returned to the origin again.
\begin{figure*}[t]
	\begin{minipage}{0.32\hsize}
		\begin{center}
	
			\includegraphics[clip,width=\textwidth]{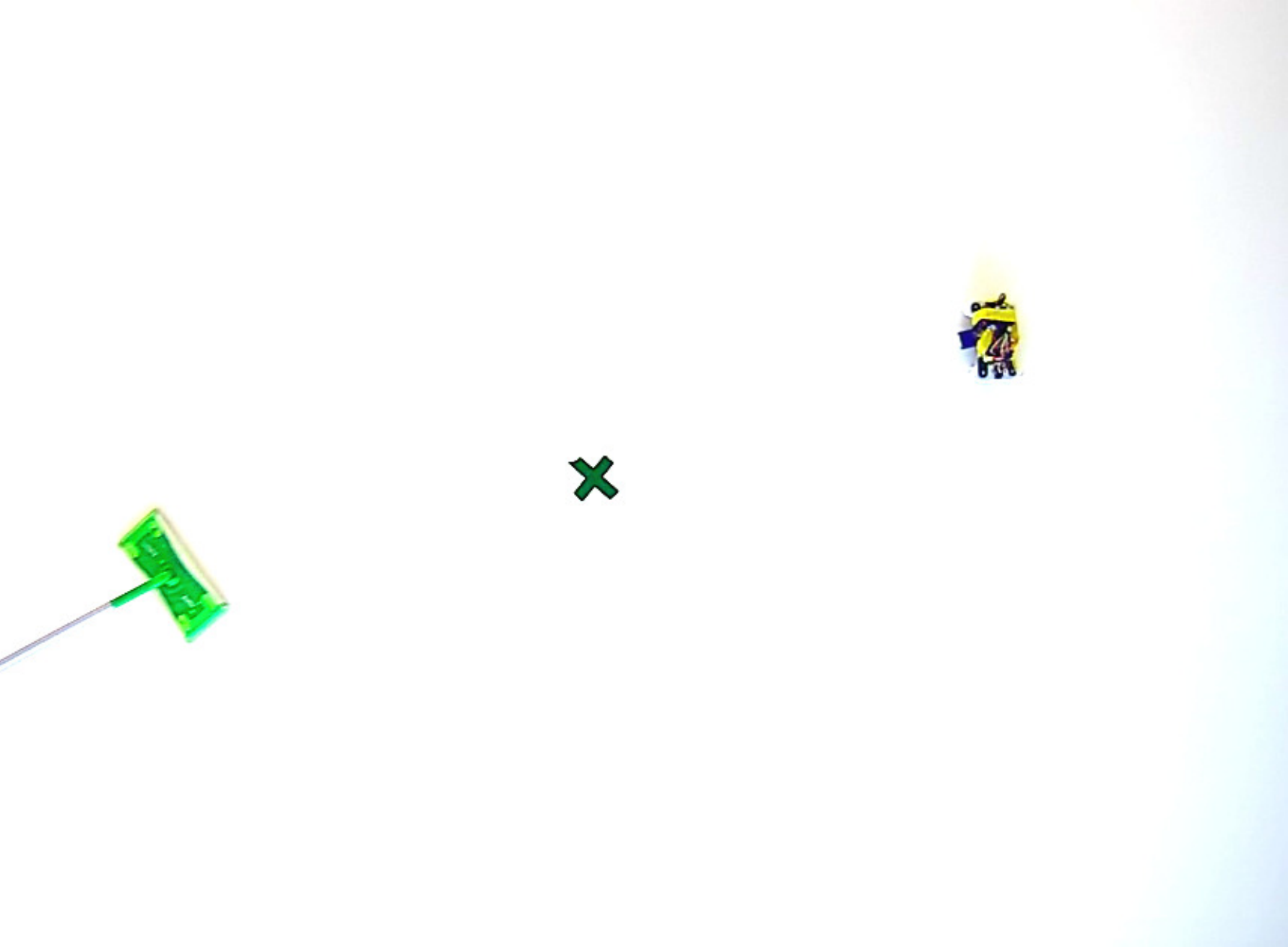}
		\end{center}	
	\end{minipage}
	\begin{minipage}{0.31\hsize}
		\begin{center}
			\includegraphics[clip,width=\textwidth]{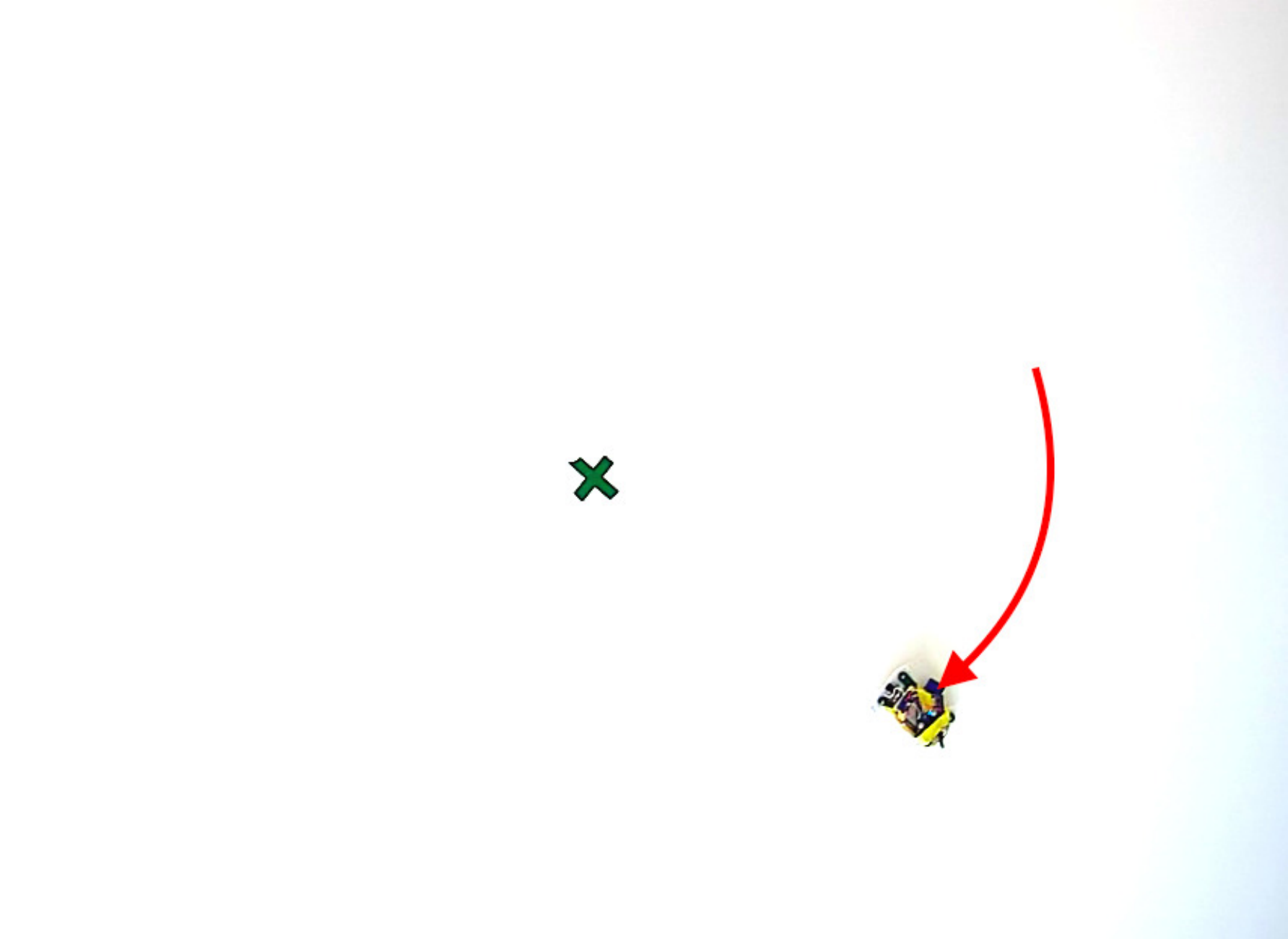}
		\end{center}
	\end{minipage}
	\begin{minipage}{0.31\hsize}
		\begin{center}
			
			\includegraphics[clip,width=\textwidth]{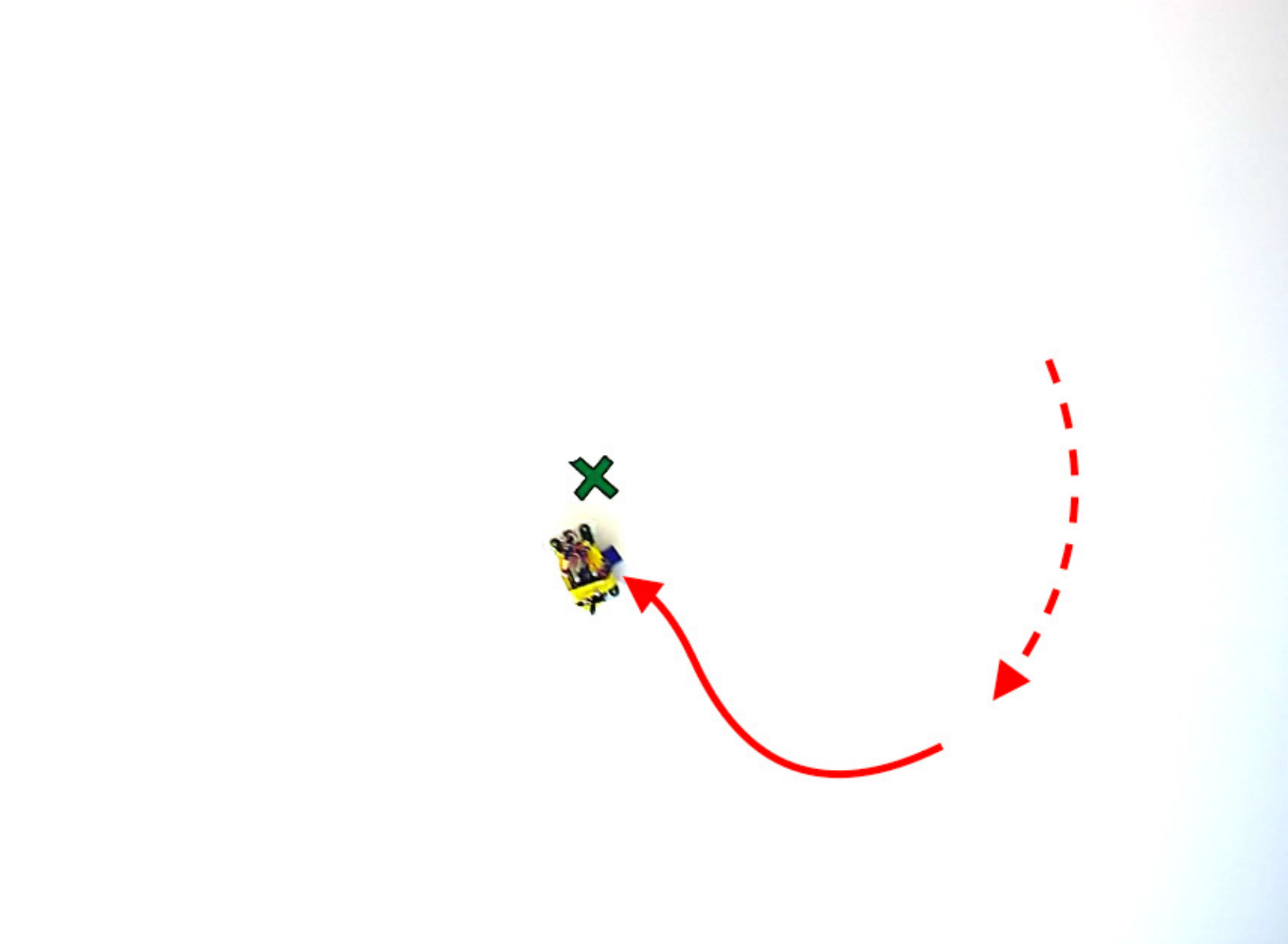}
		\end{center}
	\end{minipage}
	\begin{minipage}{0.32\hsize}
		\begin{center}
			
			\includegraphics[clip,width=\textwidth]{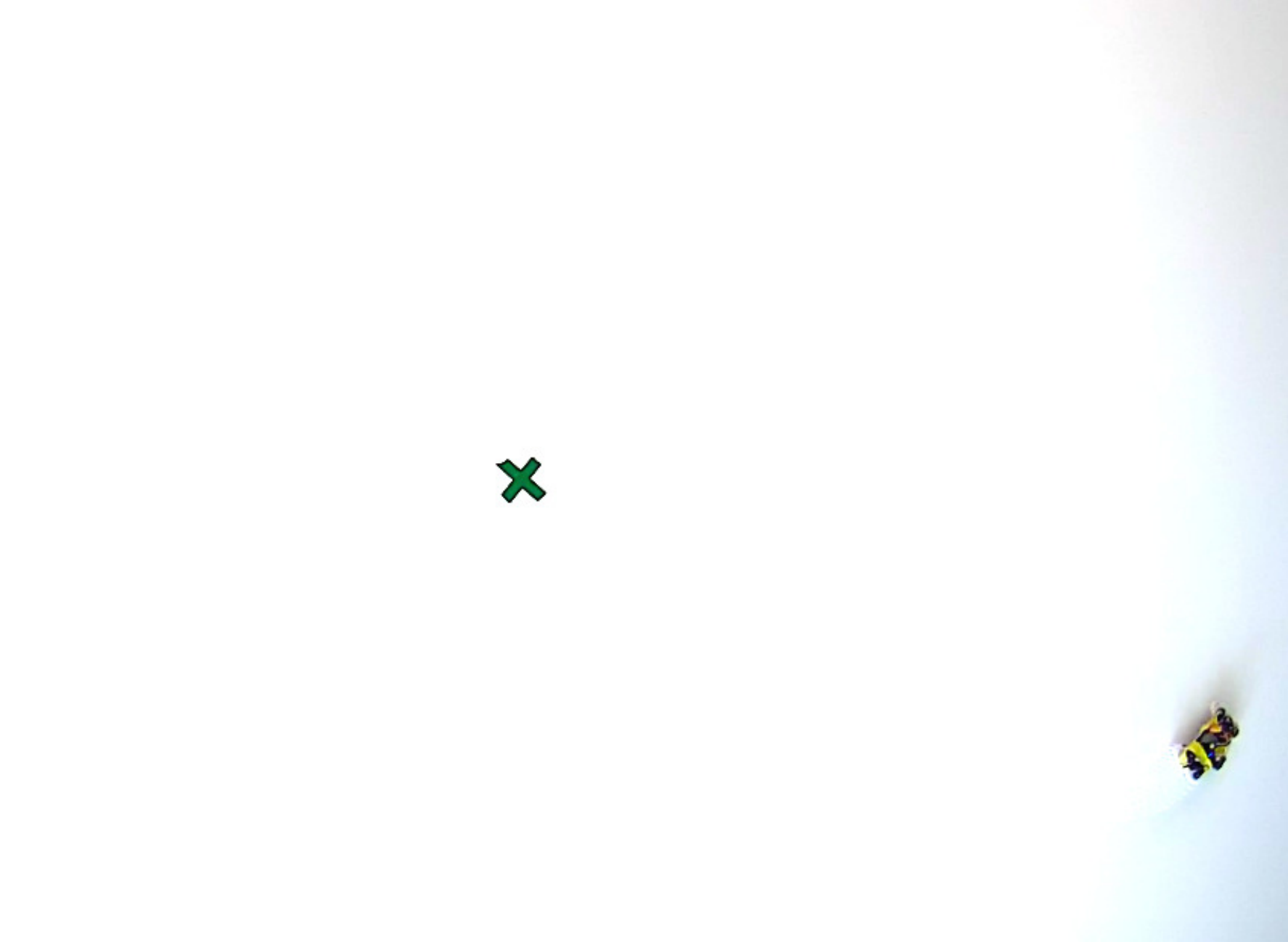}
		\end{center}	
	\end{minipage}
	\begin{minipage}{0.335\hsize}
		\begin{center}

			\includegraphics[clip,width=\textwidth]{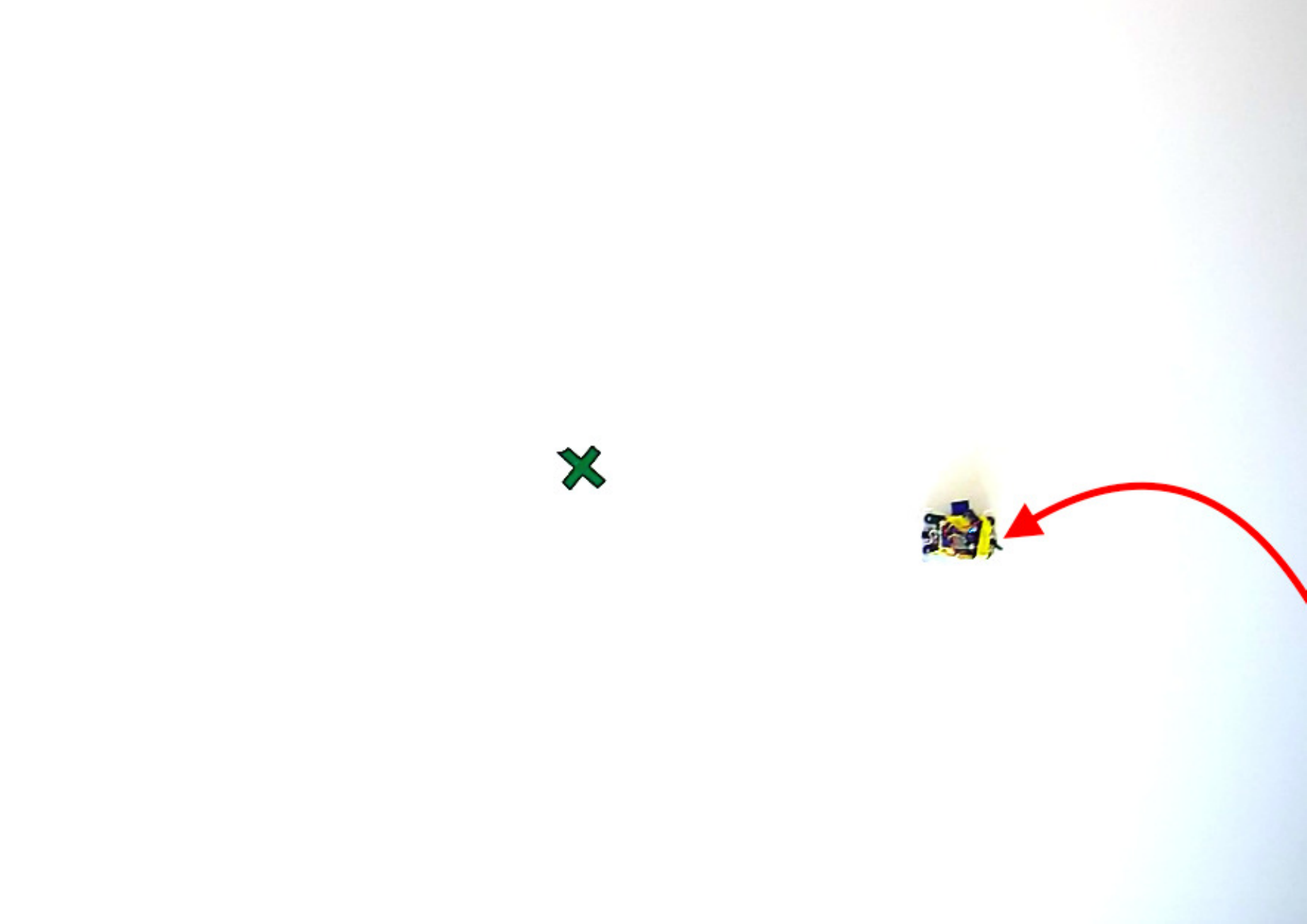}
		\end{center}
	\end{minipage}
	\begin{minipage}{0.335\hsize}
		\begin{center}
			
			\includegraphics[clip,width=\textwidth]{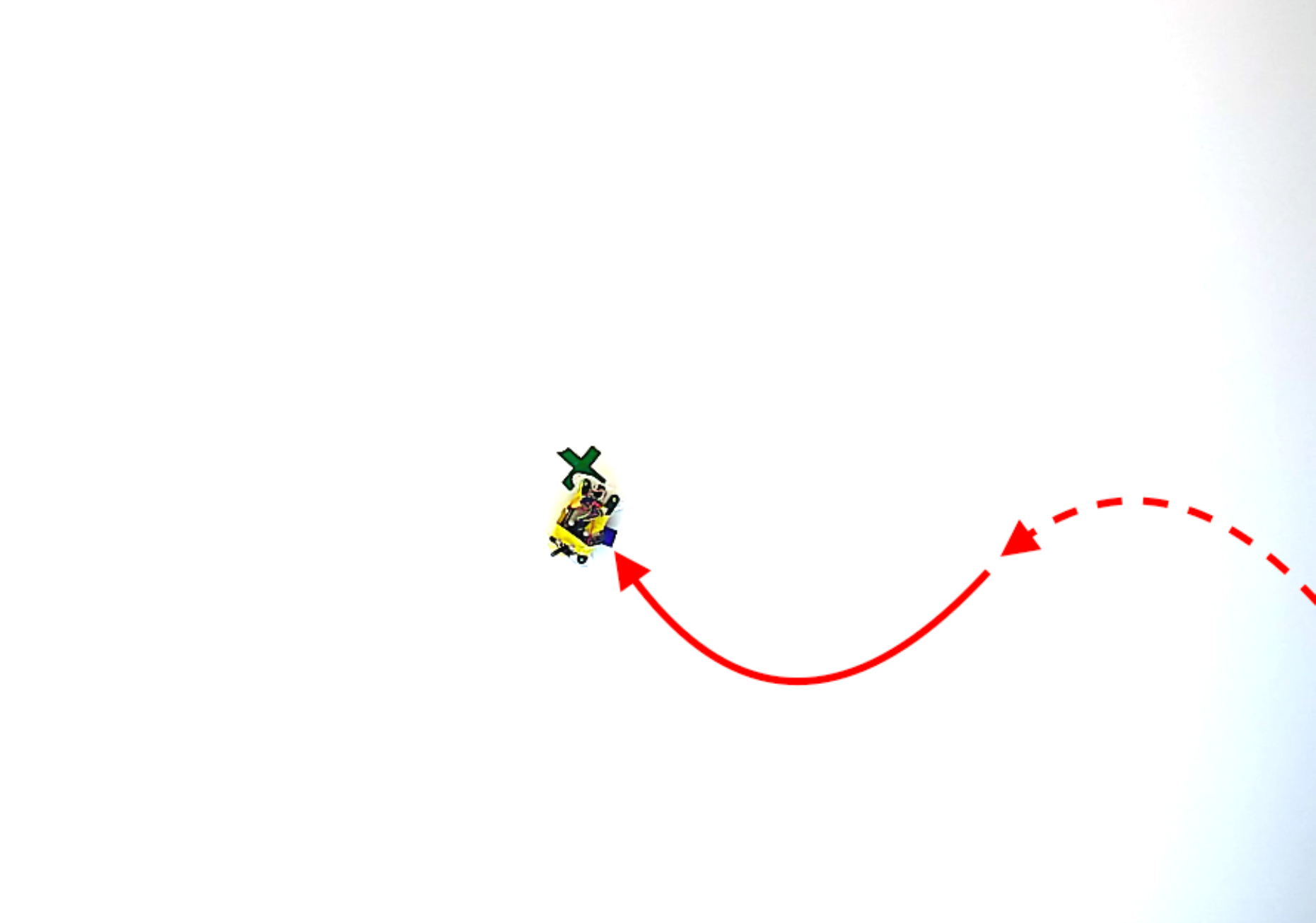}
		\end{center}
	\end{minipage}
	\caption{Two trajectories of the brushbot returning to the origin by using the action-value function saved at $n=1000$.
		Red arrows show the trajectories.
		After being pushed away from the origin, the brushbot successfully returned to the origin again.}
	\label{frameimage}
\end{figure*}
%%%%%%

%%%%%%
\subsubsection{Discussion}
One of the challenges of the experiments is that no initial data or simulators were available.
Despite the fact that the brushbot with highly complex system had to learn an optimal policy while dealing with safety by employing an adaptive model learning algorithm,
the proposed learning framework worked well in the real world.
Brushbot is powered by brushes, and its dynamics highly depends on the conditions of the floor and brushes.
The possible changes of the agent dynamics thus lead to some violations of safety.
Nevertheless, our learning framework recovered safety quickly.
In addition, the agent learned a {\em good} policy within a quite short period.
One reason of those successes of adaptivity and data-efficiency is the convex-analytic formulations.
%%%

%%%
On the other hand, because no initial nominal model or policy is available and our framework is fully adaptive, i.e., we
do {\em not} collect data to conduct batch model learning and/or reinforcement learning, we need to reduce the dimensions of input vectors
to speed-up and robustify learning.  This can be an inherent limitation of our framework. 
%%%%%%%
%%%%%%%%%%%%%%%%%%%%
\section{Conclusion}
\label{sec:conclusion}
The learning framework presented in this paper successfully tied model learning, reinforcement learning, and
barrier certificates, enabling barrier-certified reinforcement learning for unknown, highly nonlinear, nonholonomic, and possibly nonstationary agent dynamics.
The proposed model learning algorithm captures a structure of the agent dynamics by employing a sparse optimization.
The resulting model has preferable structure for preserving efficient computations of barrier certificates.
In addition, recovery of safety after an unexpected and abrupt change of the agent dynamics was guaranteed by
employing barrier certificates and a model learning algorithm with monotone approximation property under certain conditions.
For possibly nonstationary agent dynamics,
the action-value function approximation problem was appropriately reformulated
so that kernel-based methods, including kernel adaptive filter, can be directly applied in an RKHS.
Lastly, certain conditions were also presented to render the set of safe policies convex, thereby
guaranteeing the global optimality of solutions to the policy update to
ensure the greedy improvement of a policy.
The experimental result shows the efficacy of the proposed learning framework
in the real world.
%%%%%%%%%
\begin{appendices}
	\numberwithin{equation}{section}
	\numberwithin{proposition}{section}
	\numberwithin{theorem}{section}
	\renewcommand{\theequation}{\thesection.\arabic{equation}}
	\renewcommand{\theproposition}{\thesection.\arabic{proposition}}
\section{Proof of Proposition \ref{prpCBF}}
\label{appprp3-1}
See {\cite[Proposition~4]{2017dcbf}} for the proof of forward invariance.
The set $\mathcal{C}\subset\X$ is asymptotically stable as
\begin{align}
\lim_{n\rightarrow\infty}B(\sigx_n)\geq \lim_{n\rightarrow\infty}(1-\eta)^nB(\sigx_0)=0, \nn
\end{align}
where the inequality holds from {\cite[Proposition~1]{2017dcbf}}.
\section{Proof of Theorem \ref{maintheo}}
\label{applyapu}
	From Assumptions \ref{assumstab}.1, \ref{assumstab}.2, \ref{assumstab}.5, and from the facts that the estimated output is linear to the model parameter
at a fixed input and that $\norm{\sigx_{n+1}-\hat{\sigx}_{n+1}}\geq0$, we obtain
\begin{align}
\norm{\sigx_{n+1}-\hat{\sigx}_{n+1}}^2_{\Real^{n_x}}-\frac{\varrho^2_1}{\nu^2_B}\leq \varrho^2_4dist^2(\sigh_{n},\Omega_n),\nn
\end{align}
for some bounded $\varrho^2_4\geq0$.
From Assumptions \ref{assumstab}.3, we also obtain that
\begin{align}
|B(\sigx_{n+1})-B(\hat{\sigx}_{n+1})|^2\leq\nu^2_B\norm{\sigx_{n+1}-\hat{\sigx}_{n+1}}^2_{\Real^{n_x}}.
\end{align}
Therefore, from Assumptions \ref{assumstab}.6 and \ref{assumstab}.7, and from $\nu_B\geq0$, we obtain for $\sigh_n^{*}\in\{\sigh\in\Omega|dist(\sigh_n,\Omega)=\norm{\sigh_n-\sigh}_{\Real^{r}}\}$ that
\begin{align}
&|B(\sigx_{n+1})-B(\hat{\sigx}_{n+1})|^2-\varrho^2_1\leq\nu^2_B\norm{\sigx_{n+1}-\hat{\sigx}_{n+1}}^2_{\Real^{n_x}}-\varrho^2_1\nn\\
&\leq\nu^2_B\varrho^2_4dist^2(\sigh_{n},\Omega_n)\nn\\
&\leq\frac{\nu^2_B\varrho^2_4}{\varrho^2_3}\left(\norm{\sigh_{n}-\sigh_n^{*}}^2_{\Real^{r}}-\norm{\sigh_{n+1}-\sigh_n^{*}}^2_{\Real^{r}}\right),\nn\\
&\leq\frac{\nu^2_B\varrho^2_4}{\varrho^2_3}[dist^2(\sigh_{n},\Omega)-dist^2(\sigh_{n+1},\Omega)].\nn
\end{align}
If $B(\sigx_{n+1})<B(\hat{\sigx}_{n+1})$, then we obtain
\begin{align}
&B(\sigx_{n+1})-B(\hat{\sigx}_{n+1})\nn\\
&~~~~~~\geq-\sqrt{\varrho^2_1+\frac{\nu^2_B\varrho^2_4}{\varrho^2_3}[dist^2(\sigh_{n},\Omega)-dist^2(\sigh_{n+1},\Omega)]}\nn\\
&~~~~~~\geq- \varrho_1-\frac{\nu_B\varrho_4}{\varrho_3}\sqrt{[dist^2(\sigh_{n},\Omega)-dist^2(\sigh_{n+1},\Omega)]}.   \label{smallenough}
\end{align}
This inequality also holds in case when $B(\sigx_{n+1})\geq B(\hat{\sigx}_{n+1})$.
Because of the continuity of the cost function $\Theta_n$ and the barrier function $B$ (Assumptions \ref{assumstab}.3 and \ref{assumstab}.5), the set $\mathcal{C}\x\Omega$ is closed.
We show that there exists a Lyapunov function $V_{\mathcal{C}\x\Omega}$ with respect to the closed set $\mathcal{C}\x\Omega$ for the augmented state $[\sigx;\sigh]$.
A candidate function is given by
\begin{align}
&\hspace{-1em}V_{\mathcal{C}\x\Omega}([\sigx;\sigh])\nn\\
&\hspace{-1em}~~~=\begin{cases}
0 \hfill~~~{\rm if}\;[\sigx;\sigh]\in \mathcal{C}\x\Omega\\
-\min(B(\sigx),0)+\frac{2\nu_B\varrho_4}{\varrho_2\varrho^2_3}dist^2(\sigh,\Omega)\hfill~~~{\rm if}\;[\sigx,\sigh]\notin \mathcal{C}\x\Omega
\end{cases}\nn
\end{align}
Since $-\min(B(\sigx),0)+\frac{2\nu_B\varrho_4}{\varrho_3}dist^2(\sigh,\Omega)=0$, $\forall[\sigx;\sigh]\in \partial(\mathcal{C}\x\Omega)$, where $\partial(\mathcal{C}\x\Omega)$ 
is the boundary of the set $\mathcal{C}\x\Omega$, from Assumption \ref{assumstab}.3, the function $V_{\mathcal{C}\x\Omega}$ is continuous.
It also holds that $V_{\mathcal{C}\x\Omega}([\sigx;\sigh])>0$ when $[\sigx;\sigh]\notin \mathcal{C}\x\Omega$.
Under Assumption \ref{assumstab}.6, we obtain
\begin{align}
&V_{\mathcal{C}\x\Omega}([\sigx_{n+1};\sigh_{n+1}])-V_{\mathcal{C}\x\Omega}([\sigx_{n};\sigh_{n}])\nn\\
&=-\min(B(\sigx_{n+1}),0)+\frac{2\nu_B\varrho_4}{\varrho_2\varrho^2_3}dist^2(\sigh_{n+1},\Omega)\nn\\
&+\min(B(\sigx_{n}),0)-\frac{2\nu_B\varrho_4}{\varrho_2\varrho^2_3}dist^2(\sigh_{n},\Omega)\nn\\
&\leq-\frac{\nu_B\varrho_4}{\varrho_2\varrho^2_3}[dist^2(\sigh_{n},\Omega)-dist^2(\sigh_{n+1},\Omega)]\leq 0,  \label{Lyapu}
\end{align}
for all $n\in\integer_{\geq0}$.
To show that the first inequality holds, we first show
\begin{align}
&-\min(B(\sigx_{n+1}),0)+\min(B(\sigx_n),0)\nn\\
&~~~~\leq \frac{\nu_B\varrho_4}{\varrho_3}\sqrt{[dist^2(\sigh_{n},\Omega)-dist^2(\sigh_{n+1},\Omega)]}. \nn
\end{align}
(a) For $B(\sigx_n)\geq0$:
from \refeq{DCBFdef3}, \refeq{smallenough}, and $0<\eta\leq1$, we obtain
$B(\hat{\sigx}_{n+1})\geq\varrho_1$ and $B(\sigx_{n+1})\geq-\frac{\nu_B\varrho_4}{\varrho_3}\sqrt{[dist^2(\sigh_{n},\Omega)-dist^2(\sigh_{n+1},\Omega)]}$
from which it follows that
\begin{align}
\hspace{-1.5em}-\min(B(\sigx_{n+1}),0)\leq \frac{\nu_B\varrho_4}{\varrho_3}\sqrt{[dist^2(\sigh_{n},\Omega)-dist^2(\sigh_{n+1},\Omega)]}. \nn
\end{align}
(b) For $B(\sigx_n)<0$ and $B(\sigx_{n+1})\geq0$:
it is straightforward to see that
\begin{align}
&-\min(B(\sigx_{n+1}),0)+\min(B(\sigx_n),0)=B(\sigx_n)\nn\\
&~~~~<0\leq \frac{\nu_B\varrho_4}{\varrho_3}\sqrt{[dist^2(\sigh_{n},\Omega)-dist^2(\sigh_{n+1},\Omega)]}. \nn
\end{align}
(c) For $B(\sigx_n)<0$ and $B(\sigx_{n+1})<0$:
from \refeq{DCBFdef3}, \refeq{smallenough}, and $0<\eta\leq1$, we obtain
$B(\hat{\sigx}_{n+1})\geq B(\sigx_n)+\varrho_1$ and $B(\sigx_{n+1})-B(\sigx_n)\geq-\frac{\nu_B\varrho_4}{\varrho_3}\sqrt{[dist^2(\sigh_{n},\Omega)-dist^2(\sigh_{n+1},\Omega)]}$
from which it follows that
\begin{align}
&-\min(B(\sigx_{n+1}),0)+\min(B(\sigx_n),0)=B(\sigx_n)-B(\sigx_{n+1})\nn\\
&~~~~~~\leq \frac{\nu_B\varrho_4}{\varrho_3}\sqrt{[dist^2(\sigh_{n},\Omega)-dist^2(\sigh_{n+1},\Omega)]}. \nn
\end{align}
If $\sigh_n\notin\Omega_n$, under Assumption \ref{assumstab}.6,
we obtain $\varrho_2\varrho_3\leq\sqrt{[dist^2(\sigh_{n},\Omega)-dist^2(\sigh_{n+1},\Omega)]}$ from which it follows that
$\sqrt{[dist^2(\sigh_{n},\Omega)-dist^2(\sigh_{n+1},\Omega)]}\leq\frac{1}{\varrho_2\varrho_3}[dist^2(\sigh_{n},\Omega)-dist^2(\sigh_{n+1},\Omega)]$ and
\begin{align}
&V_{\mathcal{C}\x\Omega}([\sigx_{n+1};\sigh_{n+1}])-V_{\mathcal{C}\x\Omega}([\sigx_{n};\sigh_{n}])\nn\\
&\leq\frac{\nu_B\varrho_4}{\varrho_3}\sqrt{[dist^2(\sigh_{n},\Omega)-dist^2(\sigh_{n+1},\Omega)]}\nn\\
&~~~~~~~~-\frac{2\nu_B\varrho_4}{\varrho_2\varrho^2_3}[dist^2(\sigh_{n},\Omega)-dist^2(\sigh_{n+1},\Omega)]\nn\\
&\leq-\frac{\nu_B\varrho_4}{\varrho_2\varrho^2_3}[dist^2(\sigh_{n},\Omega)-dist^2(\sigh_{n+1},\Omega)],\nn
\end{align}
and the first inequality of \refeq{Lyapu} holds.
The inequality also holds for $\sigh_n\in\Omega_n$.
Moreover, if $[\sigx_n;\sigh_n]\in \mathcal{C}\x\Omega$, then $\sigh_n$ remains in $\Omega$ because of monotonic approximation property.
From \refeq{smallenough},
the control barrier certificate \refeq{DCBFdef2} is thus ensured with a control input satisfying \refeq{DCBFdef3} under Assumption \ref{assumstab}.4,
and the set $\mathcal{C}\x\Omega$ is forward invariant.
Therefore, the system for the augmented state is stable with respect to the set $\mathcal{C}\x\Omega$.
If $\sigh_n\notin\Omega_n$ for all $n\in\integer_{\geq0}$ such that $[\sigx_n;\sigh_n]\notin\mathcal{C}\x\Omega$, it follows that
\begin{align}
V_{\mathcal{C}\x\Omega}([\sigx_{n+1};\sigh_{n+1}])-V_{\mathcal{C}\x\Omega}([\sigx_{n};\sigh_{n}])<0,
\end{align}
and {\cite[Theorem~1]{jiang2002converse}} applies, i.e., the system for the augmented state is uniformly globally asymptotically stable with respect to the set $\mathcal{C}\x\Omega$.
\section{Proof of Lemma \ref{prp1}}
\label{appprp1}
Since $\kappa(\sigu,\sigv)=\textbf{1}(\sigu)=1,\forall \sigu,\sigv\in\U$, is a positive definite kernel, it defines
the unique RKHS given by ${\rm span}\{\textbf{1}\}$, which is complete because it is a finite-dimensional space.
For any $\varphi:=\alpha\textbf{1}\in\H_c$, $\inpro{\varphi,\varphi}_{\H_c}=\alpha^2\geq0$ and the equality holds if and only if $\alpha=0$, or equivalently, $\varphi=0$.
The symmetry and the linearity also hold, and hence $\inpro{\cdot,\cdot}_{\H_c}$ defines the inner product.
For any $\sigu\in\U$, it holds that
$\inpro{\varphi,\kappa(\cdot,\sigu)}_{\H_c}=\inpro{\alpha\textbf{1},\textbf{1}}_{\H_c}=\alpha=\varphi(\sigu)$.
Therefore, the reproducing property is satisfied.	
\section{Proof of Theorem \ref{directtheorem}}
\label{apptheo1}
The following lemmas are used to prove the theorem.
\begin{lemma}[{\cite[Theorem~2]{minh2010some}}]
	\label{hojo1}
	Let $\X\subset\Real^{n_x}$ be any set with nonempty interior.  Then, the RKHS associated with the Gaussian kernel for an arbitrary scale parameter $\sigma>0$
	does not contain any polynomial on $\X$, including the nonzero constant function.
\end{lemma}
\begin{lemma}
	\label{hojo2}
	Assume that $\X\subset\Real^{n_x}$ and $\U\subset\Real^{n_u}$ have nonempty interiors.  Then, the intersection of the RKHS
	$\H_u$ associated with the kernel $\ka{\sigu,\sigv}:=\sigu^{\T}\sigv,\;\sigu,\sigv\in\U$, and the RKHS $\H_c$ is $\{0\}$, i.e.,
	\begin{align}
	\H_c \cap \H_u = \{0\}.\nn
	\end{align}
\end{lemma}	
\begin{proof}
	It is obvious that the function $\varphi(\sigu)=0,\forall \sigu\in \U$, is an element of both of the RKHSs (vector spaces) $\H_u$ and $\H_c$.
	Therefore, it is sufficient to show that there exists $\sigu\in\U$ satisfying that $\varphi(\sigu)\neq \varphi(\sigu^{\rm int}),\;\sigu^{\rm int}\in{\rm int}(\U)$, where
	${\rm int}(\U)$ denotes the interior of $\U$, for any
	$\varphi\in\H_u\setminus\{0\}$.
	Assume that $\varphi(\sigv)\neq0$ for some $\sigv\in\U$.
	From {\cite[Theorem~3]{moor1}}, the RKHS $\H_u$ is expressed as $\H_u={\rm span}\{\ka{\cdot,\sigu}\}_{\sigu\in\U}$, which is finite dimension,
	implying that any function in $\H_u$ is linear.
	Since there exists $\sigu=\sigu^{\rm int}+\varrho_5 \sigv\in \U$ for some $\varrho_5>0$,
	it is proved that
	\begin{align}
	&\varphi(\sigu)=\varphi(\sigu^{\rm int}+\varrho_5 \sigv)=\varphi(\sigu^{\rm int})+\varrho_5\varphi(\sigv)\neq \varphi(\sigu^{\rm int}).\nn
	\end{align}
\end{proof}	
\begin{lemma}[{\cite[Proposition~1.3]{tensorproduct}}]
	\label{hojo3}
	If $\H_2=\H_{21}\oplus\H_{22}$ for given vector spaces $\H_1$ and $\H_2$, then
	$\H_1\otimes\H_{21}\cap\H_1\otimes\H_{22}=\{0\},$
	i.e.,
	$\H_1\otimes\H_2=(\H_1\otimes\H_{21})\oplus(\H_1\otimes\H_{22}).$
\end{lemma}	
\begin{lemma}
	\label{hojo4}
	Given $\X\subset\Real^{n_x}$ and $\U\subset\Real^{n_u}$,
	let $\H_1$, $\H_2$, and $\H$ be associated with the Gaussian kernels
	$\kappa_1(\sigx,\sigy):=\frac{1}{(\sqrt{2\pi}\s)^{n_x}}\exp\left(-\frac{\norm{\sigx-\sigy}_{\Real^{n_x}}^2}{2\s^2}\right),\;\sigx,\sigy\in\X$,
	$\kappa_2(\sigu,\sigv):=\frac{1}{(\sqrt{2\pi}\s)^{n_u}}\exp\left(-\frac{\norm{\sigu-\sigv}_{\Real^{n_u}}^2}{2\s^2}\right),\;\sigu,\sigv\in\U$,
	and $\kappa([\sigx;\sigu],[\sigy;\sigv]):=\frac{1}{(\sqrt{2\pi}\s)^{n_x+n_u}}
	\exp\left(-\frac{\norm{[\sigx;\sigu]-[\sigy;\sigv]}_{\Real^{n_x+n_u}}^2}{2\s^2}\right),\;\sigx,\sigy\in\X,\;\sigu,\sigv\in\U$, respectively,
	for an arbitrary $\s>0$.
	Then, by regarding a function in $\H_1\otimes\H_2$ as a function over the input space $\X\x\U\subset\Real^{n_x+n_u}$, it holds that
	\begin{align}
	\H=\H_1\otimes\H_2.\nn
	\end{align}
\end{lemma}	
\begin{proof}
	$\H_1\otimes\H_2$ has the reproducing kernel defined by
	\begin{align}
	&\kappa_{\otimes}([\sigx;\sigu],[\sigy;\sigv]):=\kappa_1(\sigx,\sigy)\kappa_2(\sigu,\sigv)\nn\\
	&=\frac{1}{(\sqrt{2\pi}\s)^{n_x}(\sqrt{2\pi}\s)^{n_u}}\nn\\
	&\exp\left(-\frac{\norm{\sigx-\sigy}_{\Real^{n_x}}^2}{2\s^2}\right)\exp\left(-\frac{\norm{\sigu-\sigv}_{\Real^{n_u}}^2}{2\s^2}\right)\nn\\
	&=\frac{1}{(\sqrt{2\pi}\s)^{n_x+n_u}}
	\exp\left(-\frac{\norm{\sigx-\sigy}_{\Real^{n_x}}^2+\norm{\sigu-\sigv}_{\Real^{n_u}}^2}{2\s^2}\right)\nn\\
	&=\kappa([\sigx;\sigu],[\sigy;\sigv]).\nn
	\end{align}
	This verifies the claim.
\end{proof}	
We are now ready to prove Theorem \ref{directtheorem}.\\
\begin{proof}[Proof of Theorem \ref{directtheorem}]
	By Lemmas \ref{hojo2} and \ref{hojo3}, it is derived that $\H_f\otimes\H_c\cap\H_g\otimes\H_u=\{0\}$.
	By Lemmas \ref{hojo1}, \ref{hojo3}, and \ref{hojo4}, it holds that $\H_p\cap\H_f\otimes\H_c=\{0\}$ and $\H_p\cap\H_g\otimes\H_u=\{0\}$.
\end{proof}
%%%%%%%
\section{Proof of Theorem \ref{theoRKHS}}
\label{apptheo2}
We show that the operator $U:\H_Q\rightarrow\H_{\psi^Q}$, which maps $\varphi^Q\in\H_Q$ to a function $\varphi\in\H_{\psi^Q},\varphi([\sigz;\sigw])=\varphi^Q(\sigz)-\gamma \varphi^Q(\sigw)$
where $\gamma\in(0,1),\;\sigz,\sigw\in\mathcal{Z}$, is bijective.
Because the mapping $U$ is surjective by definition,
we show it is also injective.
For any $\varphi_1^Q,\varphi_2^Q\in\H_Q$, 
\begin{align}
&\hspace{-1em}U(\varphi_1^Q+\varphi_2^Q)([\sigz;\sigw])=(\varphi_1^Q+\varphi_2^Q)(\sigz)-\gamma(\varphi_1^Q+\varphi_2^Q)(\sigw)\nn\\
&\hspace{-1em}=(\varphi_1^Q(\sigz)-\gamma \varphi_1^Q(\sigw))+(\varphi_2^Q(\sigz)-\gamma \varphi_2^Q(\sigw))\nn\\
&\hspace{-1em}=U(\varphi_1^Q)([\sigz;\sigw])+U(\varphi_2^Q)([\sigz;\sigw]),\;\forall \sigz,\sigw\in\mathcal{Z},\nn
\end{align}
and 
\begin{align}
&U(\alpha \varphi_1^Q)([\sigz;\sigw])\nn\\
&=\alpha \varphi_1^Q(\sigz)-\gamma\alpha \varphi_1^Q(\sigw)=\alpha (\varphi_1^Q(\sigz)-\gamma \varphi_1^Q(\sigw))\nn\\
&=\alpha U(\varphi_1^Q)([\sigz;\sigw]),\;\forall \alpha\in\Real,\;\forall \sigz,\sigw\in\mathcal{Z},\nn
\end{align}
from which the linearity holds.
Therefore, it is sufficient to show that ${\rm ker}(U)=0$ \cite{linearalgebra}.  For any $\varphi^Q\in{\rm ker}(U)$, we obtain
\begin{align}
&U(\varphi^Q)([\sigz;\sigz])=(1-\gamma)\varphi^Q(\sigz)=0,\;\;\forall \sigz\in\mathcal{Z},\nn
\end{align}	
which implies that $\varphi^Q=0$.
%%%

%%%
Next, we show that $\H_{\psi^Q}$ is an RKHS.
The space $\H_{\psi^Q}$ with the inner product defined in \refeq{theoinpro} is isometric to the RKHS $H_Q$, and hence 
is a Hilbert space.
Because $\kappa^Q(\cdot,\sigz)-\gamma\kappa^Q(\cdot,\sigw)\in\H_Q$, it is true that $\kappa(\cdot,[\sigz;\sigw])\in\H_{\psi^Q}$.
Moreover, it holds that
\begin{align}
&\hspace{-1em}\inpro{\kappa(\cdot,[\sigz;\sigw]),\kappa(\cdot,[\tilde{\sigz};\tilde{\sigw}])}_{\H_{\psi^Q}}\nn\\
&\hspace{-1em}=\inpro{\kappa^Q(\cdot,\sigz)-\gamma\kappa^Q(\cdot,\sigw),\kappa^Q(\cdot,\tilde{\sigz})-\gamma\kappa^Q(\cdot,\tilde{\sigw})}_{\H_Q}\nn\\
&\hspace{-1em}=\left(\kappa^Q(\sigz,\tilde{\sigz})-\gamma\kappa^Q(\sigz,\tilde{\sigw})\right)-\gamma\left(\kappa^Q(\sigw,\tilde{\sigz})-\gamma\kappa^Q(\sigw,\tilde{\sigw})\right)\nn\\
&\hspace{-1em}=\kappa([\sigz;\sigw],[\tilde{\sigz};\tilde{\sigw}]),\nn
\end{align}	
and that
\begin{align}
&\hspace{-1em}\inpro{\varphi,\kappa(\cdot,[\sigz;\sigw])}_{\H_{\psi^Q}}=\inpro{\varphi^Q,\kappa^Q(\cdot,\sigz)-\gamma\kappa^Q(\cdot,\sigw)}_{\H_Q}\nn\\
&\hspace{-1em}=\varphi^Q(\sigz)-\gamma \varphi^Q(\sigw)=\varphi([\sigz;\sigw]),\;\forall \varphi\in\H_{\psi^Q}.\nn
\end{align}	
Therefore, $\kappa(\cdot,\cdot):\mathcal{Z}^2\x\mathcal{Z}^2\rightarrow\Real$ is the reproducing kernel with which the RKHS $\H_{\psi^Q}$ is associated.	
%%%%%%%
\section{Proof of Corollary \ref{theoNonex}}
\label{apptheo3}
From the definition of the inner product in the RKHS $\H_{\psi^Q}$, it follows that
\begin{align}
&\norm{\hat{Q}_{n+1}^{\phi}-{Q^{\phi}}^{*}}_{\H_{Q}}=\norm{\hat{\psi}_{n+1}^Q-{\psi^Q}^{*}}_{\H_{\psi^Q}}\nn\\
&\leq\norm{\hat{\psi}_n^Q-{\psi^Q}^{*}}_{\H_{\psi^Q}}=\norm{\hat{Q}_n^{\phi}-{Q^{\phi}}^{*}}_{\H_{Q}}.\nn
\end{align}	
\section{Proof of Theorem \ref{CBFaffine}}
\label{apptheo5}
The line integral of $\frac{\partial {B}(\sigx)}{\partial \sigx}$ is path independent because it is the gradient
of the scaler field $B$ \cite{complexAna}.
Let $\sigx(t):=(1-t)\sigx_n+t\sigx_{n+1}=\sigx_n+t(\hat{f}_n(\sigx_n)+\hat{g}_n(\sigx_n)\sigu_n)$, where $t\in[0,1]$ parameterizes the line path between $\sigx_n$ and $\sigx_{n+1}$,
then $\frac{d B(\sigx(t))}{d t}=\frac{\partial B(\sigx(t))}{\partial \sigx}(\hat{f}_n(\sigx_n)+\hat{g}_n(\sigx_n)\sigu_n)$.
Therefore, for any path $A$ from $\sigx_n$ to $\hat{\sigx}_{n+1}:=\sigx_n+\hat{f}_n(\sigx_n)+\hat{g}_n(\sigx_n)\sigu_n$, it holds under Assumption \ref{assump1}.2 that
\begin{align}
&\hspace{-1em} B(\hat{\sigx}_{n+1})-B(\sigx_n)=\int_A\frac{\partial B(\sigx)}{\partial \sigx} \cdot\di \sigx=\int_0^{1}\frac{d B(\sigx(t))}{d t}\di t\nn\\
&\hspace{-1em} \geq\int_0^{1}\left(\frac{\partial B(\sigx_n)}{\partial \sigx}-\nu t(\hat{f}_n(\sigx_n)+\hat{g}_n(\sigx_n)\sigu_n)^{\T}\right)\nn\\
&\hspace{-1em} (\hat{f}_n(\sigx_n)+\hat{g}_n(\sigx_n)\sigu_n)\di t \nn\\
&\hspace{-1em} =\frac{\partial B(\sigx_n)}{\partial \sigx}(\hat{f}_n(\sigx_n)+\hat{g}_n(\sigx_n)\sigu_n)-\frac{\nu}{2}\norm{\hat{f}_n(\sigx_n)+\hat{g}_n(\sigx_n)\sigu_n}_{\Real^{n_x}}^2.  \label{ineq}
\end{align}	
The inequality implies that $B(\hat{\sigx}_{n+1})-B(\sigx_n)$ is greater than or equal to that in the case when
$\frac{\partial B(\sigx)}{\partial \sigx}$ decreases along the line path at the maximum rate.
Therefore, when \refeq{DCBF} is satisfied, it holds from \refeq{ineq} that
\begin{align}
B(\hat{\sigx}_{n+1})-B(\sigx_n)\geq -\eta B(\sigx_n)+\varrho_1,\nn
\end{align}	 
which is the control barrier certificate defined in \refeq{DCBFdef3}.
Hence, \refeq{DCBFdef2} is satisfied by the same argument as in the proof of Theorem \ref{maintheo}
under Assumption \ref{assumstab}.3.
Equation \refeq{DCBF} can be rewritten as
\begin{align}
&\frac{\partial B(\sigx_n)}{\partial \sigx}(\hat{f}_n(\sigx_n)+\hat{g}_n(\sigx_n)\sigu_n)\nonumber\\
&-\frac{\nu}{2}\norm{\hat{f}_n(\sigx_n)+\hat{g}_n(\sigx_n)\sigu_n}_{\Real^{n_x}}^2\geq-\eta B(\sigx_n)+\varrho_1. \label{DCBF2}
\end{align}	
The first term in the left hand side of \refeq{DCBF2} is affine to $\sigu_n$, the second term
is the combination of a concave function $-\frac{\nu}{2}\norm{\cdot}_{\Real^{n_x}}^2$ and an affine function of $\sigu_n$, which is concave.
Therefore, the left hand side of \refeq{DCBF2} is a concave function, and the inequality \refeq{DCBF2} defines
a convex constraint under Assumption \ref{assump1}.1.
%%%%
%%%%%%%
\section{Kernel Adaptive Filter with Monotone Approximation Property}
\label{appemulti}
Kernel adaptive filter \cite{liu_book10} is an adaptive extension of the kernel ridge regression \cite{muller01intro, schoelkopf2002} or GPs.
Multikernel adaptive filter \cite{yukawa_tsp12} exploits multiple kernels to conduct learning in the sum space of RKHSs associated with each kernel.
Let $M\in\integer_{>0}$ be the number of kernels employed.
Here, we only discuss the case that the dimension of the model parameter $\sigh$ is fixed, for simplicity.
Denote, by $\Dict_{m}:=\{\kappa_m(\cdot,\tilde{\sigz}_{m,j})\}_{j\in\{1,2,...,r_{m}\}},\;m\in\{1,2,...,M\},r_{m}\in\integer_{>0}$, the
time-dependent set of functions, referred to as a {\it dictionary}, at time instant $n$ for the $m$th kernel $\kappa_m(\cdot,\cdot)$.
The current estimator $\hat{\psi}_n$ is evaluated at the current input $\sigz_n$, in a linear form, as
\begin{equation}
\hat{\psi}_n(\sigz_n):=\sigh_n^{\T}\sigk(\sigz_n)=\sum_{m=1}^M\sigh_{m,n}^{\T}\sigk_{m}(\sigz_n),\nn
\end{equation}
where $\sigh_n:=[\sigh_{1,n};\sigh_{2,n};\cdots; \sigh_{M,n}]:=[h_1;h_2;\cdots; h_{r}]\in\Real^{r},\;r:=\sum_{m=1}^Mr_{m}$, is the coefficent vector, and $\sigk(\sigz_n):=\left[{\sigk_{1}(\sigz_n)};{\sigk_{2}(\sigz_n)};\cdots;{\sigk_{M}(\sigz_n)}\right]\in\Real^{r}$,
$\sigk_{m}(\sigz_n):=\left[\kappa_m\left(\sigz_n,\tilde{\sigz}_{m,1}\right);\kappa_m\left(\sigz_n,\tilde{\sigz}_{m,2}\right);\cdots;\kappa_m\left(\sigz_n,\tilde{\sigz}_{m,r_{m}}\right)\right]\in\Real^{r_{m}}$.
To obtain a sparse model parameter, we define the cost at time instant $n$ as
\begin{align}
\Theta_n(\sigh):=\frac{1}{2}\sum_{\iota=n-s+1}^{n}\frac{1}{s}{dist}^2(\sigh,C_{\iota})+\mu\norm{\sigh}_1, \label{APFBScost}
\end{align}
where $\iota\in\{n-s+1,n\}\subset\integer_{\geq0},\;s\in\integer_{>0}$, and
\begin{align}
C_{\iota}:=\{\sigh\in\Real^{r}||\sigh^{\T}\sigk(\sigz_{\iota})-\delta_{\iota}|\leq\epsilon_1\},\;\epsilon_1\geq0, \label{dataS}
\end{align}
which is a set of coefficient vector $\sigh$ satisfying instantaneous-error-zero with a precision parameter $\epsilon_1$.
Here, $\delta_n\in\Real$ is the output at time instant $n$, and
the $\ell_1$-norm regularization $\norm{\sigh}_1:=\sum_{i=1}^{r}|h_i|$ with a parameter $\mu\geq0$ promotes sparsity of $\sigh$.
The update rule of the adaptive proximal forward-backward splitting \cite{murakami10}, which is an adaptive filter designed for sparse optimizations, for the cost \refeq{APFBScost} is given by
\begin{align}
\sigh_{n+1}=\prox_{\lambda\mu}\left[(1-\lambda)I+\lambda\sum_{\iota=n-s+1}^{n}\frac{1}{s}P_{C_{\iota}}\right](\sigh_n), \label{hnupdate}
\end{align}
where $\lambda\in(0,2)$ is the step size, $I$ is the identity operator, and
\begin{align}
\prox_{\lambda\mu}(\sigh)=\sum_i^{r}\sgn{(h_i)}\max{\left\{|h_i|-\lambda\mu,0\right\}}\textbf{e}_i,\nn% 
\end{align}
where $\sgn(\cdot)$ is the sign function.
Then, the strictly monotone approximation property \cite{murakami10}: $\norm{\sigh_{n+1}-\sigh_n^{*}}_{\Real^{r}}<\norm{\sigh_{n}-\sigh_n^{*}}_{\Real^{r}},\;\forall \sigh_n^{*}\in\Omega_n:=\argmin_{\sigh\in\Real^{r}}\Theta_n(\sigh)$, holds
if $\sigh_n\notin\Omega_n\neq\emptyset$.
%%%%%%

\noindent{\bf Dictionary Construction:}
If the dictionary is insufficient, we can employ two novelty conditions when adding the kernel functions $\{\kappa_m(\cdot,\sigz_n)\}_{m\in\{1,2,...,M\}}$ to the dictionary:
(i) the maximum-dictionary-size condition
\begin{equation}
r\leq r_{{\rm max}},  \;r_{{\rm max}}\in\integer_{>0},\nn%
\end{equation}
and
(ii) the large-normalized-error condition
\begin{equation}
|\delta_n-\hat{\psi}_n(\sigz_n)|^2>\epsilon_2|\hat{\psi}_n(\sigz_n)|^2,\; \epsilon_2\geq 0.  \nn%
\end{equation}
By using sparse optimizations, {\it nonactive} structural components represented by some kernel functions can be removed, and the dictionary is refined as time goes by. 
To effectively achieve a compact representation of the model, it might be required to appropriately weigh the kernel functions
to include some preferences on a structure of the model.
The following lemma implies that the resulting kernels are still reproducing kernels.
\begin{lemma}[{\cite[Theorem~2]{yukawa_tsp15}}]
	\label{propweight}
	Let $\kappa:\mathcal{Z}\x\mathcal{Z}\rightarrow\Real$ be the reproducing kernel of an RKHS $(\H,\inpro{\cdot,\cdot}_{\H})$. Then, $\tau\kappa(\sigz,\sigw),\;\sigz,\sigw\in\mathcal{Z}$ for an
	arbitrary $\tau>0$ is the reproducing kernel of the RKHS $(\H_{\tau},\inpro{\cdot,\cdot}_{\H_{\tau}})$ with the
	inner product $\inpro{\sigz,\sigw}_{\H_{\tau}}:={\tau}^{-1}\inpro{\sigz,\sigw}_{\H},\;\sigz,\sigw\in\mathcal{Z}$.
\end{lemma}
%%%%%%
\section{Comparison to Parametric Approaches and the GP SARSA}
\label{appencomp}
If the suitable set of basis functions for approximating action-value functions 
is available, we can adopt a parametric approach for action-value function approximation.
Suppose that an estimate of the action-value function at time instant $n$ is given by
$\hat{Q}_n^{\phi}(\sigz)=\sigh_n^{\T}\zeta(\sigz)$,
where $\zeta:\mathcal{Z}\rightarrow\Real^r$ is fixed for all time.
In this parametric case, given an input-output pair $([\sigz_n;\sigz_{n+1}],R(\sigx_n,\sigu_n))$,
we can update the estimate of the action-value function as 
\begin{align}
	\hat{Q}_{n+1}^{\phi}=\sigh_n-\lambda&\left[\sigh_n^{\T}(\zeta(\sigz_n)-\gamma\zeta(\sigz_{n+1}))-R(\sigx_n,\sigu_n)\right]\nn\\
	&~~~~~~~~~~~~~\cdot(\zeta(\sigz_n)-\gamma\zeta(\sigz_{n+1})).\nn
\end{align}
Then, stable tracking is achieved if the step size $\lambda$ is properly selected, even after
the dynamics or the policy is changed.

On the other hand, when employing a kernel-based learning, it is not trivial how to update the estimate in a
theoretically formal manner.
Because the output of the action-value function is not directly observable, 
the expansion $\sum_{i=0}^n\kappa^Q(\cdot,\sigz_n)$ (where $\kappa^Q$ is the reproducing kernel of the RKHS containing the action-value function)
cannot be validated by the representer theorem \cite{repre1} any more.
By defining the RKHS $\H_{\psi^Q}$ as in Theorem \ref{theoRKHS}, however, we can view an action-value function approximation as
the supervised learning in the RKHS $\H_{\psi^Q}$, and can overcome the aforementioned issue.
We mention that when an adaptive filter is employed in the RKHS $\H_{\psi^Q}$, we do not have to reset learning
even after policies are updated or the dynamics changes, since the domain of $\H_{\psi^Q}$ is $\mathcal{Z}\x\mathcal{Z}$ instead of $\mathcal{Z}$.
The example below indicates that our approach is general.

As discussed in Section \ref{subsec:relatedwork}, the least squares temporal difference algorithm has been extended to kernel-based methods including the GP SARSA  \cite{engel2005reinforcement}.
Given a set of input data $\{\sigz_n\}_{n=0,1,...,N_d},\;\sigz_n:=[\sigx_n;\sigu_n],\;N_d\in\integer_{>0}$,
the posterior mean $m^{Q}$ and variance ${\mu^{Q}}^2$ of $\hat{Q}_{N_d}^{\phi}$ at a point $\sigz_{*}\in\mathcal{Z}$ are given by
\begin{align}
&m^{Q}(\sigz_{*})={\tilde{\sigk}_{N_d}}^{\T}\textbf{H}^{\T}(\textbf{H}\textbf{K}^Q\textbf{H}^{\T}+\Sigma)^{-1}{\textbf{R}}_{N_d-1},\label{meangptd}\\
&{\mu^{Q}}^2(\sigz_{*})=\kappa^{Q}(\sigz_{*},\sigz_{*}) -{\tilde{\sigk}_{N_d}}^{\T}\textbf{H}^{\T}(\textbf{H}\textbf{K}^Q\textbf{H}^{\T}+\Sigma)^{-1}\textbf{H}\tilde{\sigk}_{N_d}, \label{covgptd}
\end{align}	
where ${\textbf{R}}_{N_d-1}\sim\mathcal{N}([R(\sigx_0,\sigu_0);R(\sigx_1,\sigu_1);\cdots;R(\sigx_{N_d-1},\sigu_{N_d-1})],\Sigma)$ is the vector of immediate rewards,
$\kappa^{Q}$ is the reproducing kernel of $\H_Q$,
$\tilde{\sigk}_{N_d}:=[\kappa^{Q}(\sigz_{*},\sigz_0);\kappa^{Q}(\sigz_{*},\sigz_1);\cdots;\kappa^{Q}(\sigz_{*},\sigz_{N_d})]$, the $(i,j)$ entry of $\textbf{K}^Q\in\Real^{(N_d+1)\x (N_d+1)}$ is $\kappa^{Q}(\sigz_{i-1},\sigz_{j-1})$, and $\Sigma\in\Real^{N_d\x N_d}$ is the covariance matrix of $\textbf{R}_{N_d-1}$.
Here, the matrix $H$ is defined by
\begin{align}
\textbf{H}:=\left[
\begin{array}{ccccc}
1 & -\gamma & 0 & \cdots & 0\\
0 & 1 & -\gamma & \cdots & 0\\
\vdots &  &  &  & \vdots\\
0 & 0 & \cdots & 1 & -\gamma\\
\end{array}
\right]\in\Real^{N_d\x (N_d+1)}.\nn
\end{align}
If we employ a GP for learning $\psi^Q$ in $\H_{\psi^Q}$ defined in Theorem \ref{theoRKHS},
the posterior mean $m^{\psi^Q}$ and variance ${\mu^{\psi^Q}}^2$ of $\hat{\psi}_{N_d}^Q$ at a point $[\sigz_{*};\sigw_{*}]\in\mathcal{Z}\x\mathcal{Z}$ are given by
\begin{align}
&m^{\psi^Q}([\sigz_{*};\sigw_{*}])={\sigk_{N_d}}^{\T}(\textbf{K}+\Sigma)^{-1}{\textbf{R}}_{N_d-1},\nn\\
&{\mu^{\psi^Q}}^2([\sigz_{*};\sigw_{*}])=\kappa([\sigz_{*};\sigw_{*}];[\sigz_{*};\sigw_{*}]) -{\sigk_{N_d}}^{\T}(\textbf{K}+\Sigma)^{-1}\sigk_{N_d},\nn
\end{align}	
where $\sigk_{N_d}:=[\kappa([\sigz_{*};\sigw_{*}],[\sigz_0;\sigz_1]);\cdots;\kappa([\sigz_{*};\sigw_{*}],[\sigz_{N_d-1};\sigz_{N_d}])]$, and the $(i,j)$ entry of $\textbf{K}\in\Real^{N\x N}$ is ${\kappa}([\sigz_{i-1};\sigz_{i}],[\sigz_{j-1};\sigz_{j}])$.
Then, the posterior mean $m^{Q}$ and variance ${\mu^{Q}}^2$ of $\hat{Q}_{N_d}^{\phi}$ at a point $\sigz_{*}\in\mathcal{Z}$ are given by
\begin{align}
&m^{Q}(\sigz_{*})=U^{-1}(m^{\psi^Q}(\cdot))(\sigz_{*})={\sigk_{N_d}^Q}^{\T}(\textbf{K}+\Sigma)^{-1}{\textbf{R}}_{N_d-1},\nn\\
&{\mu^{Q}}^2(\sigz_{*})=\kappa^{Q}(\sigz_{*},\sigz_{*}) -{\sigk_{N_d}^Q}^{\T}(\textbf{K}+\Sigma)^{-1}{\sigk_{N_d}^Q},\nn
\end{align}	
which result in the same values as \refeq{meangptd} and \refeq{covgptd}.
%%%%%%
\end{appendices}
\section*{Acknowledgments}
M. Ohnishi thanks all of those who have given him insightful comments on this work,
including the members of the Georgia Robotics and Intelligent Systems Laboratory.
The authors thank all of the anonymous reviewers for their constructive suggestions.		

\end{document}